\documentclass[lettersize,journal]{IEEEtran}

\usepackage{amsthm}

\usepackage{times}
\usepackage{multicol}
\usepackage[bookmarks=true]{hyperref}
\usepackage{xcolor}
\usepackage{amsmath, amssymb}
\usepackage{amsfonts}
\usepackage{siunitx}
\usepackage{standalone}
\usepackage[ruled,vlined,linesnumbered,noend]{algorithm2e}
\usepackage{mdframed}
\usepackage{fancyvrb,multirow}
\usepackage{soul}
\usepackage{dsfont,mathabx}
\usepackage{array, booktabs}
\usepackage{makecell}
\usepackage[para,online,flushleft]{threeparttable}
\usepackage{tikz}
\usepackage{pgfplots}
\usepackage{tabularx}
\usepackage[english]{babel}
\usepackage{blindtext}
\usepackage{enumitem}
\usepackage{subcaption}
\usepackage{textcomp}
\usepackage{stfloats}
\usepackage{url}
\usepackage{verbatim}
\usepackage{graphicx}
\usepackage{cite}

\usepackage{caption}

\renewcommand{\tablename}{TABLE}

\DeclareCaptionLabelFormat{upper}{\MakeUppercase{#1}~#2}

\newcolumntype{Y}{>{\hsize=0.8\hsize\centering\arraybackslash}X} 
\newcolumntype{Z}{>{\hsize=1.2\hsize\centering\arraybackslash}X} 
\newcolumntype{C}{>{\centering\arraybackslash}X} 

\newcolumntype{M}[1]{>{\centering\arraybackslash}m{#1}}
\newtheorem{theorem}{Theorem}

\newtheorem{lemma}{Lemma}
\theoremstyle{definition}

\newtheorem{definition}{Definition}
\newtheorem{remark}{Remark}
\newtheorem{example}{Example}
\newtheorem{problem}{Problem}

\begin{document}

\title{SLEI3D: Simultaneous Exploration and Inspection via
Heterogeneous Fleets under Limited Communication}

\author{Junfeng Chen$^1$, Yuxiao Zhu$^2$,
Xintong Zhang$^2$, Bing Luo$^2$, and Meng Guo$^1$
\thanks{The authors are with $^1$the School of Advanced Manufacturing and Robotics,
  Peking University, Beijing 100871, China;
and $^2$the Division of Natural and Applied Sciences, Duke Kunshan University, Suzhou 215316, China.
 {\tt\small meng.guo@pku.edu.cn}
}
}



\maketitle

\begin{abstract}
 Robotic fleets such as unmanned aerial and ground vehicles
 have been widely used for routine inspections of static environments,
 where the areas of interest are known and planned in advance.
 However, in many applications, such areas of interest are unknown
 and should be identified online during exploration.
 Thus, this paper considers the problem of simultaneous exploration,
 inspection of unknown environments and then real-time
 communication to a mobile ground control station
 to report the findings.
 The heterogeneous robots are equipped with different sensors,
 e.g., long-range lidars for fast exploration and close-range cameras
 for detailed inspection.
 Furthermore, global communication is often unavailable in such environments,
 where the robots can only communicate with each other via ad-hoc wireless networks
 when they are in close proximity and free of obstruction.
 This work proposes a novel planning and coordination framework (SLEI3D)
 that integrates the online strategies for collaborative 3D exploration,
 adaptive inspection and timely communication
 (via the intermittent or proactive protocols).
 To account for uncertainties w.r.t. the number and location of features,
 a multi-layer and multi-rate planning mechanism is developed
 for inter-and-intra robot subgroups,
 to actively meet and coordinate their local plans.
The proposed framework is validated extensively via high-fidelity simulations
of numerous large-scale missions with up to~$48$ robots
and $384$ thousand cubic meters.
Hardware experiments of~$7$ robots are also conducted.
\end{abstract}

\def\abstractname{Note to Practitioners}
\begin{abstract}

    This paper is motivated by the challenges of coordinating large-scale
    heterogeneous fleets for the inspection of large buildings and infrastructure,
    where heterogeneous UAVs must collaborate to explore unknown environments,
    identify areas of interest,
    and more importantly, inspect specific features
    (such as cracks, leaks, and other anomalies).
    Furthermore, these features must be relayed back to a
    control station for further analyses.
    Existing methods predominantly focuses on exploration tasks
    and often overlooks the need for close-up inspection.
    Instead,
    a hierarchical and flexible framework is proposed to
    coordinate a group of heterogeneous UAVs online
    for efficient exploration, inspection and communication,
    subject to an unknown number and location of features.
    Instead of relying on an all-to-all communication network,
    limited communication range and bandwidth are addressed
    by leveraging intermittent and proactive communication protocols,
    i.e., to enable exchange of local plans, explored areas,
    and detected features during online execution.
    Via extensive simulations in a high-fidelity simulator,
    the proposed framework is shown to be efficient and reliable
    for large-scale simultaneous exploration and inspection tasks
    within various scenes.
    Robustness to robot failures and communication loss is also demonstrated.
    Hardware experiments over UGVs and UAVs validate the practical relevance.
\end{abstract}

\begin{figure}[!t]
    \centering
    \includegraphics[width=1\linewidth, height=0.88\linewidth]{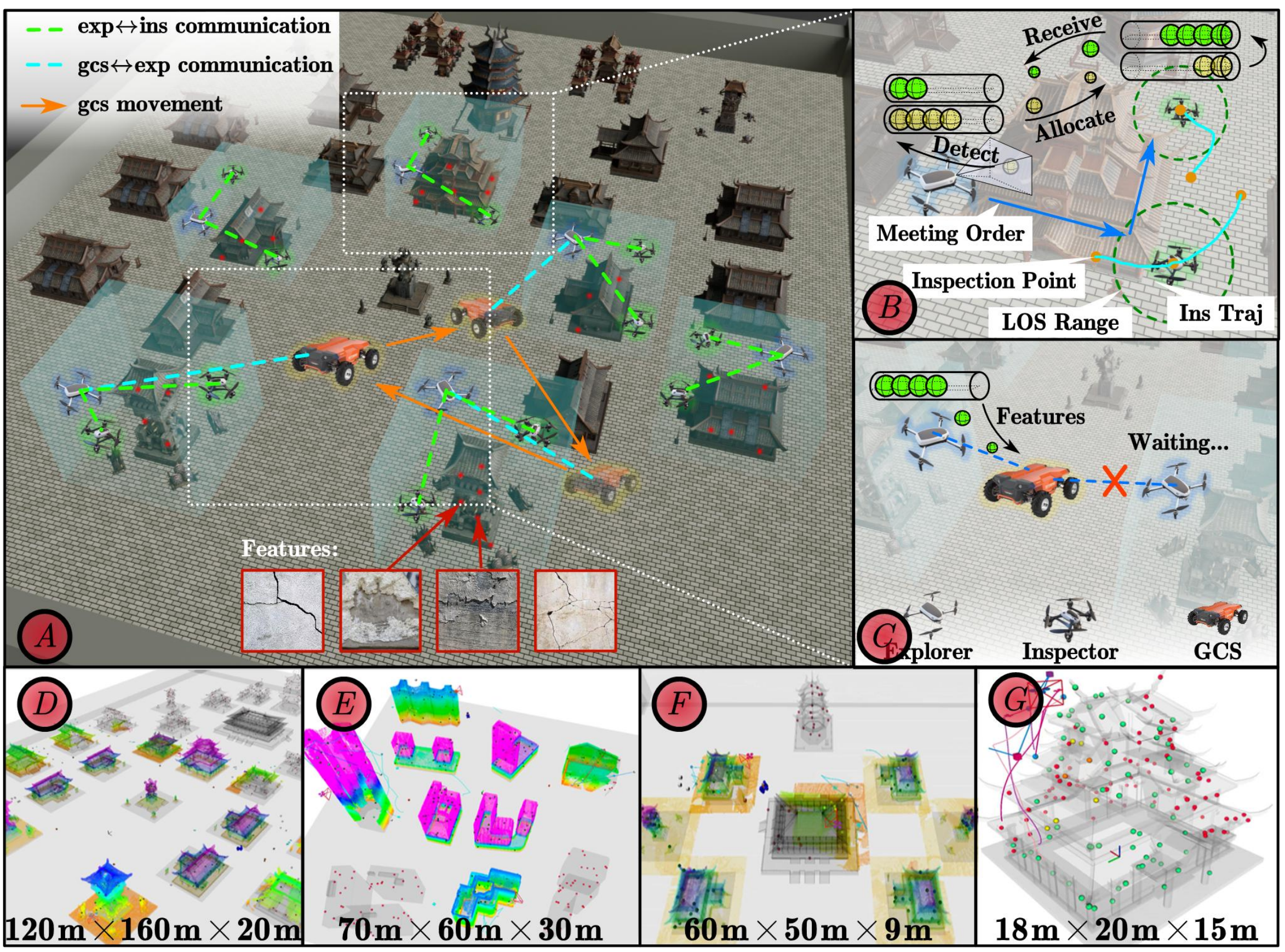}
    \vspace{-3mm}
    \caption{
      \textbf{(a)} $6$ explorers and $12$ inspectors
      are tasked to simultaneously explore and inspect
      an unknown number
      of numerous features;
      \textbf{(b)} the explorer and inspectors within the same subgroup
      coordinate for the inspection tasks;
      \textbf{(c)} the mobile ground station actively meets with
      the explorers to receive the latest features;
      \textbf{(d)-(g)} snapshots of online execution in four different large-scale scenes.
    }
    \label{fig:overall}
    \vspace{-6mm}
  \end{figure}

\section{Introduction}\label{sec:intro}
{Fleets of unmanned aerial vehicles (UAVs) and ground vehicles (UGVs)
have been deployed to perform routine maintenance and inspection tasks
for large and remote infrastructures such as
power plants~\cite{hinostroza2024autonomous, zhang2024inspection},
bridges~\cite{jiang2025key, afrazi2025use}
and industrial sites~\cite{best2024multi, wang2025multi}.}
Such tasks are often static and repetitive,
where the sequence of areas to visit and inspect are given in advance.
{However, in many applications as highlighted in Fig.~\ref{fig:overall},
both the environment
and the areas of interests (AoI) are unknown a priori,
which requires the robots to simultaneously explore
the environment, identify the AoI,
and then inspect the features therein.}
Most existing work on collaborative
exploration~\cite{yamauchi1997frontier,colares2016next,
zhou2023racer,patil2023graph}
has focused on only the exploration task
{to obtain the global map quickly},
which overlooks the need for close-up inspection
of certain features detected during exploration~\cite{naazare2022online, gao2023uav},
e.g., cracks in planetary caves~\cite{allan2019planetary},
life signs during search and rescue~\cite{couceiro2017overview},
and zoomed images at archaeological sites~\cite{brutto2012uav}.
More importantly, these work often assumes a global all-to-all communication
among the robots,
which could be impractical in many aforementioned scenes
where the communication facilities are unavailable~\cite{ray2016internet} or severely degraded~\cite{hu2012cloud}.
In such cases, the robots can only exchange information
via ad-hoc networks subject to the line-of-sight (LOS)~\cite{murayama2023bi}
and proximity constraints~\cite{mosteo2010multi}.
This imposes great challenges on the coordination of
the robotic fleet as communication events,
exploration and inspection tasks are now closely dependent,
thus should be planned simultaneously~\cite{guo2018multirobot, saboia2022achord, kantaros2019temporal, da2023communication}.

\begin{figure*}[!t]
    \centering
    \includegraphics[width=0.9\linewidth]{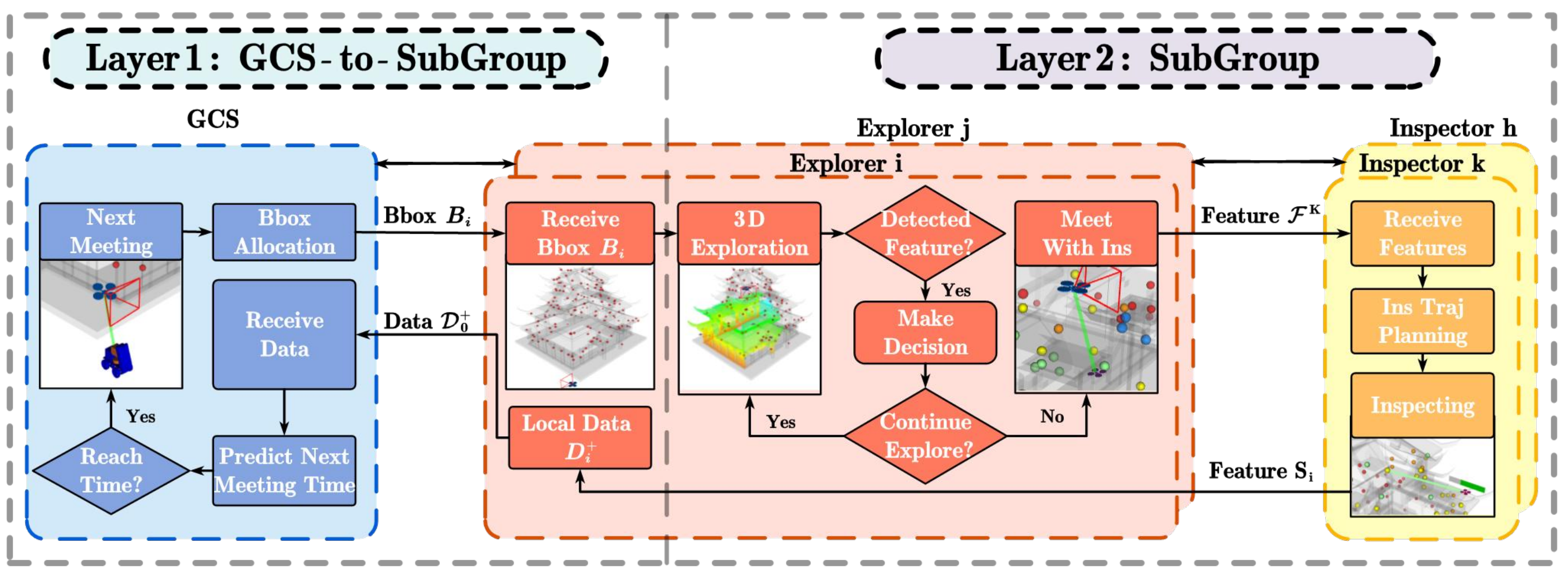}
    \vspace{-2mm}
    \centering
    \caption{Overview of the proposed method, which consists of two hierarchical layers:
        the first layer where GCS coordinates with explorers, and the second layer where the explorer
    coordinates with inspectors within each subgroup, all under limited communication.}
    \label{fig:framework}
    \vspace{-4mm}
  \end{figure*}

Moreover, it is often essential to provide a timely update
to a mobile ground control station (GCS),
regarding the progress of exploration and inspected features,
e.g., to plan for further actions such as maintenance and repair~\cite{milidonis2023unmanned}.
This can be particularly challenging without a global communication network,
as the GCS would lose connection whenever the robots
spread out for exploration~\cite{tian2024ihero}.
In other words,
the robots and the GCS should actively and frequently coordinate
their communication including time, location and the content.
How to design such a flexible and efficient coordination framework
for the GCS and robotic fleets remains unsolved.

As shown in Fig.~\ref{fig:framework},
this work proposes a simultaneous large-scale exploration, inspection
and communication framework (SLEI3D)
for a heterogeneous robotic fleet that operates in unknown environments
under limited communication.
Given the bounding boxes that enclose the AoI,
the robots are divided into numerous subgroups of explorers and inspectors.
As an essential component, a 3D collaborative exploration strategy is
designed for explorers with long-range Lidars
to detect AoI efficiently,
based on geometric-aware frontier generation.
Allocation of the identified AoI to inspectors
with close-range cameras is then formulated
as a 3D constrained routing problem,
to maximize the inspection efficiency
and ensure a safety distance.
Then a prediction algorithm for the task completion time
is designed for the GCS
to rapidly collect the result of inspected features.
More importantly,
an intermittent communication protocol is designed
between the GCS and the subgroups,
to facilitate online data exchange and adaptation given the updated map and the features.
In contrast,
a proactive communication protocol is employed within each subgroup,
where the explorer dynamically determines the time and location
to communicate with the inspectors based on the detected features.
The efficiency and reliability of the proposed framework are
analyzed theoretically and validated
via extensive large-scale simulations and hardware experiments.
Up to~$48$ robots are deployed to explore and inspect
numerous large-scale scenes with more than~$150$ inspection tasks.

Main contributions of this work are threefold:
(I) the novel problem formulation of simultaneous exploration and inspection
under limited communication for heterogeneous robotic fleets;
(II) the multi-layer and multi-rate coordination framework that
co-optimizes the exploration task, the inspection task
and the inter-robot communication;
and (III) the extensive large-scale simulations that validate
the performance in practical scenes.
To the best of our knowledge, this is the first work that
provides such a comprehensive solution.

\section{Related Work}\label{subsec:intro-related}

\subsection{Multi-robot Collaborative Exploration}
Autonomous exploration has a long history in robotics~\cite{sharma2016survey, hu2024review},
e.g.,~\cite{yamauchi1997frontier} introduces an intuitive yet powerful
frontier-based method for guiding the exploration.
It has been adapted to multi-robot teams
by assigning these frontiers to different robots for concurrent
exploration via e.g., distributed auction~\cite{hussein2014multi},
multi-vehicle routing~\cite{zhou2023racer},
optimization of information gain in~\cite{colares2016next, patil2023graph},
and dynamic optimization of topological graph in~\cite{dong2024fast}.
On the other hand,
the work in~\cite{cabrera2012flooding} presents a flooding algorithm that ensures multiple robots
can explore the entire environment without missing any area.
A multi-robot depth-first search (MR-DFS) method is proposed in~\cite{doolittle2012population} to explore unknown environments
encoded as a graph by parallel search.
However, these work commonly assumes that all robots can
communicate instantly and exchange information at all times,
i.e., they always have access to {the same global map}.
However, this is often impractical for unknown environments
where the inter-robot communication is limited in range.
Besides the above classical methods,
reinforcement learning (RL)-based approaches have also been applied to multi-robot exploration,
see~\cite{battocletti2021rl, zhu2024maexp, chen2024meta}.
They mostly focus on the design of customized observation spaces encapsulating partial observability constraints
and novel reward functions to improve exploration efficiency.
Therefore, many recent work combines the planning
of inter-robot communication and autonomous exploration~\cite{amigoni2017multirobot}.
The ``Zonal and Snearkernet'' relay algorithm in~\cite{vaquero2018approach},
the four-state strategy in~\cite{cesare2015multi},
a centralized integer optimization for rendezvous points in~\cite{gao2022meeting},
a Steiner tree-based deployment in~\cite{stump2011visibility},
the circular communication model in~\cite{pei2013connectivity}
adopt an event-based communication scheme.
On the other hand,
the work in~\cite{rooker2007multi} adopts fully-connected networks
at all time, while radio dropplets are utilized in~\cite{saboia2022achord}
as extended communication relays between robots.
Nonetheless, these work considers only the task of collaborative exploration,
without addressing the inspection tasks of certain features.

\subsection{Autonomous Inspection}
Autonomous inspection can already be found in various applications
via different sensors, e.g.,~\cite{couceiro2017overview,brutto2012uav}.
Such tasks involve generating a set of viewpoints based on
the 3D structure and the sensor intrinsics,
which are then assigned to the robots for inspection.
A skeleton-based space decomposition method is proposed in~\cite{feng2024fc}
followed by a travel salesman problem (TSP) algorithm.
The work in~\cite{jing2019coverage} generates via-points and path primitives
using voxel dilation or subtraction,
and employs a primitive coverage graph (PCG) to optimize the collective paths.
However,
these work primarily focuses on single-robot exploration strategies
and is not directly applicable to multi-robot systems.

The work in~\cite{jing1808samplingbased} employs random sampling
in combination with potential fields to generate candidate viewpoints,
which are allocated to a fleet of UAVs by solving an integer optimization problem.
Moreover,
the Multi-UAV Coverage Path Planning for Inspection (MU-CPPI) algorithm in~\cite{jing2020multi},
addresses the allocation problem of viewpoints by formulating as a set-covering vehicle routing problem (SC-VRP).
This approach builds on an \textit{exploration-then-inspection} framework,
where UAVs first fully explore the environment to construct a prior map,
followed by dedicated inspection path planning based on the acquired data.
This decoupling often leads to low efficiency of inspection.
To tackle this,
simultaneous exploration and photographing framework (SOAR) is proposed in~\cite{zhang2024soar},
where SOAR employs LiDAR-equipped explorers to detect uncovered areas and generate inspection viewpoints at surface frontiers,
while camera-equipped photographers are assigned to these viewpoints via solving a \textit{Consistent Multiple Depot Multiple Traveling Salesman Problem} (Consistent-MDMTSP).
However, it relies on the persistent all-to-all communication between UAVs and a static ground control station (GCS) for global task allocation.
Furthermore, the framework designates a single explorer as the sole frontier detector,
yielding a computational bottleneck in large-scale environments.
The most relevant work~\cite{xu2024cost} called CARIC,
introduces a hierarchical strategy for simultaneous exploration and inspection in multi-robot systems.
CARIC partitions robots into specialized teams assigned to subregions,
where they perform local tasks and relay inspection results to a static GCS through line-of-sight (LOS) communication.
{
However, CARIC enforces inter-robot connectivity by pre-defining communication points on rectangular bounding boxes, 
a method that cannot be applied to irregular bounding boxes 
with non-convex internal structures or to scenarios involving a dynamic GCS
that interacts with multiple teams of heterogeneous UAVs.
}
\subsection{Multi-robot Coordination under Limited Communication}

The key challenge in multi-robot coordination
under communication constraints,
is to determine when and where inter-robot communication should occur,
with the message content tailored to different purposes.
This issue arises not only in cooperative exploration~\cite{sheng2004multi},
but also in cooperative patrolling~\cite{mosteo2008multi},
and coverage planning~\cite{santos2018coverage}.
Different communication protocols have been proposed, e.g.,
the work in~\cite{gao2022meeting} proposes the meeting-merging-mission protocol
for all robots,
by formulating a constrained integer optimization problem to determine the rendezvous points and the meeting time.
In related work in~\cite{tian2024ihero} presents the protocol
of distributed intermittent communication,
by solving iteratively multiple vehicle routing problems with time window (MVRP-TW).
The most relevant work~\cite{xu2024cost}
addresses the LOS communication constraints
by choosing communication points on regular bounding boxes,
which is not applicable to irregular structures.
{However, these work primarily assumes a uniform purpose,
e.g., collaboration, exploration or coverage.}
In addition, the role of a mobile GCS as an information hub
and the requirement to relay real-time information
is not considered in these aforementioned work.

\section{Problem Description}\label{sec:problem}

\subsection{Model of Workspace and Robots}\label{subsec:prb_robot}
Consider a group of~$N$ robots~$\mathcal{N}\triangleq \{1,\cdots,N\}$
that collaborate in a common, unknown and bounded workspace~$\mathcal{W} \subset \mathbb{R}^3$.
{It is assumed that the bounding size and shape of the workspace is known a priori.}
Each robot~$i\in \mathcal{N}$ has a state~$x_i\in \mathbb{R}^3$
and the system state is given by the stacked vector~$X\triangleq [x_i]\in \mathcal{X}$.
Due to the safety constraints such as inter-robot and robot-obstacle collision avoidance,
the system state is restricted to a safety set~$\widehat{\mathcal{X}}\subset \mathcal{X}$.
There are two types of robots~$\mathcal{N}\triangleq \mathcal{N}_{\texttt{e}}\cup \mathcal{N}_{\texttt{o}}$
with explorers~$\mathcal{N}_{\texttt{e}}$ and inspectors~$\mathcal{N}_{\texttt{o}}$.
Each robot~$i\in \mathcal{N}$ can perform SLAM locally and navigate within its map safely
without colliding with other robots or detected obstacles.
Denote by
\begin{equation}\label{eq:robot-nav}
  \mathbf{p}_{sg} \triangleq \texttt{Navi}_i(p_i,\, p_g,\, \mathcal{M}_i),
\end{equation}
as the navigation module~\cite{khalil2002nonlinear} that guides robot~$i$ from its current pose~$p_i\in \mathcal{W}$
to a target pose~$p_g\in \mathcal{W}$
via the path~$\mathbf{p}_{sg}$ within its local workspace map~$\mathcal{M}_i$,
while ensuring that~$\mathbf{p}_{sg}\subset \widehat{\mathcal{X}}$.
Each pair of robots can exchange data via ad-hoc wireless communication networks,
subject to the line-of-sight (LOS) and limited-range constraints.
Denote by
\begin{equation}\label{eq:robot-com}
  (D^+_i,\, D^+_j)\triangleq \texttt{Comm}_{ij}(D_i,\, D_j),
\end{equation}
as the communication module that robots~$i,j\in \mathcal{N}$ update
their local data~$D_i$ and $D_j$ via communication if their LOS
is not blocked by obstacles in~$\mathcal{W}$,
and their relative distance is within their communication range,
denoted by~$r_{ij}>0$.

In addition, there is a mobile GCS which has a unique index~$0$,
{
which is only responsible for sending and receiving data and does not undertake the role of communication manager.}
It follows the same navigation module in~\eqref{eq:robot-nav},
and has the same communication constraints
with other robots in~\eqref{eq:robot-com},
i.e., the communication range is given
by~$r_{0i}>0$,~$\forall i\in \mathcal{N}$.
Similarly, the local map at the GCS is denoted by~$\mathcal{M}_0$
and local data by~$D_0$,
For brevity, denote by~$\mathcal{N}^+\triangleq \{0\}\cup \mathcal{N}$.

\subsection{Exploration and Inspection}\label{subsec:prb_robot}
Each explorer~$i\in \mathcal{N}_{\texttt{e}}$ can update its local map
via the exploration module
\begin{equation}\label{eq:robot-explore}
\mathcal{M}^+_i \triangleq \texttt{Explore}_i(p_i,\,\mathcal{M}_i,\,\mathcal{W}),
\end{equation}
where~$\mathcal{M}^+_i$ is the updated local map.
For instance, the octomap~\cite{hornung2013octomap}
can be constructed online via onboard Lidar.
Moreover, there are~$Q>0$ features of interest in the workspace,
denoted by~$\mathcal{F}\triangleq \{f_q,\,\forall q\in Q\}$.
Given the updated map,
each explorer~$i \in \mathcal{N}_{\texttt{e}}$ can identify a set of potential Areas of Interest (AoI) in its local map
that contains several features,
through a multi-modal perception system that integrates 3D reconstruction from fused visual-depth images,
enhanced by real-time semantic segmentation and accelerated object
detection~\cite{chen2024improved, wang2023yolov7}.
This detection process is formalized via the feature fusion module:
\begin{equation}\label{eq:robot-fit}
\big\{(f_q,\, \phi_q)\big\}\triangleq \texttt{Fit}_i(\mathcal{M}_i),
\end{equation}
where~$\phi_q \subset \mathcal{M}_i$
is the AoI that might contain feature~$f_q\in \mathcal{F}$.

Each inspector~$i\in \mathcal{N}_{\texttt{o}}$
can inspect an AoI~$\phi_q$ in close range,
i.e., to further determine whether the feature~$f_q$ exists within~$\phi_q$
and its state, via the inspection module:
\begin{equation}\label{eq:robot-inspect}
  (D^+_i, \, s_q, \, t_q) \triangleq \texttt{Inspect}_i(p_i,\, f_q,\, \phi_q),
\end{equation}
where~$s_q \in \mathbb{R}^{L}$ is the grounded representation of feature~$f_q$
with dimension~$L>0$, such as positions, images and point clouds;
$t_q>0$ is the duration of inspecting feature~$f_q$;
and the local data~$D^+_i$ is updated with the inspection results.

The GCS is required to collect data from the robotic fleet
regarding the progress of exploration and inspection.
The data collection process follows the same communication protocol
as in~\eqref{eq:robot-com} to update its local data~$D_{0}$.

\begin{example}\label{example:prob}
As shown in Fig.~\ref{fig:overall}, the archaeological mission considered in the simulation
deploys~$6$ large UAVs as explorers with high-resolution Lidar
to construct the global map of the entire site;
$12$ small UAVs as inspectors with cameras to
take close-up images of numerous potential features;
and~$1$ GCS to gather, coordinate and update features from the fleet.
\hfill $\blacksquare$
\end{example}

\subsection{Problem Formulation}\label{subsec:prb_formulation}
The local plan of each robot~$i \in \mathcal{N}^+$
is given by a sequence of navigation and various actions, i.e.,
\begin{equation}\label{eq:local-plan}
\xi_i\triangleq \mathbf{p}^1_{i}a^1_{i}\cdots \mathbf{p}^t_{i}a^t_{i},
\end{equation}
where~$\mathbf{p}^t_i\in \mathcal{W}$ is the sequence of waypoints;
and~$a^t_i$ is the sequence of actions, i.e., exploration
or communication for explorers~$i\in \mathcal{N}_{\texttt{e}}$;
inspection or communication for inspectors~$i\in \mathcal{N}_{\texttt{o}}$
and communication for the GCS.
The planning objective is to design the exploration, inspection
and communication strategy for the robotic fleet,
such that the total time of gathering all features by the GCS
is minimized, i.e.,
  \begin{subequations}\label{eq:prb_formulation}
    \begin{align}
      & \underset{\{\xi_i\}}{\textbf{min}}\quad T \notag \\
      \textbf{s.t.} \quad & \mathcal{W}\subseteq M_0(T); \label{eq:prb_formuation:a}\\
      & s_{q} \subseteq D_0(T),
      \;\forall q\in Q; \label{eq:prb_formulation:b} \\
      & x(t) \in \widehat{\mathcal{X}},\; \forall t \in [0,\,T]; \label{eq:prb_formulation:c} \\
      & \eqref{eq:robot-nav}-\eqref{eq:robot-inspect},\;
       \forall i\in \mathcal{N}; \label{eq:prb_formulation:d}
    \end{align}
  \end{subequations}
where $T>0$ is the duration of the mission to be minimized;
the constraint~\eqref{eq:prb_formuation:a} ensures that the local map~$M_0$ of the GCS at time~$t=T$
contains the entire workspace;
the constraint~\eqref{eq:prb_formulation:b} requires that the inspection results of all features
within the entire workspace are obtained in the local data~$D_0(T)$ of the GCS;
and the other constraints~\eqref{eq:prb_formulation:c},~\eqref{eq:prb_formulation:d}
ensure that the fleet follows the navigation, exploration, inspection and communication modules
as described earlier.

\section{Proposed Solution}\label{sec:solution}


The proposed solution consists of four parts:
(I) the overview of proposed method in Sec.~\ref{subsec:overview};
(II) the coordination of exploration tasks and communication events in {Sec.~\ref{sec:gcs-explore}},
under limited communication between explorers and a GCS;
(III) the communication-aware 3D exploration and inspection algorithm in {Sec.~\ref{sec:inspect}},
tailored for a set of bounding boxes;
(IV) the online adaptation of the collaborative exploration in Sec.~\ref{sec:online},
inspection and communication,
given the updated map and detected features.

\subsection{Overview of Proposed Method}
\label{subsec:overview}

The proposed method tackles above optimization problem in~\eqref{eq:prb_formulation}
via a multi-layer and multi-rate coordination framework that simultaneously
co-optimizes the collaborative behaviors of GCS, explorers and inspectors.
{As illustrated in Fig.~\ref{fig:framework},
the robotic exploration efficiency is enhanced by constraining the search space with prior knowledge regarding the distribution of AoI:
bounding boxes $\mathcal{B} \triangleq \{B_i\} \subset \mathcal{W}$ that encapsulate clusters of internal architectures.}
However, this framework can also be applied to fully unknown workspace without any prior information,
via a prior-free exploration module described in the sequel.
{Given BBoxes,
the GCS is responsible for receding-horizon allocation of
bounding boxes (BBoxes) denoted as~$\mathcal{B} \triangleq \{B_1, B_2, \cdots\} \subset \mathcal{W}$,
reassigning the subgroup to other unfinished BBoxes,
and collecting both map information~$\mathcal{M}$
and inspection results of features~$\mathcal{S} \triangleq \{ s_1, s_2, \cdots \}$.}
Initially,
GCS divides all robots into multiple subgroups based on the number of explorers, BBoxes and the robotic sensing capabilities.
Then,
each subgroup is assigned to the nearest BBox using rolling assignment algorithm as described in the sequel.
{Furthermore,
the proposed method adopts a two-layer communication structure,
i.e., the first layer coordinates the GCS and subgroup at a low frequency,
while the second layer manages the intra-group collaboration in higher frequency.}

\subsubsection{Layer of GCS-to-SubGroups}
\label{subsec:gcs-subgroup}
{As for the GCS-to-SubGroup layer,
an intermittent communication protocol is designed to facilitate the coordination between GCS and explorers},
mainly focusing on three aspects:
(I) Management of BBoxes.
The GCS informs the location of BBoxes to the explorer of the subgroup,
and explorer reports the completion status of BBoxes back to GCS.
The rolling assignment algorithm is applied by GCS to assign any remaining BBoxes to the nearest subgroup
to accelerate overall mission.
(II) Collection of inspection results.
Inspectors transmit their inspection results~$\mathcal{S}$ to the explorer,
which then forwards these results to GCS along with the exploration map~$\mathcal{M}$.
(III) Coordination of meeting time and location.
GCS and the explorer negotiate the time and location of their next meeting by utilizing
the prediction algorithm for task completion time.
To further quantify the efficiency of task execution,
the metrics to measure idle time associated with GCS and explorers
are introduced.
Specifically,~$\tau_0$ denotes the time GCS spends waiting
to meet the explorer.
For each explorer~$i \in \mathcal{N}_{\texttt{e}}$,
the idle time is defined as~$\tau_i \triangleq \tau_i^- \cup \tau_i^+$,
where~$\tau_{i}^-$ denotes the travel time to the meeting location
with GCS or the inspectors,
and~$\tau_{i}^+$ is the waiting time at the meeting location
before the meeting starts.
Detailed descriptions are given in Sec.~\ref{sec:gcs-explore}.

\begin{figure}[t]
    \centering
    \includegraphics[width=1.0\linewidth]{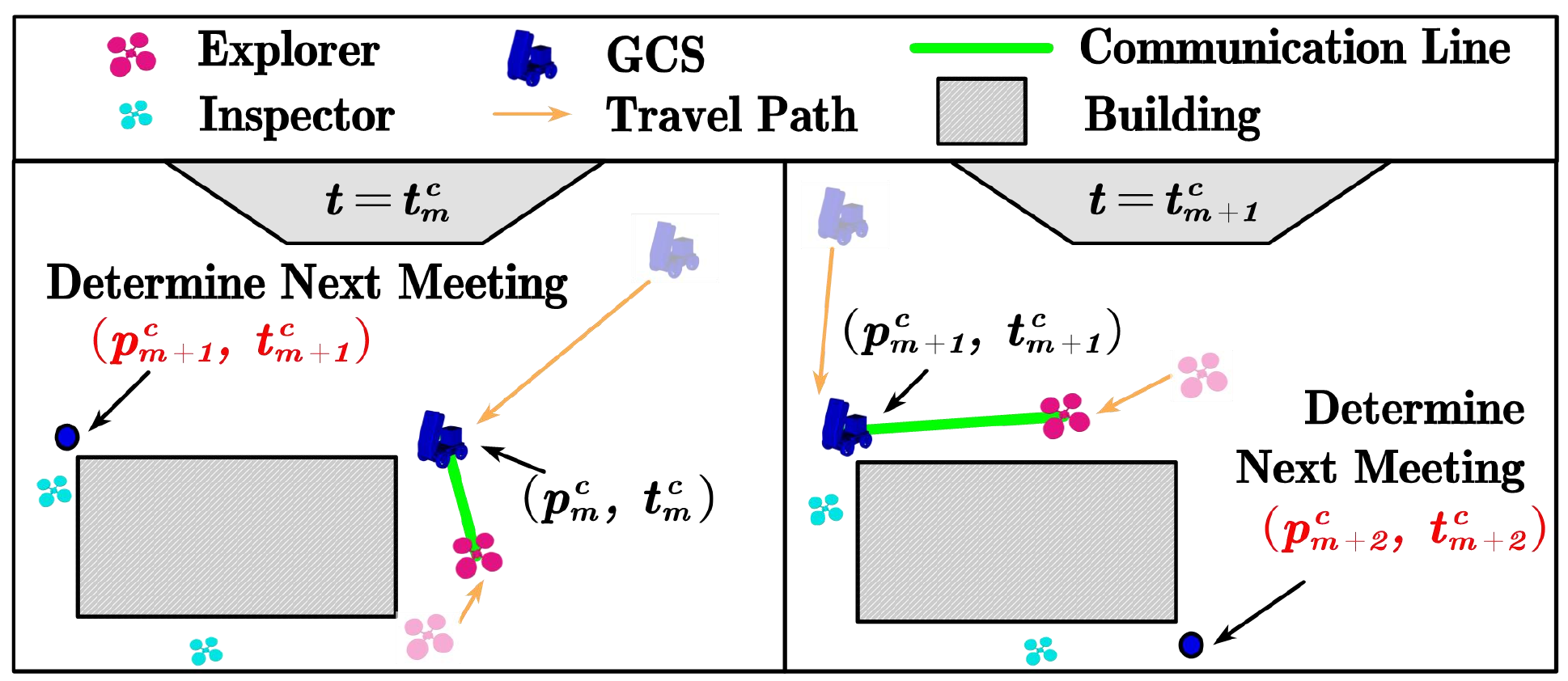}
    \centering
    \caption{Intermittent communication between GCS and explorers.}
    \label{fig:intermittent}
    \vspace{-6mm}
  \end{figure}

\subsubsection{Layer of SubGroups}
\label{subsec:subgroup}
{A proactive communication protocol within a subgroup is proposed to
coordinate the collaboration between explorer and inspectors,
under limited communication.}
It contains two main components:
(I) Allocation of AoIs.
The explorer assigns detected AoIs to the inspectors.
(II) Planning for inspection trajectories and relaying inspection results.
Inspectors relay their inspection results and their planned trajectories to the explorer,
to coordinate their next meeting events.
To evaluate the efficiency of this intra-group coordination,
the idle time~$\tau_{j} \triangleq \tau_{j}^- \cup \tau_{j}^+$
is introduced as the efficiency metrics for each inspector~$j \in \mathcal{N}_{\texttt{o}}$,
where~$\tau_{j}^-$ is the idle period when no features are available for inspection,
and~$\tau_{j}^+$ is the travel time to the features.
Detailed descriptions are given in Sec.~\ref{sec:inspect}.


Given the above multi-layer framework and the coordination strategies,
the original  problem in~\eqref{eq:prb_formulation} is reformulated
to minimize the total idle time of all robots,
subject to the same constraints,
i.e., the objective function is re-stated as follows:
\begin{equation}\label{eq:new-objective}
    \textbf{min}_{\{\xi_i\}}
    \big{\{}\sum_{i \in \mathcal{N}^+} \tau_{i}\big{\}},
\end{equation}
where~$\tau_{i}$ is the idle time for each robot $i\in \mathcal{N}^+$
as defined above.
Note that minimizing~$T$ in~\eqref{eq:prb_formulation}
can be achieved by minimizing the total idle time,
i.e., maximizing the mission efficiency
for exploration, inspection and communication.


\begin{remark}\label{remark:practitioners}
  A general guideline of the size of robotic fleet
  is provided here for practitioners,
  based on the prior information regarding numbers and volumes of BBoxes across scenarios.
  Namely,
  each BBox is allocated~$1$ GCS and~$1$ explorer to ensure parallel execution of exploration, inspection and communication.
  The number of inspectors is determined by:
  $
  |\mathcal{N}_{\texttt{O}}| = \sum_{i=0}^{|\mathcal{B}|} \left( \frac{2V_{B_i}}{V_{B_{\texttt{base}}}} \right)
  $,
  where~$V_{B_{\texttt{base}}}$ denotes the volume of BBox for a single building,
  with the coefficient 2 reflecting empirically validated inspectors per building.
  These values represent minimum recommendations for task completion.
  Lower allocations compromise efficiency but remain executable.
  Exceeding recommended quantities generally enhances task efficiency through improved resource parallelism.
\end{remark}

\subsection{Layer of GCS-to-SubGroups}
\label{sec:gcs-explore}


\subsubsection{Subproblem Formulation}
\label{subsec:explore-prob}

To ensure data collection online,
the GCS and subgroups are required to meet and communicate frequently
via an intermittent communication protocol
{under the communication module~$\texttt{Comm}_{0i}, i \in \mathcal{N}_{\texttt{e}}$
in~\eqref{eq:robot-com}}.
Denote by~$\mathcal{C} \triangleq C_1 C_2 \cdots$ the sequence of communication events,
where~$C_m=(p^{\texttt{c}}_m, t^{\texttt{c}}_m)$ represents the~$m$-th communication between GCS and the explorers
within each subgroup;
$p^{\texttt{c}}_m$ and~$t^{\texttt{c}}_m$ are the location and time for the~$m$-th communication.
Then, the protocol of intermittent communication is as follows:
(I) GCS can communicate with the explorers when they satisfy the LOS and communication range;
(II) during each communication,
they determine locally the next meeting time and location.
Then they depart and do not communicate until the next meeting.
This procedure is repeated until termination, as shown in Fig.~\ref{fig:intermittent}.
Under this protocol,
the idle time~$\tau_i$ of the explorers~$i \in \mathcal{N}_{\texttt{e}}$
mainly originates from the travel time~$\tau^-_{i_m}$ from exploration to the meeting location~$p^{\texttt{c}}_m$
at the predefined time~$t^{\texttt{c}}_m$,
and the waiting time~$\tau^+_{i_m}$ for the GCS.
On the other hand,
the idle time~$\tau_{0_m}$ of the GCS is the waiting time
for the explorer to arrive at~$p^{\texttt{c}}_m$ for the $m$-th communication.
Therefore,
the objective in~\eqref{eq:new-objective} is to minimize the total idle time for explorers and GCS
during all communication rounds.

\begin{problem}\label{prob:intermit-comm}
Determine the optimal plan~$\{\xi_i, i\in \{0\}\cup \mathcal{N}_{\texttt{e}}\}$
such that the total idle time for the explorers and GCS is minimized, i.e.,
$\textbf{min}_{\{\xi_i\}} \sum^M_{m=1}
(\sum_{i \in \mathcal{N}_{\texttt{e}}} \tau_{i_m} + \tau_{0_m}),$
where~$M>0$ is a predefined planning horizon
as the number of communication rounds ahead.
\hfill  $\blacksquare$
\end{problem}

\begin{figure}[!t]
    \centering
    \includegraphics[width=1.0\linewidth]{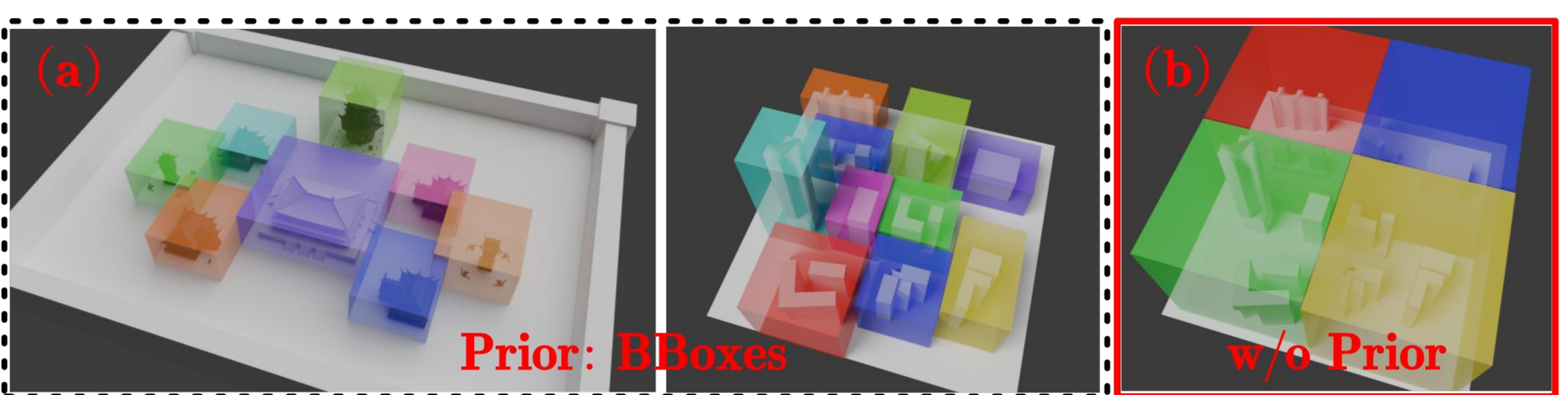}
    \centering
    \caption{BBoxes Construction via the vertical buildings (\textbf{left}) for two scenarios;
    Adaptive BBox partition without any priors (\textbf{right}).}
    \label{fig:bboxes}
    \vspace{-7mm}
  \end{figure}

To tackle Problem~\ref{prob:intermit-comm},
a parametric method for BBoxes construction is proposed,
as illustrated in Fig.~\ref{fig:bboxes} (a).
For each building, its principal axis-aligned footprint $B^{\mathcal{F}}_i \subset \mathbb{R}^2$
is extracted from prior structural data or low-resolution scans.
This footprint is then extruded vertically with a safety margin
$\delta_z \in [\delta_{\text{min}}, \delta_{\text{max}}]$ to form the
cubic bounding box $B_i \triangleq B^{\mathcal{F}}_i \times [z_{\texttt{base}} - \delta_z, z_{\texttt{top}} + \delta_z]$,
where $z_{\texttt{base}}$ and $z_{\texttt{top}}$ denote the structure's elevation bounds.
The complete search space is defined as $\mathcal{B} \triangleq \bigcup_{i=1}^N B_i \subset \mathcal{W}$,
effectively encapsulating all surface-adjacent subspaces within $\epsilon$-neighborhoods of building envelopes.
This construction enables robotic fleets to execute surface-normal trajectory planning
while avoiding computationally prohibitive full 3D reconstructions.
For scenarios lacking prior structural knowledge,
the system initiates an adaptive bounding box generation process,
as detailed in Section~\ref{subsec:w/o prior}.
Then, a {predictor} for task completion time is first proposed to estimate the time needed for the explorers
to explore each BBox~$B_i \in \mathcal{B}$,
such that GCS can schedule the next communication event.
Then, an exploration algorithm {named~$\texttt{FF3E}(\cdot)$} is proposed for the explorers to
arrive at the meeting location~$p^{\texttt{c}}_m$ at time~$t^{\texttt{c}}_m$,
thereby minimizing the waiting time~$\tau_{0}$.
Lastly,
an {algorithm} is proposed for the GCS to
coordinate the communication events with the explorers,
to reduce the waiting time~$\tau_{i}^+$.

\begin{figure}[!t]
    \centering
    \includegraphics[width=1.0\linewidth]{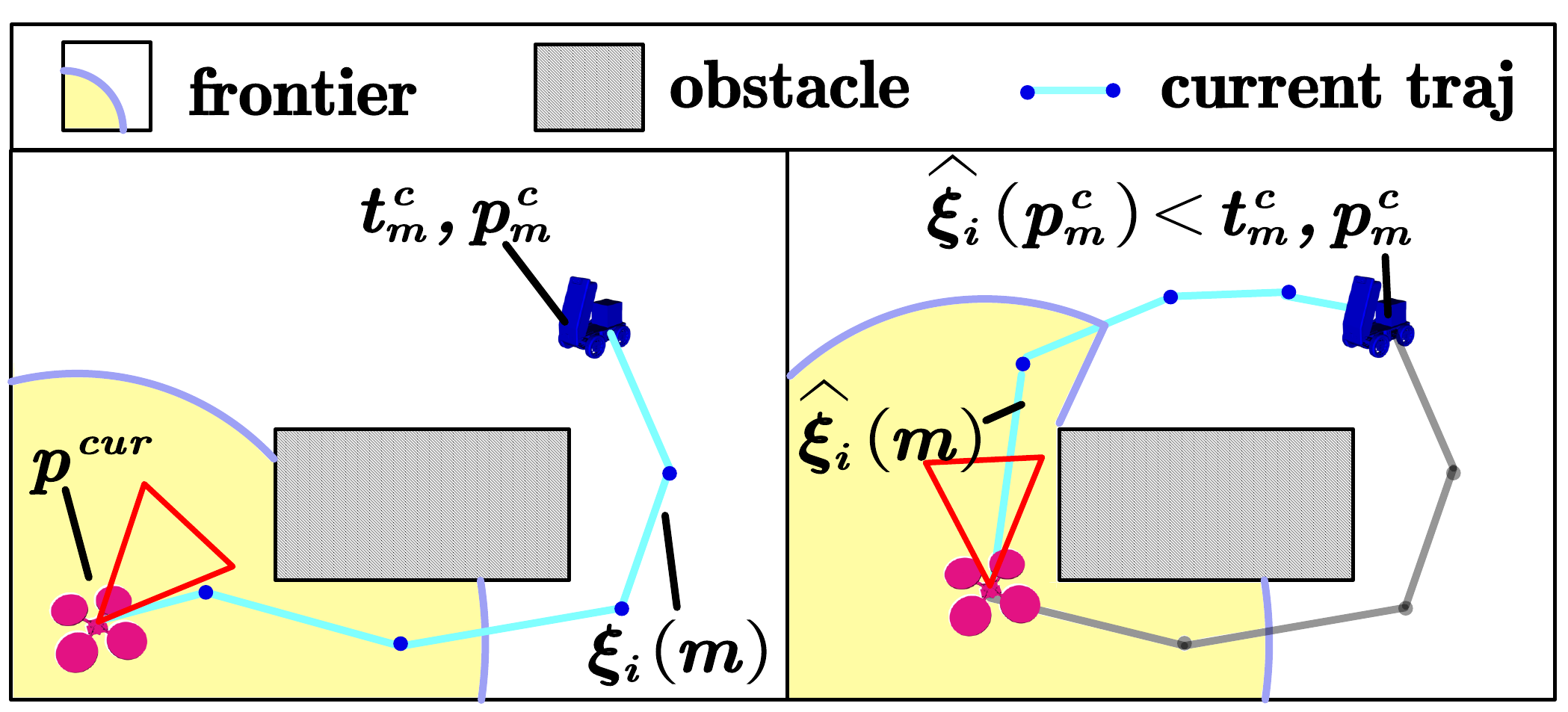}
    \vspace{-2mm}
    \centering
    \caption{Fast frontier-based exploration and local adaptation
      of the exploration strategy~$\{\mathbf{p}^t_i\}$
      by the explorer.}
    \label{fig:explore}
    \vspace{-4mm}
  \end{figure}

\subsubsection{Prediction of Task Completion Time
and Planning for Next Communication Event}
\label{subsec:time-pred}
An explorer~$i\in \mathcal{N}_{\texttt{e}}$ first explores for
a user-defined time~$t^{\texttt{e}}_{i_{m}}<E_i$,
during which it fits a set of features~$\mathcal{F}_{i_m}$ by~\eqref{eq:robot-fit},
and receives the inspection results of features~$\mathcal{S}_{i_{m}}$.
Then,
explorer~$i$ estimates the time needed to complete the exploration of the entire BBox~$B_i$,
which serves as the next communication time~$t^{\texttt{c}}_{m+1}$, i.e.,
{the predictor for task completion time} as follows:


\begin{equation}\label{eq:estimate-duration}
    \begin{aligned}
    & t^{\texttt{c}}_{m+1} = \left(\frac{V_{B_i}}{|\mathcal{S}_{i_{m}}|}\right) \cdot \left(\frac{|\mathcal{F}_{i_{m}}|}{V^{m}_{B_i}} \right) \cdot t^{\texttt{e}}_{i_m}, \\
    \end{aligned}
\end{equation}
where~$\mathcal{S}_{i_{m}} \subset \mathcal{S}$ is the set of inspection results,
received by explorer~$i$ during the~$m$-th communication;
$V_{B_i} > 0$ is the total volume of BBox~$B_i$,
while~$V^{m}_{B_i}$ represents the volume that has been explored
by explorer~$i$ by the~$m$-th communication.
If no features are received, i.e., $\mathcal{S}_{i_{m}}=\emptyset$,
the communication time~$t^{\texttt{ideal}}_{m+1}$ is approximated
by the ratio of explored volume, i.e., $V_{B_i}t^{\texttt{e}}_{i_m}/V^{m}_{B_i}$.
In addition,
the location~$p^{\texttt{c}}_{m+1}$ is chosen among the four corner points of the BBox~$B_i$ to the current location of explorer~$i$,
such that the estimated arrival time by the~$A^\star$ algorithm
following the navigation module~$\texttt{Navi}(\cdot)$ in~\eqref{eq:robot-nav}
is closest to the newly planned communication time~$t^{\texttt{c}}_{m+1}$.
Note that other prediction algorithms than~\eqref{eq:estimate-duration}
can be integrated into the proposed framework, e.g., regression of existing data.

\begin{remark}\label{remark:time-pred}
    Different from the work~\cite{tian2024ihero}
    where the GCS communicates with all explorers in every round,
    the proposed method schedules communications based on the predicted time of task completion,
    thus eliminating the need for the GCS to communicate with all explorers in a predefined order.
    \hfill $\blacksquare$
\end{remark}

\subsubsection{Fast Frontier-based 3D Exploration}
\label{subsec:explore}
With a slight abuse of notation,
the current local \emph{plan} of explorer~$i$ is given by:
\begin{equation}\label{eq:local-plan}
\xi_i(m) \triangleq \mathbf{p}^1_{i}a^1_{i}\cdots
\mathbf{p}^t_{i}a^t_{i} \cdots {p}^c_m a_i^m, \quad t \le t^{\texttt{c}}_m,
\end{equation}
where~${p}^{\texttt{c}}_m, t^{\texttt{c}}_m$ are the \emph{confirmed}~$m$-th communication between the GCS and explorer~$i$;
$a_i^m$ is the communication action with GCS;
and~$\mathbf{p}^t_{i}, a_i^t$ are the waypoints and actions of exploration,
which should be optimized to maximize the explored workspace within the communication time~$t^{\texttt{c}}_m$.
As illustrated in Fig.~\ref{fig:explore},
More specifically,
denote by~$\widehat{\xi}_i(m) \triangleq
\widehat{\mathbf{p}}^1_ia^1_i \cdots \widehat{\mathbf{p}}^t_i a^t_i \cdots
\widehat{\mathbf{p}}^{\texttt{c}}_m a^{\texttt{c}}_m$
as the {revised} local plan,
where~$\{\widehat{\mathbf{p}}^t_i\}$ are the updated waypoints
and~$\widehat{\mathbf{p}}^{\texttt{c}}_m$ are the newly-added waypoints
to the meeting location~$p^{\texttt{c}}_m$.
Thus, the waiting time for explorer~$i$ is given
by~$\tau^{+}_i = \widehat{\xi}_i(p^{\texttt{c}}_m) - t^{\texttt{c}}_m$,
where~$\widehat{\xi}_i(p^{\texttt{c}}_m)$ is the estimated time of arriving at~$p^{\texttt{c}}_m$.

\begin{problem}\label{prob:3D-explore}
    Determine the local plan~$\widehat{\xi}_i(m)$
    such that:
    (I) the explored area is maximized
    within the communication time~$t^{\texttt{c}}_m$;
    (II) the waiting time~$\tau^+_i$ is minimized.
    \hfill  $\blacksquare$
\end{problem}

\begin{algorithm}[!t]
    \caption{Fast Frontier-based 3D Exploration: $\texttt{FF3E}(\cdot)$}
    \label{alg:explore}
    \KwIn{Current position \( p_i \), Frontiers set \( \Gamma \), Meeting location \( p^{\texttt{c}}_m \),
          Meeting time \( t^{\texttt{c}}_m \), Robot speed \( v_i \)}
    \KwOut{Optimal path \( \widehat{\xi}^{\star}_i \)}

    \SetKwFunction{FMain}{planPath}
    \SetKwProg{Fn}{Function}{:}{}\label{alg-line:sub-func}
    \Fn{\FMain{$p_0, \ \widehat{\xi}', \ T^{+}$}}
        {
        \If{$T^{+} + T_{p^{\texttt{c}}_k} > t^{\texttt{c}}_k$}{  
            \Return $\widehat{\xi}'$ \;  \label{alg-line: return-path}
        }

         \(\widehat{\xi}^{\star} \gets \widehat{\xi}' \); \label{alg-line: initial} 

        \For{each point \( P \in \Gamma \setminus \widehat{\xi}'$}
        {
            \label{alg-line:for-loop}
            \( D \gets \texttt{GetDistance}(p_0,\, P) \) \; \label{alg-line:subfunc-distance}
            \( T^{P}_{p_0} \gets {D}/{v_i} \) \; \label{alg-line:subfunc-time}
            \( \widehat{\xi}' \gets \widehat{\xi}' + [P] \) \;
            \( T^{+} \gets T^{+} + T^{P}_{p_0} \) \;
            \( \widehat{\xi}^+ \gets \FMain(P, \widehat{\xi}', T^{+}) \) \; \label{alg-line:recursive-call} 

            \If{ \(\texttt{dist}(\widehat{\xi}^+) >
            \texttt{dist}(\widehat{\xi}^{\star}) \) }
            {
                \label{alg-line:if-condition}
                \( \widehat{\xi}^{\star} \gets  \widehat{\xi}^+ \) \; \label{alg-line:update-best} 
            }
        }
        \Return $\widehat{\xi}^{\star}$  \;  
    }

    \( D^{\texttt{c}}_m \gets \texttt{GetDistance}(p_i,\, p^{\texttt{c}}_m) \) \; \label{alg-line:distance}
    \( T_{p^{\texttt{c}}_m} \gets {D^{\texttt{c}}_m}/{v_i}\) \; \label{alg-line:time}
    \( T' \gets t^{\texttt{c}}_m -  T_{p^{\texttt{c}}_m}\) \; \label{alg-line:time-aval}

    \If{$T' > 0$}{
        \label{alg-line:decide}
        \( \widehat{\xi}^{\star}_i \gets \FMain(p_i, [], 0) \) \;  \label{alg-line:plan-path}
        \( \widehat{\xi}^{\star}_i \gets \widehat{\xi}'_i + [p^{\texttt{c}}_m] \) \; \label{alg-line:plan-path1} 
    }

    \Return \( \widehat{\xi}^{\star}_i \) \;  \label{alg-line:empty}
\end{algorithm}

To begin with, it is worth noting that these two objectives above
are often conflicting.
Existing algorithms~\cite{zhou2021fuel}
based on TSP or TSP-TW
can not be directly applied as not all frontiers have to be visited.
Therefore, the proposed solution
under the module~$\texttt{Explore}(\cdot)$
in~\eqref{eq:robot-explore} is summarized in Alg.~\ref{alg:explore}.
At each round~$m$,
explorer~$i$ first checks whether it has enough time~$t^{\texttt{c}}_m$ to
reach the meeting location~$p^{\texttt{c}}_m$ in Lines~\ref{alg-line:distance}-\ref{alg-line:time-aval}.
If $T' > 0$ does not hold, then~$\emptyset$ is returned in Line~\ref{alg-line:empty}.
Otherwise,
the algorithm proceeds to find the optimal waypoints~$\widehat{\xi}^{\star}_i$
via the function~$\texttt{planPath}(\cdot)$ in Line~\ref{alg-line:plan-path}-\ref{alg-line:plan-path1}.
More specifically,
it checks if the total time spent
plus the time needed to reach~$p^{\texttt{c}}_m$
exceeds the agreed meeting time~$t^{\texttt{c}}_m$ in Line~\ref{alg-line:plan-path}.
If so,
this function returns the current path~$\widehat{\xi}'$ in Line~\ref{alg-line: return-path}.
Otherwise,
it initializes the optimal path~$\widehat{\xi}^{\star}$
as the current path~$\widehat{\xi}'$ in Line~\ref{alg-line: initial}.
Then,
it iterates through each frontier~$P\in \Gamma$ in Line~\ref{alg-line:for-loop},
by computing the distance~$D$ and the travel time~$T^{P}_{p_0}$,
from its position~$p_0$ to~$P$ in Lines~\ref{alg-line:subfunc-distance}-\ref{alg-line:subfunc-time}.
Afterwards, the set of visited frontiers~$\widehat{\xi}'$
is updated along with the total time spent~$T^{+}$,
before calling $\texttt{planPath}(\cdot)$ again
in Line~\ref{alg-line:recursive-call}.
Lastly,
if the resulting path~$\widehat{\xi}^+$ is longer
than~$\widehat{\xi}^{\star}$,
the optimal path~$\widehat{\xi}_i^{\star}$ is updated accordingly
in Lines~\ref{alg-line:if-condition}-\ref{alg-line:update-best}.
Lastly,~$\widehat{\xi}^{\star}_i$ from Alg.~\ref{alg:explore}
is returned as the optimized solution for Problem~\ref{prob:3D-explore}.

The time complexity of Alg.~\ref{alg:explore} can be analyzed step by step based on its recursive structure.
During initialization in Lines~\ref{alg-line:distance}-~\ref{alg-line:time-aval},
it mainly computes the distance and time to reach the meeting location~$p^{\texttt{c}}_m$,
thus the complexity is~$\mathcal{O}(1)$.
Then, during the recursive path planning,
function~$\texttt{planPath}(\cdot)$ explores
all possible frontier points~$P \in \Gamma$
by calling itself recursively in Line~\ref{alg-line:recursive-call}.
The worst-case time complexity is~$\mathcal{O}(|\Gamma|!)$,
where~$|\Gamma|$ is the total number of frontiers.
{Lastly,
to update the optimal path in Line~\ref{alg-line: return-path}
has a complexity of~$\mathcal{O}(1)$.}
Therefore, the overall complexity of Alg.~\ref{alg:explore}
is given by~$\mathcal{O}(|\Gamma|!)$.

\begin{figure}[!t]
    \centering
    \includegraphics[width=1.0\linewidth]{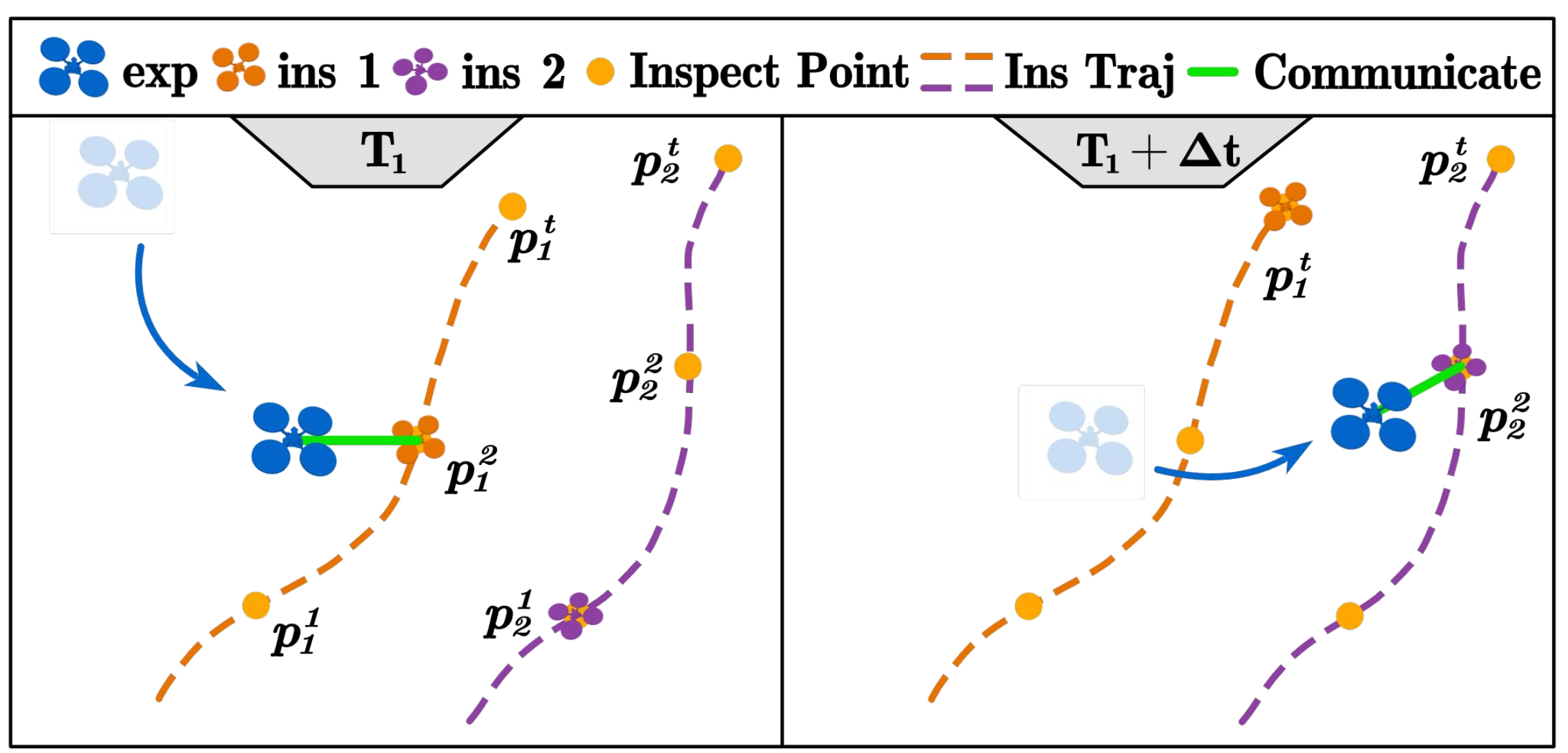}
    \vspace{-3mm}
    \centering
    \caption{Proactive communication between an explorer (in blue)
      and the inspectors (in yellow) within a subgroup.}
    \label{fig:proactive-comm}
    \vspace{-4mm}
  \end{figure}

\subsubsection{Communication Coordination for GCS}
\label{subsec:patrol}
To minimize the idle time~$\tau_0$ for GCS,
it is sufficient to arrive at the meeting location~$p^{\texttt{c}}_m$ within
the predefined time~$t^{\texttt{c}}_m$.
Therefore,
a TSP-TW problem is formulated for
GCS to optimize the sequence of meetings with the explorers
in each communication~$m \in M$.
Detailed description is omitted due to limited space.
In summary,
the subproblem of coordinating GCS and explorers to minimize
the idle time~$\tau_i$ and~$\tau_{0}$ is solved by the proposed
prediction algorithm of task completion time,
3D exploration algorithm and the intermittent communication protocol.
  It is worth noting that our framework does not require explicit calibration of the global localization between the GCS and explorers,
  via external positioning systems such as GPS or RTK.
  In addition,
  any pre-calibrated 3D coordinate in the exploration area could serve
  as absolute references for GCS-to-explorer communications
  without additional calibration.

\begin{lemma}\label{lemma:gcs-exp}
    Under the proposed coordination strategy for the layer of GCS-to-SubGroups,
    all features~$\mathcal{F}$ in the BBoxes
    can be fitted by explorers in finite time,
    of which the results can be collected
    by the GCS in finite time.
\end{lemma}

\begin{proof}\label{proof:gcs-exp}
    During exploration,
    the explorers explore and fit features~$\mathcal{F}_{i_m}$
    via Alg.~\ref{alg:explore}.
    The planning module~$\texttt{planPath}(\cdot)$
    in Line~\ref{alg-line:sub-func}
    ensures that all frontiers
    within a BBox~$B_i$ are visited and the whole BBox is fully explored.
    Since the number of frontiers~$|\Gamma|$ in each BBox is finite,
    the total time required to explore any BBox is also limited.
    Furthermore,
    all BBoxes can be fully explored by the explorers
    via the rolling assignment algorithm.
    Thus, all features~$\mathcal{F}$ in the environment can be fitted
    in finite time.
    On the other hand,
    the communication protocol schedules the next communication event~$(p^{\texttt{c}}_m, t^{\texttt{c}}_m)$,
    for which the explorers would adhere by Line~\ref{alg-line:decide}.
    {Similarly,
    the GCS ensures a timely arrival at the meeting event
    via solving the TSP-TW problem.}
    During each communication,
    explorers can transmit all detected features to GCS.
    Since both the number of features and communication rounds are finite,
    the GCS can collect all results in finite time.
\end{proof}

\subsection{Communication-aware  Exploration and Inspection}
\label{sec:inspect}

To minimize the idle time within the subgroup,
a simultaneous exploration and inspection algorithm is proposed
under the limited communication conditions.

\subsubsection{Subproblem Formulation}
\label{subsec:inspect-prob}

Different from the scheme of intermittent communication,
a novel proactive communication protocol
under the module~{$\texttt{Comm}_{ij}(\cdot)$} in~\eqref{eq:robot-com},
is designed for the explorer~$i\in \mathcal{N}_{\texttt{e}}$
and the inspectors~$\mathcal{N}^i_{\texttt{o}}\triangleq \{1,\cdots,K\}
\subset \mathcal{N}_{\texttt{o}}$
within the subgroup.
The protocol has two steps:
(I) the explorer~$i$ communicates with the inspectors in~$\mathcal{N}^i_{\texttt{o}}$
under the LOS constraints;
(II) during each communication,
the explorer determines not only the time and location
of the next meeting
with the inspectors under the energy constraint,
but also assigns the fitted features to the inspectors.
i.e., the features~$\widehat{\mathcal{F}}_{j} \subset \mathcal{F}$ to inspect
and the accordingly updated local plan~$\xi_{j}^+$ of each inspector~$j\in \mathcal{N}^i_{\texttt{o}}$.
As illustrated in Fig.~\ref{fig:proactive-comm},
this procedure is repeated until all features are fitted and inspected.


\begin{definition}[Plan of Subgroup]\label{def:sub-sol}
    The overall plan of the subgroup is defined
    as a 4-tuple:
    \begin{equation}\label{eq:overall-plan}
    \Xi_i \triangleq (\widehat{\mathcal{N}}_i,\,
    \mathbf{c}_i,\, \boldsymbol{\varphi}_i,\, \boldsymbol{\xi}_i),
    \end{equation}
    where~$\widehat{\mathcal{N}}_i\subset \mathcal{N}^i_{\texttt{o}}$
    is the subset of inspectors to communicate with;
    $\mathbf{c}_i\triangleq \{(t_j,\,p_j),
    \forall j\in \widehat{\mathcal{N}}_i\}$
    is the set of meeting time and location
    for each inspector in~$\widehat{\mathcal{N}}_i$;
    $\boldsymbol{\varphi}_i\triangleq
    \{\widehat{\mathcal{F}}_j,
    \forall j\in \widehat{\mathcal{N}}_i\}$
    is the set of new features allocated to each inspector,
    given the set of fitted features~$\mathcal{F}^+_i$;
    and $\boldsymbol{\xi}_i\triangleq
    \{\xi_j^+, \forall j\in \widehat{\mathcal{N}}_i\}$
    is the set of updated local plans for all inspectors.
    \hfill $\blacksquare$
\end{definition}
\begin{problem} \label{prb:inspect}
    Determine the plan of the subgroup~$\Xi_i$ as defined in~\eqref{eq:overall-plan}
    such that the total idle time of the explorer and inspectors within the subgroup
    is minimized, namely:
    \begin{equation}\label{eq:obj-inspect}
    \underset{\Xi_i}{\textbf{min}}\Big{\{} \tau_i
    +\sum_{j\in \mathcal{N}^i_{\texttt{o}}} (\tau^{+}_j+\tau^{-}_j)\Big{\}},
    \end{equation}
    where~\(\tau^+_j\) is the idle time of inspector~$j\in \mathcal{N}^i_{\texttt{o}}$
    due to waiting;
    and \(\tau^-_j\) is the remaining idle time excluding \(\tau^+_j\).
    \hfill $\blacksquare$
\end{problem}


\subsubsection{Efficient 3D Exploration and Inspection under Limited Communication}
\label{subsec:sei}

\begin{figure}[!t]
  \centering
  \includegraphics[width=0.98\linewidth]{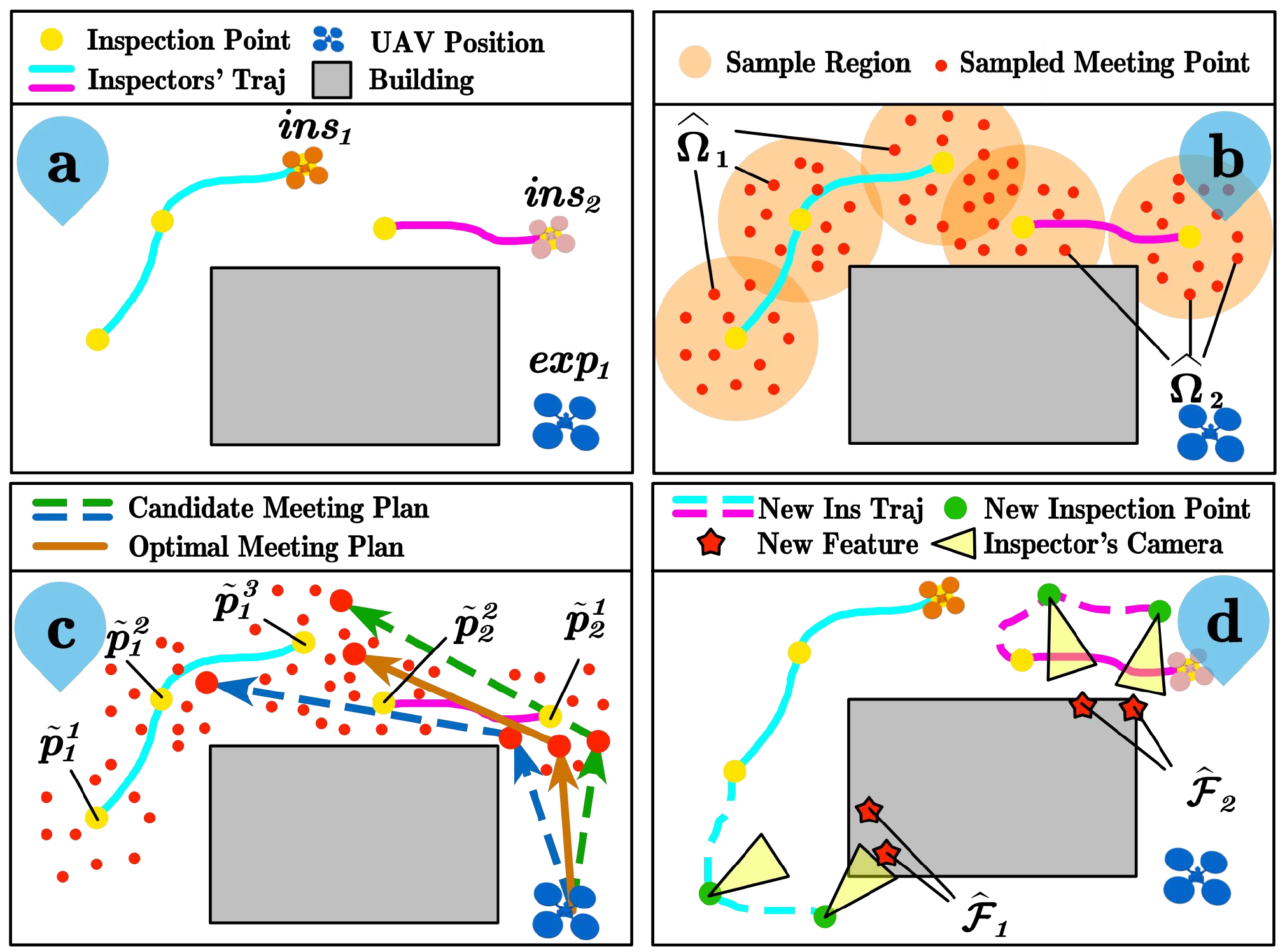}
  \vspace{-1mm}
  \centering
  \caption{Main components of the proposed SOEI algorithm in Alg.~\ref{alg:SOEI}:
   (\textbf{a}) the initialization of a subgroup including the explorer (in blue)
    and the inspectors (in yellow);
   (\textbf{b}) the sampling of candidate meeting locations~$\{p_j\}$ (in red);
   (\textbf{c}) the optimization of meeting sequence~$\mathbf{c}_i$
  including meeting time and locations (in brown);
  and (\textbf{d}) the allocation of fitted features~$\mathcal{F}^{+}_i$ (in red stars)
  to the inspectors~$\widehat{\mathcal{N}}_i$.}
  \label{fig:SOEI}
  \vspace{-5mm}
\end{figure}

An efficient exploration and inspection algorithm encapsulated as~$\texttt{SOEI}(\cdot)$
is proposed to solve the above problem.
As illustrated in Fig.~\ref{fig:SOEI} and summarized in Alg.~\ref{alg:SOEI},
it consists of three main parts as described in the sequel.

(I) {Choice of inspectors}.
The optimal subgroup solution~$\Xi^{\star}_i$
and the idle time~$\tau_{\Xi^{\star}}$ are initialized in Line~\ref{alg-line2:sub-group-sol}.
Then,
the algorithm iterates over each subset~$\widehat{\mathcal{N}}_i \subset \mathcal{N}^i_{\texttt{o}}$,
in order to evaluate the associated idle time.
Note that if no inspectors are selected,
the explorer would continue exploration.
(II) {Meeting Sequence}.
A sampling-based algorithm is proposed to determine the optimal sequence of
meeting events between the explorer and the inspectors in~$\widehat{\mathcal{N}}_i$.
{A set of sample points~$\widehat{\Omega}_j$} are generated around the executing local plan~$\{\xi^-_j\}$ of each inspector,
subject to the LOS and range communication constraints.
Denote by~$\widehat{\mathbf{c}}'\triangleq \{\mathbf{c}'_i\}$
a set of potential sequences of meeting events,
{where each~$\mathbf{c}'_i$ is formed by selecting one sample
  from~$\widehat{\Omega}_j$ for each inspector
and traversing the samples in a timed order}.
Afterwards, the optimal trajectory for the explorer to visit these samples can be obtained
via the navigation module~$\texttt{Navi}(\cdot)$ in~\eqref{eq:robot-nav},
with adaptive velocity to minimize the total idle time~$\tau_{\Xi^+} = (\tau_i+\sum_{j\in\widehat{\mathcal{N}}_i} \tau^+_j)$.
As shown in Line~\ref{alg-line2:optimization}, it generates the actual schedule of
meeting events~$\mathbf{c}_i$ and the associated idle time~$\tau_{\Xi^+}$
via solving the following optimization:

\begin{subequations}\label{eq:optmeet}
    \begin{align}
        & {\textbf{min}}_{\mathbf{c}'_i} \quad \tau_{\Xi^+} \notag \\
      \textbf{s.t.} & \quad \tau^{-}_i = \texttt{A}^{\star}(\{p_j\}),
      \,\tau^{+}_i = t_j - \xi^{-}_j(p_j);  \label{eq:optmeet:a-b} \\
      & \quad \tau^{+}_{j} = t_j - t^{\texttt{e}_{-}}_j,\, \forall j \in
      \widehat{\mathcal{N}}_i, \, p_j \in \widehat{\Omega}_j, \label{eq:optmeet:c-d}
    \end{align}
\end{subequations}
where~$t^{\texttt{e}_{-}}_j$ is the end timestamp of executing local plan~$\xi^{-}_j$ for inspector~$j$;
$\xi^{-}_j(t_j)$ returns location of inspector~$j$ at time~$t_j$,
while~$\xi^{-}_j(p_j)$ outputs the timestamp for inspector~$j$ at location~$p_j$.
Furthermore,
constraint~\eqref{eq:optmeet:a-b} calculates the travel idle time of the explorer~$i$ using~$A^{\star}$ algorithm and the waiting time of the inspector~$j$;
\eqref{eq:optmeet:c-d} quantifies the waiting idle time of the explorer~$i$.
Since non-analytic constraint~\eqref{eq:optmeet:a-b} cannot be directly applied
in the traditional optimization tools,
thus a genetic algorithm is leveraged by following the standard steps:
the potential sequence~$\mathbf{c}'_i$ is jointly encoded as a chromosome;
the fitness function is defined as the reciprocal of total idle time,
i.e., $\text{Fitness} = 1/{\tau_{\Xi^+}}$;
the constraints~\eqref{eq:optmeet} are enforced by the genetic algorithm;
and the standard genetic algorithm procedure~\cite{alhijawi2024genetic}
is followed to obtain the actual meeting sequence~$\mathbf{c}_i$.
Finally, the total idle time~$\tau_{\Xi^+}$ is returned by~$\texttt{OptMeet}$ algorithm as
Line~\ref{alg-line2:optimization}.
(III) {Allocation of Features}.
Given the confirmed meeting events,
the stored new features~$\mathcal{F}^{+}$ are then assigned to the inspectors
and appended to the end of their local plans.
To determine the optimal assignment and minimize idle time~$\tau_{\Xi^{-}}$,
a MVRP is formulated and solved to obtain~$\boldsymbol{\varphi}_i$,
by setting the end of local plans as initial positions and all features
in~$\mathcal{F}^{+}$ as positions to visit.
Once~$\boldsymbol{\varphi}_i$ is obtained,
the next local plans~$\{\xi^+_j\}$ are updated by interpolating the
shortest path between assigned features,
along with the total idle time~$\tau^{-}_j$.
Lastly,
the total idle time~$\tau_{\Xi}$ is derived by adding~$\tau_{\Xi^{-}}$ and~$\tau_{\Xi^{+}}$ (Line~\ref{alg-line2:total_idle}).
If it is less than the current optimal idle time~$\tau_{\Xi^{\star}}$,
the optimal subgroup solution~$\Xi^{\star}_i$ is updated by this solution
(Line~\ref{alg-line2:update-idle-1}-\ref{alg-line2:update-idle-2}).
  It is important to clarify that our framework does not
  require the intra-group sharing of global coordinates due to the
  relative localization within the same group.
Particularly, when entering a BBox, all group members adopt the entry point as their local origin.
The explorers map the detected features (structural elements, inspection targets) relative to this origin.
{Finally,
the inspectors navigate directly via these relative coordinates without the need for external alignment between different groups.
}

\begin{algorithm}[!t]
  \caption{Simultaneous Optimized Exploration and Inspection
    $\texttt{SOEI}(\cdot)$}
    \label{alg:SOEI}
    \KwIn{$\mathcal{F}^{+}_i$, $\mathcal{N}_{\texttt{o}}^i$, $\{\xi^{-}_j\}$.}
    \KwOut{ \( \Xi^{\star}_i \).}
    $\Xi^{\star}_i \gets \text{None}$,\, $\tau_{\Xi^{\star}} \gets \infty$ \; \label{alg-line2:sub-group-sol}
    \ForEach{$\widehat{\mathcal{N}}_i \subset \mathcal{N}_{\texttt{o}}^i$}{
        {
        $\widehat{\mathbf{c}}' =\texttt{SampleLOS}(\{\xi^{-}_j\})$\; \label{alg-line:sample-los}
        \ForEach{$\mathbf{c}'_i \in \widehat{\mathbf{c}}$}{
        $(\mathbf{c}_i,\, \tau_{\Xi^{+}}) \gets
        \texttt{OptMeet}(\mathbf{c}'_i,\, \{\xi^{-}_j\})$ by~\eqref{eq:optmeet} \; \label{alg-line2:optimization}
        $(\boldsymbol{\varphi}_i,\, \boldsymbol{\xi}_i,\, \tau_{\Xi^{-}})
        \gets \texttt{MVRP}(\mathbf{c}_i, \mathcal{F}^{+}, \{\xi^{-}_j\})$ \; \label{alg-line2:mvrp}
                    $\tau_{\Xi} \gets \tau_{\Xi^{-}} +  \tau_{\Xi^{+}}$ \; \label{alg-line2:total_idle}
                    \If{$\tau_{\Xi} < \tau_{\Xi^{\star}}$}{
                        $\tau_{\Xi^{\star}} \gets \tau_{\Xi}$ \; \label{alg-line2:update-idle-1}
                      $\Xi^{\star}_i \gets (\widehat{\mathcal{N}}_i,\,
                      \mathbf{c}_i,\, \boldsymbol{\varphi}_i,\, \boldsymbol{\xi}_i)$ \; \label{alg-line2:update-idle-2}
                    }
        }
        }
    }
    \Return{$\Xi^{\star}_i$} \;
\end{algorithm}

\subsubsection{Algorithm Summary}
\label{subsec:alg-summ}
It is worth noting that
the number of subsets $\widehat{\mathcal{N}}_i$ is combinatorial to
the size of~${\mathcal{N}}^i_{\texttt{o}}$, i.e., {$\mathcal{O}(2^{|{\mathcal{N}}^i_{\texttt{o}}|})$.}
{To generate all potential meeting sequences~$\widehat{\mathbf{c}}$
has a factorial complexity of~$\mathcal{O}(|\widehat{\mathcal{N}}_i|!)$.
Then,
to determine the optimal sequence of meeting events~$\mathbf{c}_i$,
the genetic algorithm has a complexity of~$\mathcal{O}(|\widehat{\mathcal{N}}_i|^3)$.}
Finally,
the MVRP algorithm to allocate features has a complexity of~$O(|\widehat{\mathcal{N}}_i|^{|\mathcal{F}^{+}|})$.

\begin{lemma}\label{lemma:exp-ins}
    {{
    Under the proposed Alg.~\ref{alg:SOEI} for the layer of SubGroups,
    all features~$\mathcal{F}$ can be inspected,
    and the results~$\mathcal{S}$ are collected by the explorers in finite time.}}
\end{lemma}
\begin{proof}\label{proof:exp-ins}
    {
    During each communication event,
    the explorer assigns the newly discovered features~$\mathcal{F}^{+}$ to the
    selected inspectors~$\widehat{\mathcal{N}}_i$ via the MVRP algorithm
    in Line~\ref{alg-line2:mvrp},
    which ensures that each inspector is allocated a subset of features~$\widehat{\mathcal{F}}_j$.
    Since the number of features~$\mathcal{F}$ is finite,
    the recursive allocation of features guarantees all features
    are eventually assigned to inspectors.
    On the other hand,
    the proactive communication protocol guarantees
    that explorers actively communicate with
    each inspector more than once
    to minimize the idle time due to waiting.
    In addition,
    the inspectors transmit their inspection results~$\mathcal{S}$ to the explorer.
    Since the number of inspectors~$\mathcal{N}_{\texttt{o}}$ and communication events
    are finite,
    the total time needed to collect all inspection results~$\mathcal{S}$ is also finite.
    }
\end{proof}

\subsection{Online Execution and Adaptation}\label{sec:online}

\subsubsection{Rolling Assignment of Exploration Tasks}
\label{subsec:roll}

{Since the GCS is fully aware of the location of all BBoxes~$\mathcal{B}$,
and it coordinates with the explorers directly in the layer of GCS-SubGroups,
it is reasonable for GCS to assign these BBoxes to robotic fleets.}
Moreover,
instead of assigning all BBoxes at once,
a rolling assignment strategy is adopted in a receding horizon way,
which is particularly useful
as the actual structure of BBoxes in~$\mathcal{B}$ is unknown,
making it difficult to predict the exploration time of each BBox.
In the initial phase,
GCS assigns each explorer~$i\in \mathcal{N}_{\texttt{e}}$ to
the nearest BBox~$B_i \in \mathcal{B}$.
Then, a subgroup is formed by allocating a specific number of inspectors
to each explorer,
denoted by~$\mathcal{N}^i_{\texttt{o}} \subset \mathcal{N}_{\texttt{o}}$,
according to the volume~$V_{B_i}$ of the BBox~$B_i$, i.e.,
\begin{equation}\label{eq:bbox-alloc}
    |\mathcal{N}^i_{\texttt{o}}| = \left\lfloor
    |\mathcal{N}_{\texttt{o}}| \frac{V_{B_i}} {\sum_i V_{B_i}} \right\rfloor,
\end{equation}
where~$|\mathcal{N}_{\texttt{o}}|$ is the total number of inspectors;
and~$\lfloor \cdot \rfloor$ represents the floor function.
During online execution,
when GCS and explorer meet at the designated location~$p^{\texttt{c}}_m$,
each explorer reports the current exploration status of~$B_i$ to GCS.
Once~$B_i$ is fully explored and all inspected features are fed back,
the GCS assigns the next nearest BBox to the entire subgroup for execution,
until all BBoxes are assigned,
{as shown in Fig.~\ref{fig:gcs-exp-evol}}.

\subsubsection{Handling Mismatch of Meeting Time}
\label{subsec:mis-handle}

Due to motion uncertainty and tracking errors,
it is possible that actual arrival time at the designated meeting location
is different from the planned meeting time,
especially in unknown or complex environment.
This is called the mismatch of meeting time.
To address this challenge,
a series of mechanisms are proposed:
(I) If the waiting time of GCS for an explorer
exceeds a given tolerance~$\overline{\delta}>0$,
i.e., $\tau_0 \ge \overline{\delta}$,
then GCS assumes that this explorer fails
and proceeds to communicate with the next explorer.
Conversely,
if the waiting time of an explorer for the GCS exceeds~$\overline{\delta}$,
i.e., $\tau^{+}_i \ge \overline{\delta}$,
the explorer has to wait until the GCS arrives;
(II) Within a subgroup,
if the waiting time of an explorer for an inspector
exceeds~$\overline{\delta}$,
the explorer would continue its local plan.
Since the explorer has full knowledge of the local plans of all inspectors,
the next scheduled meeting event is appended to the end of the local plan
of this inspector.
If this tolerance is exceeded again in the next meeting,
the explorer would assume that this inspector has failed,
in which case the failure recovery mechanism is discussed
in the sequel.
Last but not least,
whenever scheduling conflicts may arise that the explorer
should meet with both the GCS and inspectors in close time,
higher priority is given to the GCS.
{
This prioritization is justified because GCS-explorer communication occurs infrequently, whereas explorer-inspector communications happen more regularly, 
allowing remaining data to be deferred to subsequent meetings.
}

\subsubsection{Local Adaptation during Exploration}
\label{subsec:adapt-exp}

During 3D online exploration,
although the optimal path~$\widehat{\xi}^{\star}_i$
of each explorer~$i\in \mathcal{N}_{\texttt{e}}$
is synthesized
given the existing frontiers via Alg.~\ref{alg:explore},
new frontiers may be generated during the exploration process,
yielding the need for local adaptation.
Thus, to address this issue,
Alg.~\ref{alg:explore} can be triggered in a receding horizon manner
by adding new frontiers to the existing set and replanning the path.
This iterative re-planning process continues until the explorer
reaches the confirmed meeting location~$p^{\texttt{c}}_m$ within~$t^{\texttt{c}}_m$.

\begin{theorem}\label{theory:gcs}
{
    The proposed SLEI3D framework
    ensures that all inspected features~$\mathcal{S}$ are collected by the GCS
    in a finite time,
    while minimizing the total idle time in~\eqref{eq:new-objective}.}
\end{theorem}

\begin{proof}\label{proof:exp-ins}
    {
    Lemma~\ref{lemma:gcs-exp} guarantees all features~$\mathcal{F}$ are fitted by explorers in finite time,
    while Lemma~\ref{lemma:exp-ins} ensures all features~$\mathcal{F}$ are allocated,
    inspected and transmitted to explorers by inspectors within finite time.
    Under the protocol of intermittent communication,
    the GCS solves the TSP-TW problem to arrive at the meeting location~$p^{\texttt{c}}_m$ at the scheduled time~$t^{\texttt{c}}_m$,
    where the explorers transmit the inspected features~$\mathcal{S}$.
    Since both the number of features~$\mathcal{F}$ and communication rounds are finite,
    the GCS collects all inspection results~$\mathcal{S}$ within finite time.
    Furthermore, the strategy of online adaptation can effectively
    mitigate delays and failures caused by uncertainties.
    Therefore, the SLEI3D framework guarantees that all
    features~$\mathcal{F}$ are collected back by GCS in a finite time.
    On the other hand,
    as re-formulated in Problem~\ref{prob:intermit-comm}
    algorithms of GCS-to-SubGroups layer coordinate the meeting events
    between GCS and explorers, to minimize the waiting time.
    As formulated in Problem~\ref{prb:inspect},
    algorithm~$\texttt{SOEI}(\cdot)$
    coordinates the explorer and inspectors in each subgroup,
    to minimize the idle time.
    Thus, given the set of currently-known features,
    the overall plan of the robotic fleet and
    the GCS minimizes the total idle time
    in~\eqref{eq:new-objective} and maximizes
    the efficiency to collect explored and inspected features.
    This completes the proof.
    }
\end{proof}

\subsection{Generalization}\label{sec:discuss}

\subsubsection{Robot Failures and Recovery}
\label{subsec:fail-recover}

Consider the following cases:
(I) Failure of an explorer:
Assume that explorer~$i$ fails at~$t_{\texttt{f}} > 0$,
the GCS would wait for the explorer~$i$ at the predefined time~$t^{\texttt{c}}_m$ and location~$p^{\texttt{c}}_m$
for a maximum waiting time~$\overline{\delta}$.
Then GCS determines that explorer~$i$ has failed and {directly} moves to next explorer
in its local plan.
Subsequently, the meeting events of next rounds are scheduled
without considering the failed explorer.
In addition,
once another subgroup has finished exploring its assigned BBox,
it will be reassigned to the BBox of the failed explorer
to restart the task of exploration and inspection,
along with the remaining inspectors in~$\mathcal{N}^i_{\texttt{o}}$.
(II) Failure of an inspector:
Assume that inspector~$j \in \mathcal{N}^i_{\texttt{o}}$
has failed at~$t_{\texttt{f}} > 0$,
explorer~$i$ within the same subgroup
would detect its failure {after waiting for a maximum duration~$\overline{\delta}$}.
Then, explorer~$i$ would exclude inspector~$j$ from the subgroup
by updating~$\mathcal{N}^i_{\texttt{o}} \leftarrow \mathcal{N}^i_{\texttt{o}} \setminus \{j\}$,
after which Alg.~\ref{alg:SOEI} is applied to reassign the remaining inspectors.

    Beyond the above measures,
    communication instability between robots (GCS-explorer or explorer-inspector)
    can be addressed via redundancy.
    While the baseline mechanism identifies failures after prolonged period of disconnections
    over $\overline{\delta}$,
    this might overreact to transient disruptions.
    Thus, a recovery mechanism at the hardware level is proposed, i.e.,
    each robot carries an ad-hoc mesh network that can be activated during communication disruption.
    This network can immediately relay critical status updates
    while maintaining local coordination.
    This integrated approach ensures resilience against both permanent failures of robots
    and temporary degradation of the communication network.

\subsubsection{Multiple GCS}
\label{subsec:multi-gcs}
Multiple GCS units are available, such as in large-scale scenes.
In this case,
the GCS is assumed to have similar capabilities and
all-time connectivity with each other,
e.g., regarding the progress of exploration and inspected features.
Regarding the communication with explorers,
since a subset of explorers is pre-assigned to communicate with a specific GCS,
the intermittent communication protocol in Sec.~\ref{sec:gcs-explore}
should be modified by replacing the single TSP-TW with the
multiple parallel TSP-TW formulations,
to generate a local plan of navigation and communication for each GCS,
thus improving overall efficiency.

\subsubsection{High-priority Features}
\label{subsec:high-pri-feat}
If the features have different priorities,
e.g., some features may require immediate response,
the feature allocation module in Alg.~\ref{alg:SOEI}
can be modified by adding precedence constraints
to the inspection tasks associated with high-priority features.
Namely, algorithms similar to the MVRP with priority or
precedence constraints (MVRP-PC) should be employed.
Thus, high-priority features can be inspected
first while minimizing the total idle time.

\subsubsection{Spontaneous Meeting}
\label{subsec:spont-meet}
During execution, the GCS, explorers or inspectors,
might meet
at locations other than the confirmed meeting locations,
which are called {spontaneous} meetings.
In this case,
they exchange their local data, without coordinating
the next meeting event nor modifying their local plans,
such that all confirmed meetings remain valid.



\subsubsection{Fully-unknown Environment}
\label{subsec:w/o prior}
When the prior information regarding BBoxes is unavailable,
an adaptive partitioning method for BBoxes is proposed to dynamically allocate the unexplored regions in $\mathcal{W}$
based on the real-time progress of collaborative exploration.
{It should be noted that the method relies on prior information about the environmental shape and size, 
where oddly-shaped workspaces are first enclosed by a regular bounding box before partitioning.}
As shown in Fig.~\ref{fig:bboxes} (b),
the whole space~$\mathcal{W}$ is initially divided into $|\mathcal{N}_{\texttt{e}}|$
cubic subspaces with equal volume denoted by~$\{B_j\}$.
Then, each subgroup autonomously explores its assigned~$B_j$ via the framework proposed above,
the completion of which is notified to the GCS.
The GCS keeps track of the set of unexplored frontiers through $\Psi \triangleq \mathcal{W} \setminus \bigcup_{j \in \mathcal{I}_{\texttt{comp}}} B_j$,
where $\mathcal{I}_{\texttt{comp}}$ is the set of explored regions.
{This set $\Psi$ is dynamically partitioned into new set of BBoxes via iterative octree partitioning
that preserves the alignment of all axes.}
This iteration terminates when $\Psi = \emptyset$ or $\min(\text{dim}(B_j)) \leq 10.0\text{m}$.
Subsequent inspections follow the procedure described in Alg.~\ref{alg:SOEI},
which ensures a complete coverage through hierarchical decomposition.

\subsubsection{Energy-constrained Fleet Management}
\label{subsec:energy}
When the robots have limited energy capacity and require charging,
our framework can be adopted by applying modifications across all layers.
To begin with,
the robot models should incorporate these constraints.
Namely, each robot \( i \in \mathcal{N} \) has a limited energy capacity \( E_i(t) \le \overline{E} \in \mathbb{R}^+ \),
which evolves by \( \dot{E}_i(t) = -\alpha_i \) with \(\alpha_i > 0\)
during operation, and \( \dot{E}_i(t) = \beta_i \) with \(\beta_i > 0\)
when charging at the designated station \( P_{\texttt{chg}} \subset \mathcal{W} \).
The charging process is triggered when \( E_i(t) \le \underline{E} \in \mathbb{R}^+ \)
and the robot is at the charging station,
with a duration~\( t_{\texttt{b}} = (\overline{E} - \underline{E})/\beta_i \)
for a full charge.
This constraint is formally stated as:
$0 < E_i(t) \le \overline{E},\; \forall t \in [0,\,T],\;
\forall i\in \mathcal{N}$.

To address the above constraint,
the proposed framework is adapted with the following three key modifications:
{(I) The prediction of completion time for an assigned BBox as presented in Sec.~\ref{subsec:time-pred}
must account for the duration of recharging.}
Consequently, the planned communication time~$\widetilde{t}^{\texttt{c}}_{m+1}$ between
the GCS and an explorer is modified as follows:
\begin{equation}\label{eq:duration-energy}
    \widetilde{t}^{\texttt{c}}_{m+1} = t^{\texttt{c}}_{m+1} + \left\lceil \frac{t^{\texttt{c}}_{m+1} - t^{\texttt{c}}_m}{\overline{E} - \underline{E}} \right\rceil \cdot (t_b + T_{\texttt{chg}}),
\end{equation}
where $\lceil\cdot\rceil$ calculates the minimum recharge cycles,
$T_{\texttt{chg}}$ is the round-trip travel time from one of the four corner points
to the charging station $P_{\texttt{chg}}$ and $t_b$ is the maximum charging duration;
(II) During the coordination of GCS-to-SubGroups as described in Sec.~\ref{subsec:patrol},
the planned meeting time $t^{\texttt{c}}_m$ between the GCS and explorers should also be modified under the following cases:
(i) when the remaining energy \( (E_0- \underline{E}) /\alpha \geq \Delta t_m \),
where \( \Delta t_m = t^{\texttt{c}}_m - t^{\texttt{c}}_{m-1} \) denotes the time interval between two consecutive meetings,
the planned \( t^{\texttt{c}}_m \) remains feasible under the energy constraint.
(ii) if \( (E_0 - \underline{E})/\alpha < \Delta t_m \), the GCS needs an intermediate charging, for which the meeting time is adjusted by:
\[
\tilde{t}^{\texttt{c}}_m = t^{\texttt{c}}_m + \left\lceil \frac{\Delta t_m - E_0}{\overline{E} - \underline{E}} \right\rceil \cdot (T^{0}_{\texttt{chg}} + t_b),
\]
where~\( T^{0}_{\texttt{chg}} \) denotes the round-trip travel time between the meeting point \( p^{\texttt{c}}_m \) and the charging station;
(III) During the coordination within each SubGroup,
the initial assignment $\phi_i$ computed via solving the MVRP as described in Sec.~\ref{subsec:sei}
can be reformulated to generate an energy-unconstrained routing sequence $\mathcal{F}^+$ for each inspector,
i.e., by predicting the energy consumption along the local plan of each inspector
and inserting charging events as needed.
The planned timestamps for subsequent inspection events are shifted accordingly.


\section{Numerical Experiments}
\label{sec:experiments}

To further validate the effectiveness of the proposed method,
extensive numerical experiments for large-scale systems are conducted,
of which the performance is compared against several state-of-the-art methods.
The proposed method is implemented in~$\texttt{Python3}$
within the framework of~$\texttt{ROS}$,
and tested on a workstation with Intel(R) i9-13900KF 24-Core CPU
@3.0GHZ with a RTX-4090 GPU.
Simulation videos can be found in the supplementary material.

\subsection{System Description}
\label{subsec:sys-desc}

\begin{figure*}[!t]
  \centering
  \includegraphics[width=0.92\linewidth, height=0.3\linewidth]{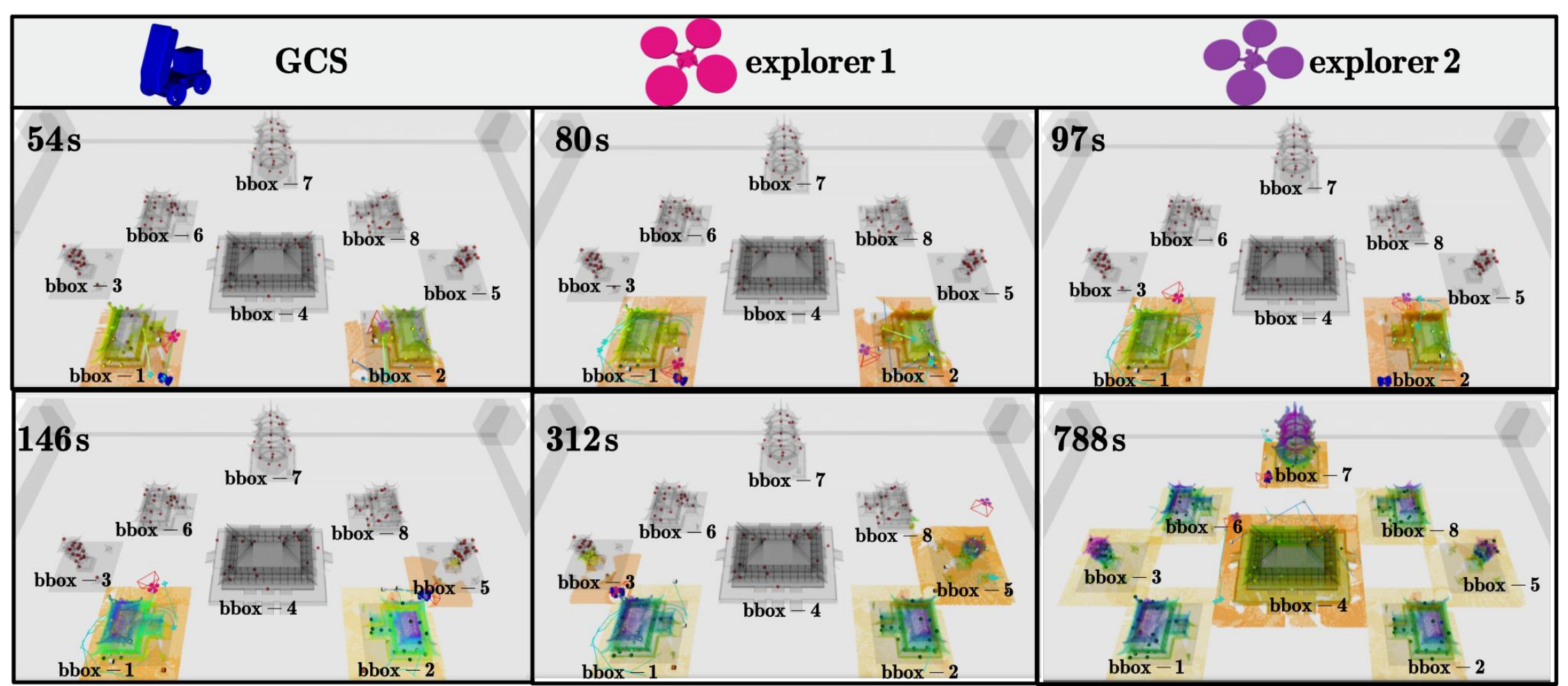}
  \centering
  \caption{
  Snapshots of the GCS-to-SubGroup layer in Scenario-A
  with~$1$ GCS,~$2$ explorers and~$4$ inspectors.
  {At~$t=146s$, BBox-5 is reassigned to explorer~$2$,
  and BBox-3 to explorer~$1$ at~$t=312s$.
  }
  }
  \label{fig:gcs-exp-evol}
  \vspace{-4mm}
\end{figure*}

The simulated robotic fleet
consists of UAVs as explorers and inspectors,
and UGVs as the GCS,
which is deployed in a light-weight simulator~\cite{marsim, fuel}.
As shown in Fig.~\ref{fig:sys-scene},
the following \textbf{four} scenarios are considered.
Scenario-A: $1$ GCS,~$2$ explorers and~$4$ inspectors are deployed
in a cluster of ancient structures of size~$60\times50\times9m^3$
to validate the performance of the GCS-to-SubGroups layer as described in Sec.~\ref{subsubsec:gcs-exp-evol};
Scenario-B: $1$ explorer and~$3$ inspectors are deployed in a large and single
architecture structure of size~$18\times20\times15m^3$
to validate the performance of the SubGroups layer as described in Sec.~\ref{subsubsec:exp-ins-evol};
Scenario-C: $1$ GCS, $4$ explorers and $8$ inspectors are deployed
in a medium-scale city of size $70\times60\times30m^3$,
to validate the performance of the overall scheme as described in Sec.~\ref{subsubsec:full-process-sim};
and Scenario-D: $1$ GCS, $8$ explorers and $40$ inspectors are deployed
in a large-scale ancient site of size $120\times160\times20m^3$
to validate the overall performance in a large-scale scene.

To bridge the gap between theoretical validation and practical deployment in high-fidelity simulations,
our experimental framework is constructed with three meticulously designed core modules:
perception, planning, and communication.
The system implementation leverages a novel lightweight architecture that diverges from conventional Gazebo-based workflows through:
(I) Perceptual realism via sensor emulation: Laser rangefinders and RGB-D cameras are simulated with configurable noise models (±2cm ranging error, 5\% depth distortion),
replicating Gazebo plugin functionalities while eliminating computational overhead.
(II) Environment abstraction: Operational spaces are discretized into point cloud representations (0.1m resolution voxels) through RVIZ integration,
preserving geometric fidelity while enabling real-time collision checking.
(III) ROS-native interoperability: Standardized message interfaces ensure seamless transition between simulated and physical platforms.
The system implementation details are specified as follows:
A 3D Euclidean Signed Distance Field (ESDF) map is generated
via~\cite{han2019fiesta} as local map,
with a~$5~m$-range depth camera and odometry sensor as a volumetric map of the environment.
Then, the Area of Interest (AoI) identification pipeline concurrently processes
the synchronized data streams at $0.5~Hz$:
RGB images analyzed by predicting the depth and semantic information via MiDas network and DeepLabV3+ for 3D reconstruction of geometric features,
and object recognition by YOLOv7.
Each explorer and inspector navigates using a kinodynamic motion planner to generate collision-free path,
which contains multiple topologically distinct waypoints capturing the
structural complexity of the 3D environment,
with a maximum speed of~$2m/s$ and acceleration of~$2m/s^2$.
Feature fitting, detection and inspection are facilitated
by inspectors equipped with a field-of-view (FOV) camera,
which has a left and right view of~$90^{\circ}$,
a front view of $60^{\circ}$, and a detection range of $5~m$.
Once features are within the FOV,
the features can be fitted via~$\texttt{Fit}(\cdot)$ in~\eqref{eq:robot-fit}
or inspected via~$\texttt{Inspect}(\cdot)$ in~\eqref{eq:robot-inspect}
automatically by the explorers or inspectors.
Moreover, the communication within the robotic fleet
is restricted to a range of~$5m$
and has a LOS among the GCS, explorers and inspectors.

\subsection{Results}
\label{subsec:results}

The results associated with the four scenarios are summarized below,
including generalization to robot failures, multiple GCSs
and high-priority features.
Lastly, scalability and robustness analyses are performed w.r.t. fleet size,
various uncertainties and different duration of inspection tasks.

\begin{figure*}[!t]
    \centering
    \includegraphics[width=0.8\linewidth,height=0.3\linewidth]{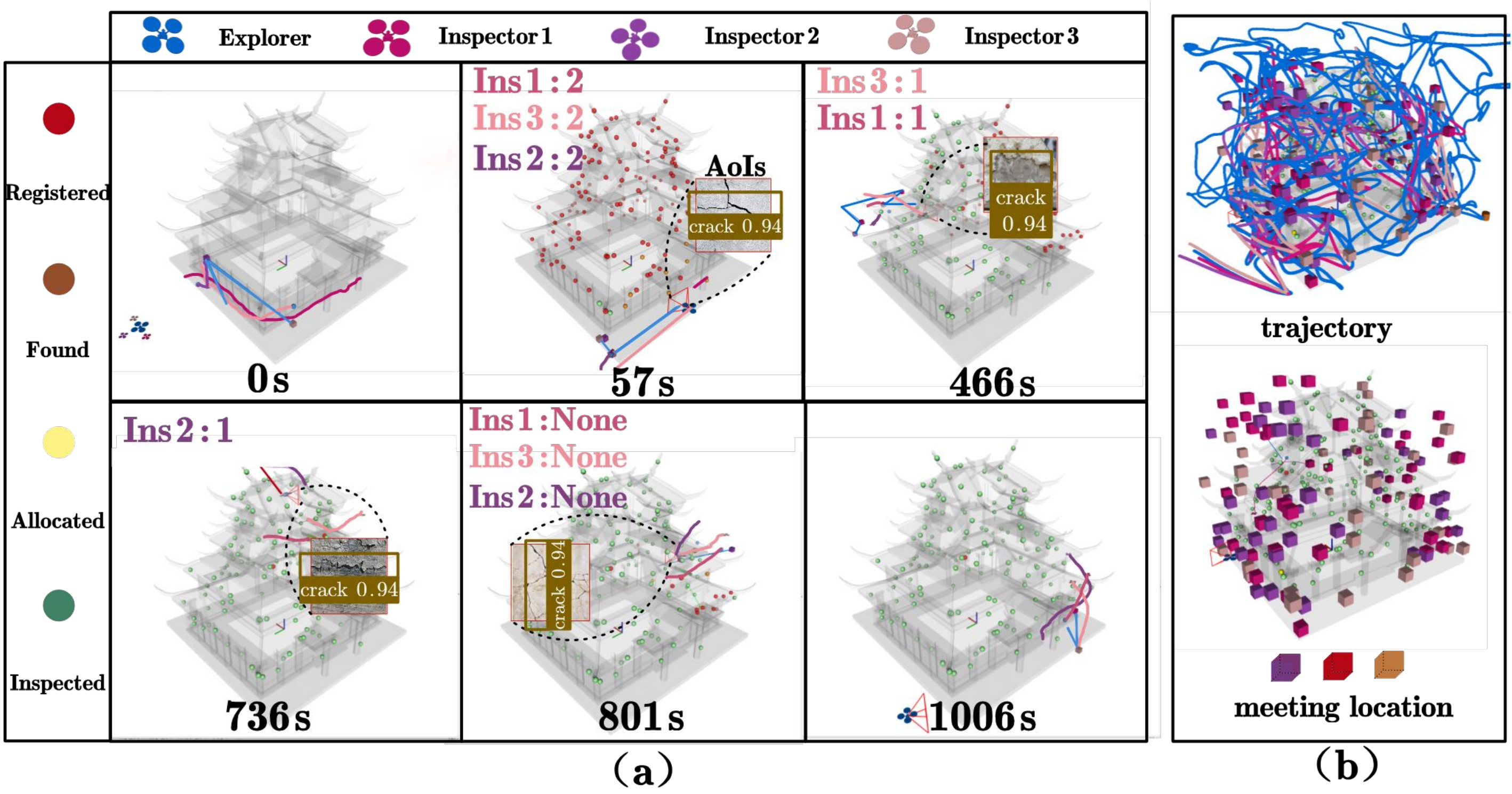}
    \centering
    \caption{
      Evolution of the layer of SubGroups in Scenario-B consisting
      of one explorer and three inspectors:
      (\textbf{a}) snapshots of the planned trajectories of the
      explorer and inspectors at different time instants;
      (\textbf{b}) the entire trajectories and the complete meeting
      events of all robots.
      }
    \label{fig:single-bbox-evol}
    \vspace{-4mm}
  \end{figure*}

\begin{figure}[!t]
  \centering
  \includegraphics[width=0.99\linewidth]{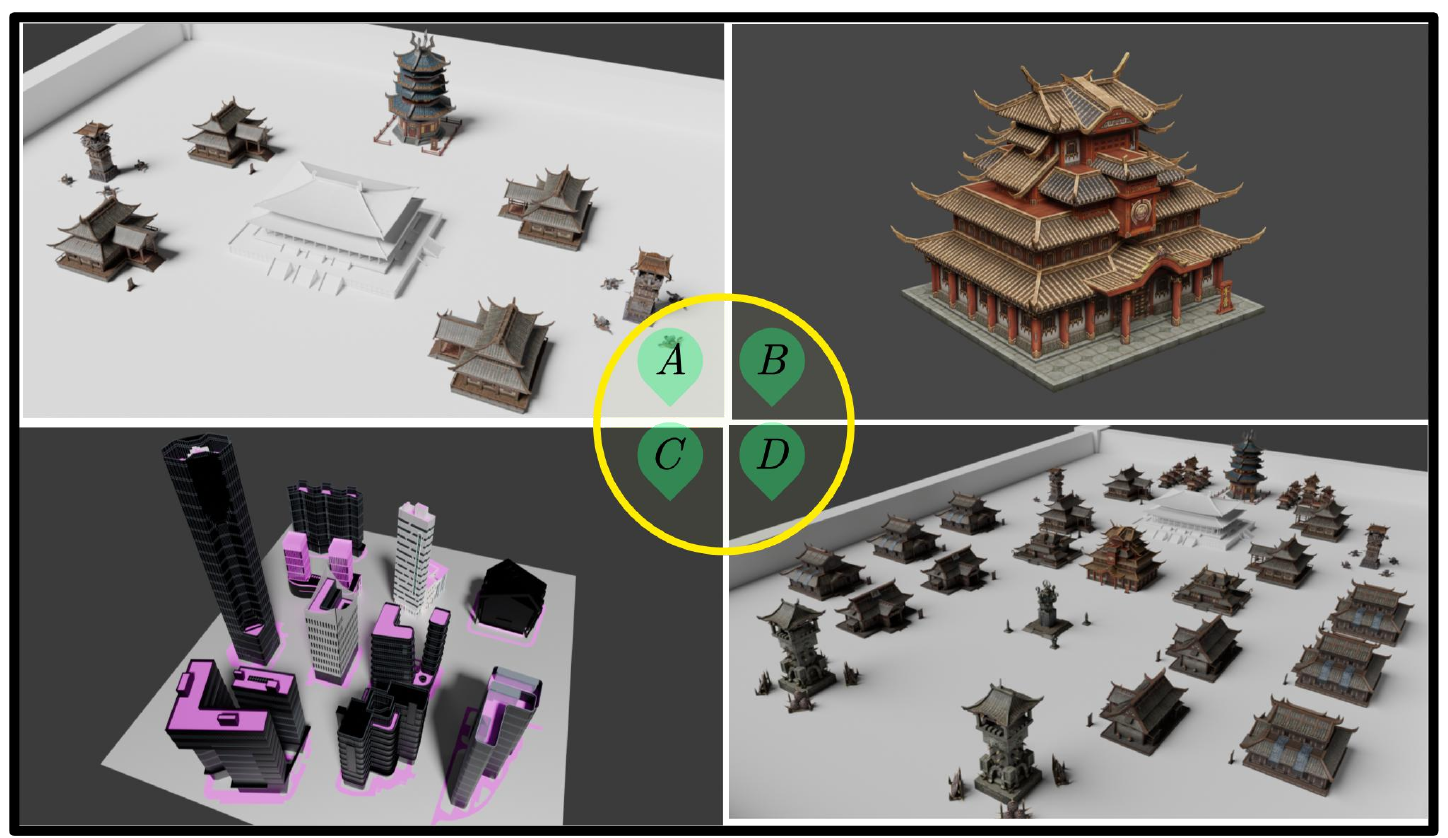}
  \centering
  \caption{Scenarios tested in the numerical experiments:
    Scenario-A with a set of small ancient structures;
    Scenario-B with a single ancient structure;
    Scenario-C with a medium-scale city structures;
    and Scenario-D with a large-scale ancient site.}
  \label{fig:sys-scene}
  \vspace{-4mm}
\end{figure}

\subsubsection{Evaluation of the Layer of GCS-to-SubGroups}
\label{subsubsec:gcs-exp-evol}

As depicted in Fig.~\ref{fig:gcs-exp-evol},
the scenario-A includes~$1$ GCS,~$2$ explorers and~$4$ inspectors,
to explore an area of~$8$ BBoxes.
{To initiate the exploration process,
the GCS assigns two inspectors to each explorer as subgroups
based on the fleet size and number of BBoxes,
using the rolling assignment strategy.
Then, these subgroups are assigned to the nearest BBoxes with an average assignment time of~$1ms$.
Subsequently,
each explorer is guided by Alg.~\ref{alg:explore} with an average planning time of~$26ms$,
which dynamically adapts the exploration path based on the
current position, velocity and the confirmed meeting events~$\mathcal{C}$.
At the same time,
the {communication coordination algorithm}
has an average computation time of~$3.5ms$,
which determines the optimal visit sequence
between the GCS and the explorers.
For instance,
during the first meeting,
the GCS is set to visit BBox-1 to meet explorer~$2$,
then proceed to BBox-2 for explorer~$1$,
return to BBox-1 for another interaction with explorer~$2$,
and finally head to BBox-5 for explorer~$1$ again,
as shown in Fig.~\ref{fig:gcs-exp-evol}.
The prediction algorithm of the task completion time in~\eqref{eq:estimate-duration}
takes on average~$0.47ms$,
with an average prediction error around~$13.2s$.
The execution process can be further explained in the following timeline.
At~$t=54s$,
the GCS follows the confirmed meeting event with explorer~$2$,
and arrives early at the meeting location of BBox-1,
as shown in Fig.~\ref{fig:gcs-exp-result}.
At~$t=80s$, explorer~$2$ reaches the specified location to meet with GCS.
Similarly, the GCS waits for explorer~$1$ at~$t=97s$
at the meeting location of BBox-2 earlier than the scheduled time.
At~$t=146s$,
explorer~$1$ meets with the GCS and a new task to explore BBox-5 is assigned to it.
After meeting with GCS at~$t=312s$, explorer~$1$ sends the inspected features,
and GCS assigns the nearest BBox-3 as the new task to explorer~$1$.
After~$t=788s$, all~$150$ features are collected by the GCS.

\begin{figure}[!t]
  \centering
  \includegraphics[width=1.0\linewidth, height=0.35\linewidth]{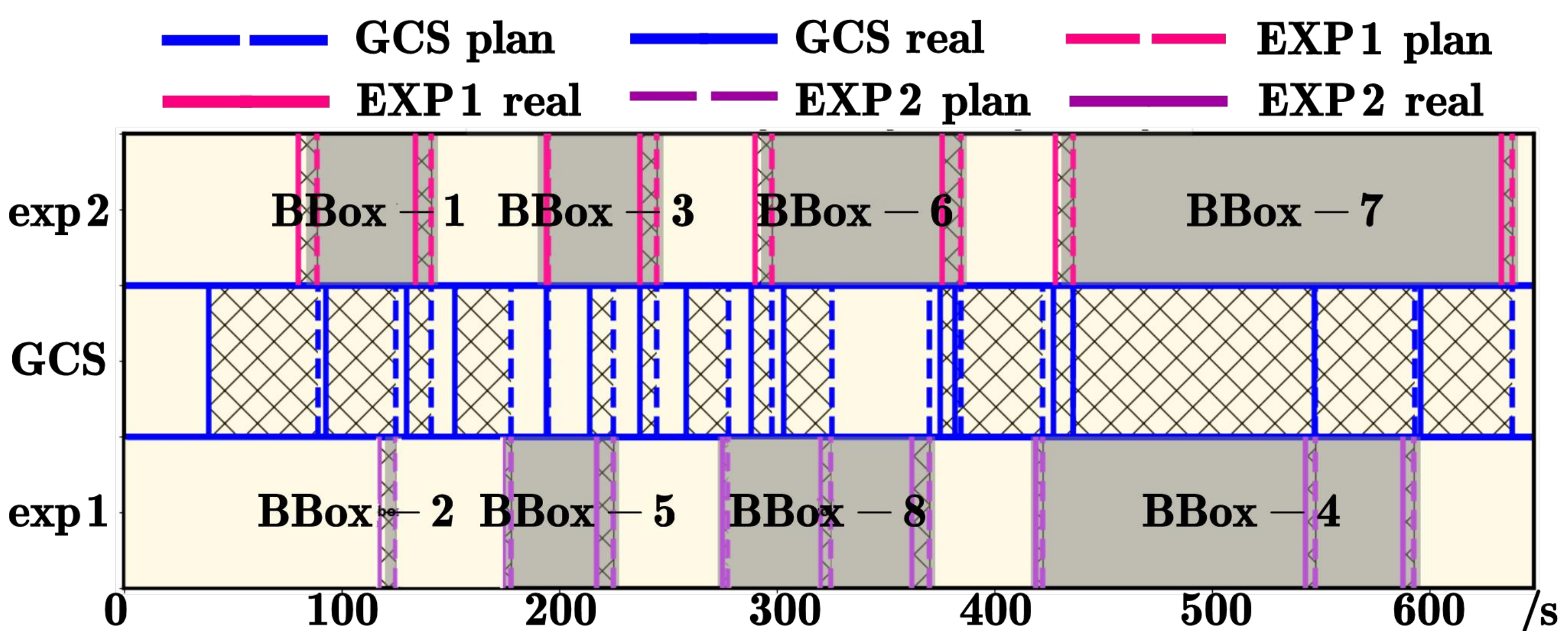}
  \centering
  \caption{
  The planned and actual meeting events between the GCS and two explorers
  in Scenario-A.
  The planned meeting time for the GCS (blue dashed lines)
  aligns with the explorers (pink and purple dashed lines),
  indicating the confirmed meeting event and the target BBox.
  The solid lines represent the actual arrival time,
  which are bound with the planned time of the same meeting
  by latticed areas.
  }
  \label{fig:gcs-exp-result}
  \vspace{-4mm}
\end{figure}

  \begin{figure*}[!t]
    \centering
    \includegraphics[width=0.92\linewidth, height=0.2\linewidth]{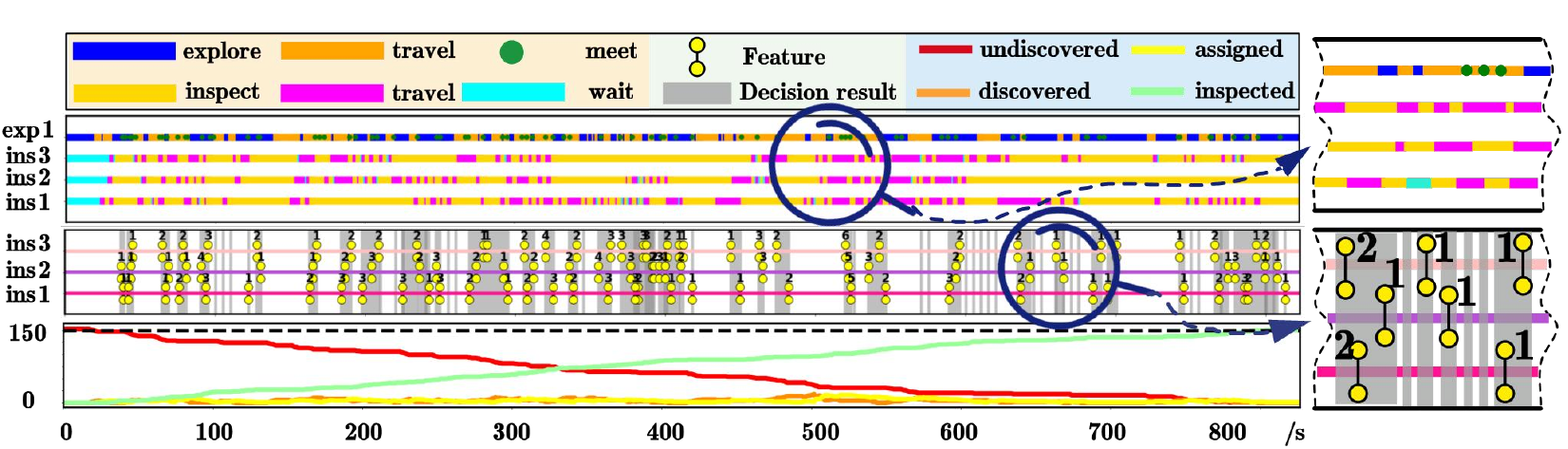}
    \centering
    \caption{
      \textbf{Top}: Evolution of the status of all robots within the subgroup in Scenario-B:
      explorers with exploring (deep blue), traveling (deep orange)
        and meeting (green);
      inspectors with inspecting (light orange), traveling (pink) and waiting (cyan).
      \textbf{Middle}: Planning results of explorers during the meetings
      with inspectors, represented by gray boxes.
      Hammer-shaped symbols with different inspectors
      along the timeline indicate the sequence of planned meeting events,
        with the number of features assigned to each inspector.
        \textbf{Bottom}: Progression of different status
        associated with the features over time:
        undiscovered (red), discovered by explorers (orange),
        assigned to inspectors (yellow),
        and inspected by inspectors (green).
    }
    \label{fig:exp-ins-result}
    \vspace{-4mm}
  \end{figure*}

Moreover,
as illustrated in Fig.~\ref{fig:gcs-exp-result},
the number of meetings for each BBox is less than~$3$,
i.e., the proposed algorithm can accurately predict the completion time of
each exploration task.
Moreover, it is worth noting that the
actual arrival times of all explorers for the meeting events are consistently
earlier than the planned meeting time,
indicating that Alg.~\ref{alg:explore} can ensure the timely arrival within
the specified time window,
thus reducing the idle time for the GCS and the explorers.

\begin{figure*}[!t]
    \centering
    \includegraphics[width=0.9\linewidth, height=0.35\linewidth]{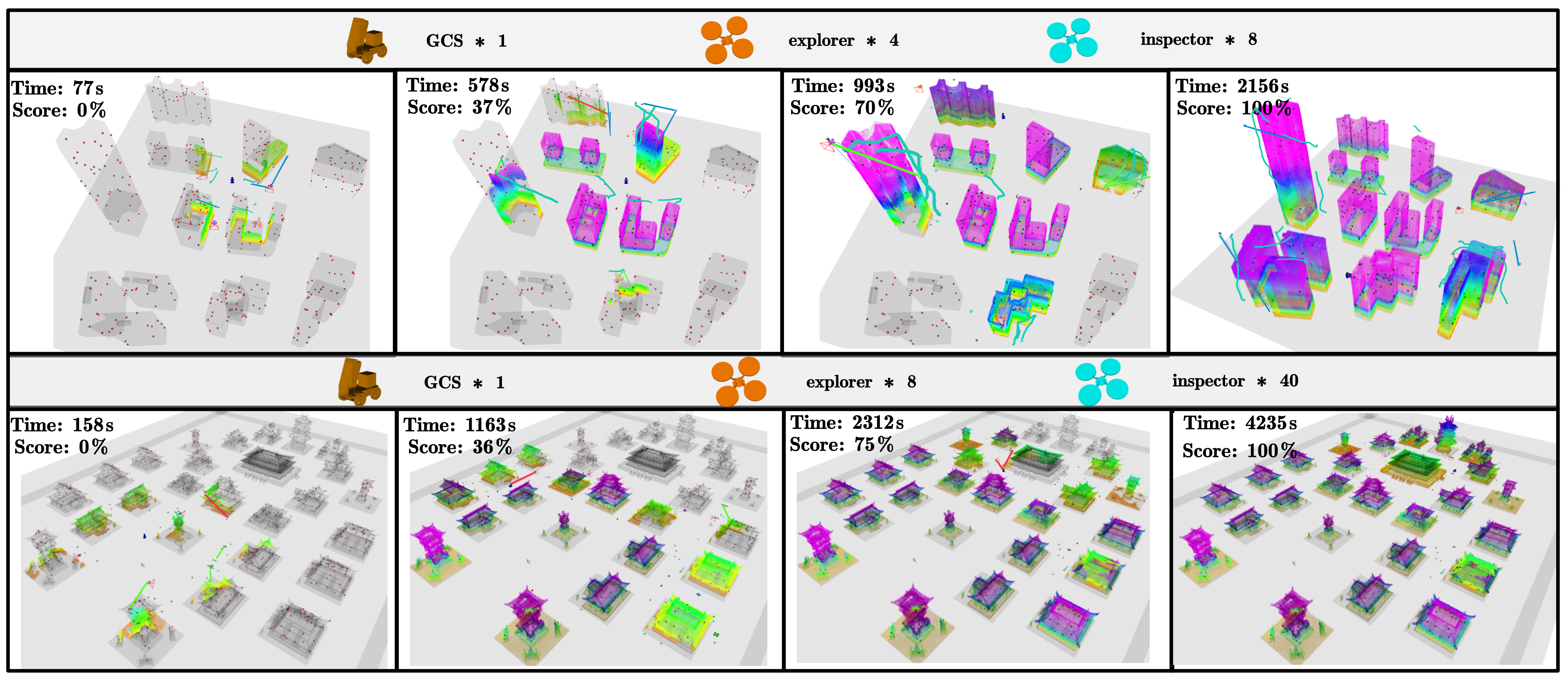}
    \centering
    \caption{
      Snapshots of overall execution for
      Scenario-C under~$1$ GCS,~$4$ explorers and~$8$ inspectors (\textbf{Top});
      and Scenario-D under~$1$ GCS,~$8$ explorers and~$40$ inspectors
      (\textbf{Bottom}).
    }
    \label{fig:modern-evol}
    \vspace{-6mm}
  \end{figure*}
\subsubsection{Evaluation of the Layer of SubGroups}
\label{subsubsec:exp-ins-evol}

This section highlights the intra-group dynamics in Scenario-A,
as depicted in Fig.~\ref{fig:sys-scene},
and the influence of different inspection time~$t_q$
on the overall performance.

{(I) \textbf{Overall execution}}.
As illustrated in Fig.~\ref{fig:single-bbox-evol},
the robot motion and actions within the subgroup are monitored during the execution in Scenario-B,
along with the exploration and fitting of features.
The Alg.~\ref{alg:SOEI} is triggered every~$5s$ if new features~$\mathcal{F}^+$ are discovered
to optimize the subgroup solutions~$\Xi_i$ with an average computation time of~$1.384s$.
Specifically,
given the executing local plans~$\{\xi_j^-\}$ of all inspectors,
the explorer utilizes the~$\texttt{SmapleLos}(\cdot)$ function
which has an average runtime of~$1.3s$,
while the~$\texttt{OptMeet}(\cdot)$ has an average runtime of~$0.04ms$
to optimize the inspectors~$\widehat{\mathcal{N}}_i$ and the meeting events~$\mathbf{c}_i$.
Subsequently,
the explorer assigns the newly-fitted features to the selected inspectors
by invoking the modified~$\texttt{MVRP}(\cdot)$ algorithm
with an average runtime of~$0.08s$.
Lastly,
the local plans~$\{\xi_j^+\}$ of the inspectors are updated based on
the executing local plans~$\{\xi_j^-\}$ and assigned features~$\boldsymbol{\varphi}_i$,
by invoking the~$\texttt{Navi}(\cdot)$ with an average computation time of~$0.05s$.
During the communication phase,
the explorer transmits the allocated features~$\boldsymbol{\varphi}_i$ and
the updated local plans~$\{\xi_j^+\}$ to the respective inspectors,
ensuring the consistency of tasks across the subgroup.
Some representative moments are described as follows:}
At~$t=57s$, the explorer decides to meet with
the inspector~$1$ which is assigned~$2$ features,
then with the inspector~$3$ which is assigned~$2$ features,
and finally the inspector~$2$ which is assigned~$2$ features.
At~$t=466s$, the explorer meets with the inspector~$3$ first which is assigned~$1$ feature,
followed by the inspector~$1$ and assigned~$1$ feature.
The explorer avoids meeting with inspector~$2$ due to longer travel distance,
thereby reducing the total idle time.
At~$t=736s$, the explorer chooses to meet only inspector~$2$,
which is assigned the newly discovered features.
These features are in the proximity of inspector~$2$ and there are no remaining features to inspect.
At~$t=801s$, the explorer meets with no inspectors,
as all inspectors have ongoing inspection tasks and are far away.
Moreover, the entire trajectories of all robots within the subgroup
and the meeting locations of the subgroup are also shown.
It can be seen that the trajectory of the explorer
is rather complex due to the online adaptation
to meeting events with all inspectors,
whereas the trajectories of the inspectors are smoother
as the features are appended to their local plans incrementally.

Last but not least,
the online change of the status of each robot is shown in Fig.~\ref{fig:exp-ins-result},
e.g., ``explore'', ``travel'', ``inspect'' and  ``wait'';
the evolution of the number of ``registered'', ``found'',
``allocated'', and ``inspected'' features;
and the communication events between the explorer and inspectors.
It can be seen that
the explorer spends~$60.35\%$ of the time to explore and fit AoIs across the entire scenario.
The inspectors primarily travel to features for inspection
with an average waiting time of~$33s$.
In addition,
the middle figure shows that the explorer meets with  all, some or none
of the inspectors at different moments,
with the average communication cycle of~$6.8s$.
The bottom sub-figure shows that~$150$ features are fully inspected
within~$1006s$, yielding a~$100\%$ coverage.

(II) \textbf{Different Inspection Time}.
To evaluate the effect of different inspection times~$t_q$
on the performance of our proposed method,
we have conducted a series of experiments with different~$t_q$,
ranging from~$1s$ to~$10s$.
The metric to compare includes
the average number of features assigned per meeting.
{The results show that}
the number of assigned features {monotonically} decreases from $3.11$ to $1.65$ as~$t_q$ increases
from~$1s$ to~$8s$.
Thus, more frequent meetings are required if it takes shorter duration
to inspect the features.



\subsubsection{Full-process Simulation}
\label{subsubsec:full-process-sim}

To validate the complete performance as shown in Fig.~\ref{fig:modern-evol},
large-scale simulations are conducted under different
sizes of workspace and robotic fleets:
Scenario-C with~1 GCS, 4 explorers and 8 inspectors,
and Scenario-D with~1 GCS, 8 explorers and 40 inspectors.
As described earlier, the two layers are integrated with different frequencies.
In Scenario-C,
the average meeting cycle is~$40.6s$ between GCS and explorers,
and the average meeting cycle is~$6.8s$ between explorers and inspectors.
In contrast to Scenario-D,
these rates are $63.2s$ and $5.5s$, respectively.
The rolling assignment strategy further improves the efficiency
of different subgroups by dynamically reallocating BBoxes.
In Scenario-C, the explorers complete tasks at $578s$ and $993s$,
which are then reassigned to nearby BBoxes to minimize idle time.
Similarly, in Scenario-D, explorers are reallocated at $1163s$ and $2321s$.
{On average,
each explorer is assigned $2.5$ BBoxes in Scenario-C and $3.3$ BBoxes in Scenario-D,
with an average mission time of~$529s$ and~$513s$, respectively.
This indicates that the proposed method can effectively balance
the task allocation for each subgroup across different scenarios.}
It is worth noting that the explorers achieve a $100\%$ success rate
in reaching meeting locations on time in both scenarios,
while the GCS occasionally experiences minor delays in Scenario-D
due to overlapping tasks in cluttered environments.
Moreover,
the average number of communications between the GCS and each subgroup is less than~$3$ in both scenarios,
indicating that the proposed algorithm is effective in predicting the completion time.
The layer of SubGroups also demonstrates adaptability in complex environments,
such as cluttered buildings in Scenario-C
and intricate structures in Scenario-D.
In both cases, the proposed method achieves $100\%$ feature coverage
for $1006s$ in Scenario-C and $2006s$ in Scenario-D.

The overall efficiency is influenced by multiple factors,
with the exploration efficiency of the explorers
being particularly critical for enhancing
the feature discovery and the inspection process.
A key metric quantifying exploration efficiency is the overlap rate, 
which reflects redundant coverage across robots.
In our framework,
the inter-group exploration overlap
is fully eliminated via the pre-assigned bounding boxes (BBoxes)
for the dedicated explorers in the robotic fleet.
However, intra-group operations by single explorer rather than multi-explorer exploration exhibit an average $15\%$ overlap rate
across the four scenarios, i.e., $20\%, 10\%, 13\%, \text{and}~17\%$ in Scenario A--D respectively.
This result aligns with the state-of-the-art overlap rate for single-robot exploration in~\cite{wang2025multi}.
By enforcing a strict zero inter-group overlap rate and suppressing intra-group overlap to minimal levels, 
the proposed approach significantly enhances exploration efficiency through systematic reduction of redundant exploration.


\begin{figure*}[!t]
  \centering
  \includegraphics[width=0.99\linewidth]{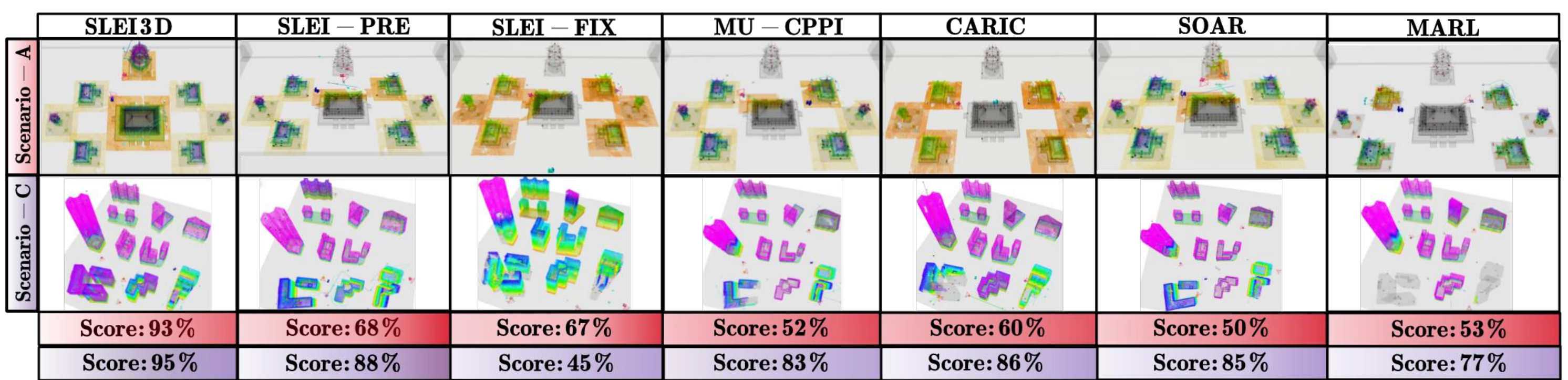}
  \centering
  \caption{
    Comparison of the final score between
    the proposed method (SLEI3D)
    and six baselines in Scenario-A and Scenario-C.
  }
  \label{fig:all-mtds-evol}
  \vspace{-6mm}
\end{figure*}

\subsubsection{Maximum Fleet Size under Communication Constraints}

Firstly,
our experimental measurements demonstrate that the average communication bandwidth per interaction
between the GCS and subgroups in Scenarios~B--D remains on average~$1.9$Mbps ($2.018$Mbps vs.~$1.755$Mbps vs.~$1.940$Mbps).
This bandwidth consumption exhibits direct proportionality to the average transmitted features per meeting($11$ vs.~$9.5$ vs.~$10.5$).
Notably, the system maintains this bandwidth efficiency ($\mathbf{1.9}$Mbps)
even under dense feature conditions inherent to our problem configuration,
demonstrating the algorithm's capability for compact information encoding while preserving operational fidelity.
{The framework adopts a \textbf{single-pair-per-link} paradigm for communication asynchrony,
where the pairwise interactions strictly avoid simultaneous multi-robot communication over 
the shared bandwidth in an instantaneous and centralized manner.
As a result, it does not introduce additional constraints on the feasible fleet size, beyond those 
imposed by the underlying communication hardware and mission environment.
}

\subsection{Comparisons}
\label{subsec:exp-comp}

To further validate the effectiveness of the proposed method (as \textbf{SLEI3D}),
a quantitative comparison is conducted against \textbf{six} strong baselines:

\begin{enumerate}[label=(\roman*)]
    \item \textbf{CARIC},
          which is based on a hierarchical framework for simultaneous exploration and inspection
          proposed in~\cite{xu2024cost}.
          The robots are partitioned into teams, where each team is independently assigned to specific sub-regions,
          and then relays the inspection results to a static GCS under the LOS communication.
          For a fair comparison,
          a modification is added such that the GCS can move to communicate with the explorers.
    \item \textbf{SOAR},
          which is built upon the work in~\cite{zhang2024soar}.
          Only one explorer identifies the unknown areas and generates viewpoints
          that are subsequently assigned to inspectors through global all-to-all communication.
          However, it cannot handle scenarios involving multiple explorers and inspectors under limited communication,
          nor does it account for the need to return information back to the GCS.
          To ensure a fair comparison with our method,
          SOAR is only adopted for the task of assigning features from explorers to inspectors,
          while the rest of components follows the proposed framework
          to be compatible with the problem settings.

    \item \textbf{MU-CPPI},
      which is based on the exploration-then-inspection
      framework proposed in~\cite{jing2020multi},
      The robots first explore the entire environment,
      and then the local plans of inspection are generated
      based on the exploration results.
      However,
      as this approach does not account for a movable GCS,
      limited communication constraints
      and simultaneous exploration and inspection.
      Therefore,
      above method can be only leveraged for the subgroups layer
      to guide explorers to meet with the inspectors in the subgroup.
      The rest of components follow the proposed framework
      to match the problem settings.

    \item \textbf{MARL},
          a multi-robot strategy for collaborative exploration via multi-agent reinforcement learning (MARL)
          as seen in~\cite{battocletti2021rl, zhu2024maexp, chen2024meta}.
          These works develop end-to-end coordination strategies that jointly
          optimize exploration policies and collision-free motion planning.
          However, they exclusively focus on the exploration phase within our problem formulation,
          without addressing communication requirements between inter-and-intra groups.
          Thus, the MARL module is adopted only to replace the exploration strategy for explorers as presented in Sec.~\ref{subsec:explore}
          while retaining other components of the proposed framework
          to ensure system compatibility.

    \item \textbf{SLEI-PRE},
          where all BBoxes are initially allocated to the explorers,
          i.e., each explorer follows a fixed order to explore the designated BBoxes,
          instead of relying on dynamic allocation by the GCS.

    \item \textbf{SLEI-FIX},
          where the GCS remains stationary at its initial position throughout the entire process.
          Consequently, the explorers are required to return to the GCS at fixed intervals
          for communication.
\end{enumerate}

\begin{table*}[!t]
  \centering
  \caption{\\\uppercase{Comparison of baselines across two scenarios.}}
  \begin{tabularx}{\textwidth}{c c *{4}{C} *{4}{C} *{4}{Z} *{4}{Y} *{4}{Z}}
  \toprule\midrule
    \multirow{2}{*}{\textbf{Scene}} & \multirow{2}{*}{\textbf{Method}} & \multicolumn{4}{c}{\textbf{Finish Rate (\%)}}
                                                                        & \multicolumn{4}{c}{\textbf{Finish Time (s)}}
                                                                        & \multicolumn{4}{c}{\textbf{GCS-EXP Meeting (\#)}}
                                                                        & \multicolumn{4}{c}{\textbf{EXP-INS Meeting (\#)}}
                                                                        & \multicolumn{4}{c}{\textbf{Total Waste Time (s)}} \\
                                    &                                  & \textbf{Avg} & \textbf{Std} & \textbf{Max} & \textbf{Min}
                                                                        & \textbf{Avg} & \textbf{Std} & \textbf{Max} & \textbf{Min}
                                                                        & \textbf{Avg} & \textbf{Std} & \textbf{Max} & \textbf{Min}
                                                                        & \textbf{Avg} & \textbf{Std} & \textbf{Max} & \textbf{Min}
                                                                        & \textbf{Avg} & \textbf{Std} & \textbf{Max} & \textbf{Min} \\
    \midrule
    \multirow{6}{*}{\begin{tabular}[c]{@{}c@{}}Scenario\\A\end{tabular}}
      & CARIC                              & 97.6 & 2.1 & 100 & 96 & 1009 & 29.6 & 1037 & 978 & 7.3 & 0.6 & 8 & 7 & 57.3 & 1.15 & 58 & 56 & 1967 & 72.4 & 2018.3 & 1884.2\\
      & SLEI-PRE                               & 92 & 2 & 94 & 90 & 988.7 & 116.3 & 1100 & 868 & 10.3 & 0.6 & 11 & 10 & 90 & 6.1 & 97 & 86 & 2259.4 & 292.9 & 2586.2 & 2020.5 \\
      & SOAR                             & 73 & 1 & 74 & 72 & 1110 & 141.1 & 1254 & 972 & 10.3 & 1.5 & 12 & 9 & 174 & 14.1 & 189 & 161 & 2716.1 & 603.9 & 3304.6 & 2097.8 \\
      & MU-CPPI                            & 83.7 & 2.5 & 86 & 81 & 1239 & 72.4 & 1322 & 1189 & 9.7 & 1.5 & 11 & 8 & 83 & 4.4 & 86 & 78 & 3104.1 & 467.9 & 3616.8 & 2700.1\\
      & SLEI-FIX                             & 96 & 2.7 & 98 & 93 & 939.7 & 66.30 & 1002 & 870 & 10 & 0 & 10 & 10 & 52.7 & 1.2 & 54 & 52 & 1551.9 & 85.9 & 1636.6 & 1464.8\\
      & MARL                             & 92.4 & 4.2 & 97.1 & 87.5 & 932.5 & 120.3 & 1060 & 800.3  & 13.5  & 4.2  & 18  & 9  & 93.5 & 13.2 & 108 & 88 & 2218.8 & 512.9 & 3000.8  & 1737.2 \\
      & \textbf{SLEI3D}                            & \textbf{94.0} & 1 & 94 & 93 & \textbf{743.7} & 208.5 & 976 & 573 & \textbf{13.7} & 3.5 & 17 & 10 & \textbf{84.7} & 11.9 & 98 & 75 & \textbf{1434.5}  &17.0  &1454.1  &1423.6 \\
    \midrule
    \multirow{6}{*}{\begin{tabular}[c]{@{}c@{}}Scenario\\C\end{tabular}}
      & CARIC                             & 98.7 & 0.6 & 99 & 98 & 2659 & 317.5 & 3007 & 2385 & 8.7 & 0.6 & 9 & 8 & 182.7 & 11.1 & 194 & 172 & 10032 & 2178.2 & 12288 & 7940.4 \\
      & SLEI-PRE                              & 95.3 & 1.2 & 96 & 94 & 1767 & 102.9 & 1882 & 1684 & 38.3 & 1.2 & 39 & 37 & 199.3 & 6.8 & 207 & 194 & 8085.5 & 411.1 & 8459.7 &  7645.4  \\
      & SOAR                           & 94 & 1 & 95 & 93 & 1908 & 54.2 & 1954 & 1848 & 43 & 3 & 46 & 40 & 415.3 & 11.2 & 425  & 403 & 8542.1 & 616.3 & 9229.8 & 8039.9 \\
      & MU-CPPI                           & 94.7 & 0.6 & 95 & 94 & 2117 & 51.5 & 2172 & 2070 & 44 & 10.2 & 55 & 35 & 165.3 & 29.1 & 198 & 142  & 9725.3  & 1048.1  & 10742  & 8648.6\\
      & SLEI-FIX                            & 96.7 & 1.2 & 98 & 96 & 2462 & 365.5 & 2870 & 2165 & 14 & 0 & 14 & 14 & 188.7 & 25.5 & 218 & 172 & 8809.4 & 490.1 & 9252.3 & 8282.9 \\
      & MARL                            & 77.8  & 5.5 & 85 & 70  & 1648 & 210.8 & 2000 & 1314 & 37 & 0  & 37 & 37 & 162.7  & 28.6  & 200 & 130 & 8479.1  & 504.2  & 9000.1  & 7800.5  \\
      & \textbf{SLEI3D}                 & \textbf{96.7} & 1.5 & 98 & 95 & \textbf{1699} & 66.7 & 1776 & 1658 & \textbf{42} & 3.5 & 44 & 38 & \textbf{279.7} & 80.8 & 353 & 193 & \textbf{7920.3} & 733.6 & 8767.4 & 7496.8 \\
     \midrule\toprule
  \end{tabularx}
  \label{tab:result}
  \vspace{-4mm}
\end{table*}

    \begin{figure}[!t]
      \centering
      \includegraphics[width=0.95\linewidth]{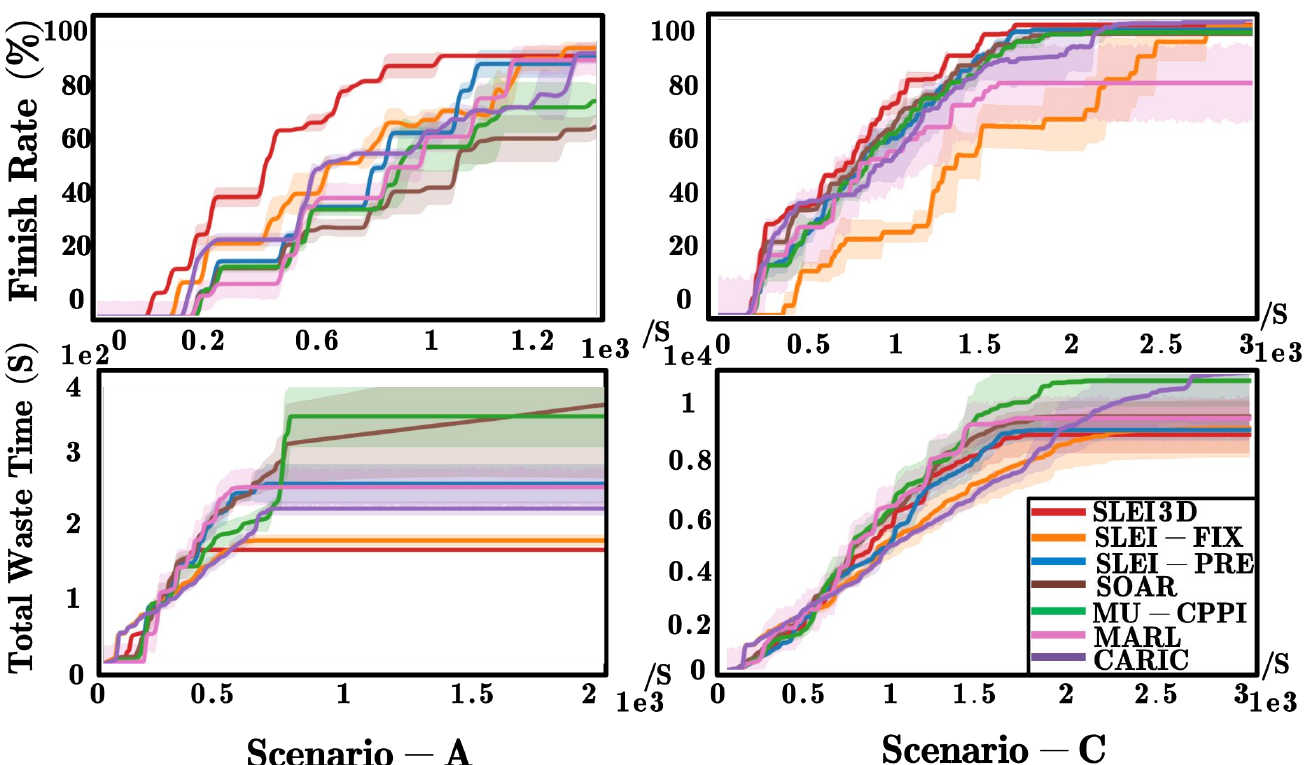}
      \centering
      \caption{
      Comparison of the number of collected features over time,
      between the proposed method and six baselines over three runs.}
      \label{fig:feature-waste-time-result}
      \vspace{-4mm}
    \end{figure}



The above methods are evaluated across both scenario-A and scenario-C,
each of which $3$ tests are conducted.
The experimental results are summarized in Fig.~\ref{fig:all-mtds-evol} and Table~\ref{tab:result}.
Collected features by the GCS over time are shown in Fig.~\ref{fig:feature-waste-time-result},
where the execution time for scenario-A is~$600s$ and~$1660s$ for scenario-C.
For both scenarios,
SLEI3D takes the least amount of time~$743.7s$ and~$1699s$ to collect all features
(thus the highest efficiency) in scenario-A and scenario-C,
compared to CARIC ($1009s$ and~$2659s$),
SOAR ($1110s$ and~$1908s$),
MU-CPPI ($1239s$ and~$2117s$),
MARL ($932.5s$ and~$1648s$),
SLEI-PRE ($988.7s$ and~$1767s$),
and SLEI-FIX ($939.7s$ and~$2462s$).
Furthermore,
it is evident that employing a movable GCS
facilitates faster task completion and less idle time,
compared with SLEI-FIX ($939.7s$ and~$2462s$).
In addition,
while both SLEI3D and SLEI-PRE achieve similar finish rates
(($94.0\%$ and~$96.7\%$) vs.~($92\%$ and~$95.3\%$)),
the dynamic allocation of BBoxes
allows for much less idle time, compared with SLEI-PRE.
SLEI3D also facilitates more effective collaboration between explorers and inspectors.
This is evident by the higher number of meeting events between explorers and inspectors,
particularly in scenario-C, where SLEI3D achieves an average of $279.7$ meetings,
compared to $199.3$ for SLEI-PRE and $188.7$ for SLEI-FIX.
This difference is even more pronounced in comparison to CARIC and SLEI3D.
While CARIC adopts a simpler intra-group communication strategy,
it lacks the adaptability for complex tasks.
SLEI3D achieves faster task completion
(($743.7s$ and $1699s$) vs. ($1009s$ and $2659s$))
and less idle time (($1434.5s$ and $7920.3s$) vs. ($1967s$ and $10032s$)).
  Additionally, MARL trained specifically on scenario-A, exhibits significant performance degradation when transferred to scenario-C.
  This is evident from the substantial drop in success rate (from 92.4\% to 77.8\%) between the two scenarios,
  as shown in Table~\ref{tab:result}.
  In contrast, our approach demonstrates superior generalization capability with consistent success rates
  across both scenarios (94.0\% vs. 96.7\%).
  This performance disparity highlights the inherent limitations of pure learning-based methods like MARL,
  which tend to overfit to their training environments.
  Our hybrid architecture effectively mitigates this issue through its adaptive task allocation mechanism
  and geometric-aware coordination,
  enabling robust performance in both familiar and novel scenarios.
  Furthermore, MARL shows high instability with large success rate deviations (4.2\% vs. 5.5\%),
  while SLEI3D demonstrates consistent robustness across runs (1.0\% vs. 1.5\%).
Lastly,
SLEI3D achieves the lowest idle time in both scenarios,
with~$1434.5s$ and~$7920.3s$ in scenario-A and scenario-C, respectively,
compared to CARIC ($1967s$ and~$10032s$),
SOAR ($2716.1s$ and~$8542.1s$),
MU-CPPI ($3104.1s$ and~$9725.3s$),
MARL ($2218.8s$ and~$8479.1s$),
SLEI-PRE ($2259.4s$ and~$8085.5s$),
and SLEI-FIX ($1551.9s$ and~$8809.4s$).
In summary,
SLEI3D effectively addresses the challenges posed by the state-of-the-art frameworks
such as CARIC, SOAR and MARL,
as well as the SLEI variants SLEI-FIX and SLEI-PRE.

\usetikzlibrary{patterns}

  \pgfplotstableread[col sep=comma]{
    type,          ours,     caric,     soar,     mu-cppi
    Scenario-A,    0.118,    0.0307,    0.0535,   0.2655
    Scenario-C,    1.4225,   0.1611,    0.0641,   0.0491
    Scenario-D,    0.357,    0.4465,    0.081,    0.2335
  }\Mtable

\definecolor{v0}{rgb}{0.1, 0.3, 0.8}          
\definecolor{v1}{rgb}{0.6, 0.4, 0.1}          
\definecolor{v2}{rgb}{0.2, 0.7, 0.4}          
\definecolor{v3}{rgb}{0.7, 0.1, 0.5}          
\definecolor{v4}{rgb}{1.0,1.0. 0.0}           
\definecolor{v5}{rgb}{0.7, 0.7, 0.7}          

\pgfplotsset{/pgfplots/error bars/error bar style={thick, black}}
\begin{figure}[t]
	\centering

    \vspace{2mm}
    \begin{subfigure}[t]{0.50\textwidth}
      \begin{tikzpicture}
        \begin{semilogyaxis}[
          width=\linewidth,
          ybar,
          bar width=5pt,
          log origin = infty,
          ymin=0.00001,
          ymax=1.5,
          xtick pos=bottom,
          ytick={0.00001,0.0001,0.001, 0.01, 0.1, 1},
          yticklabel style = {font=\scriptsize},
          minor y tick num=10,
          width=1.0\linewidth,
          height=0.30\linewidth,
          ymajorgrids=true,
          enlarge x limits={abs=35pt},
          legend style={at={(0.45,1.15)},
          anchor=south,legend columns=4, font=\footnotesize},
          ylabel={Comput. Time},
          y label style={at={(axis description cs:0.07,0.5)},anchor=south, font=\scriptsize},
          symbolic x coords={Scenario-A, Scenario-C, Scenario-D},
          x label style={font=\tiny},
          xtick=data,
          cycle list={v5, v4, v3, v2, v1, v0},
        ]

            \addplot+[draw=black,fill,postaction={
              pattern=crosshatch
          }] table[x=type,y=ours]{\Mtable};
          \addlegendentry{SLEI3D}

            \addplot+[draw=black,fill,postaction={
              pattern=crosshatch
          }] table[x=type,y=caric]{\Mtable};
          \addlegendentry{CARIC}

            \addplot+[draw=black,fill,postaction={
              pattern=crosshatch
          }] table[x=type,y=soar]{\Mtable};
          \addlegendentry{SOAR}

            \addplot+[draw=black,fill,postaction={
              pattern=crosshatch
          }] table[x=type,y=mu-cppi]{\Mtable};\
          \addlegendentry{MU-CPPI}

      \end{semilogyaxis}
      \end{tikzpicture}
    \end{subfigure}
  \caption{
        Comparison of computation time between SLEI3D and other methods
        in different scenarios.
  }
  \label{fig:comp-computation}
\vspace{-5mm}
\end{figure}
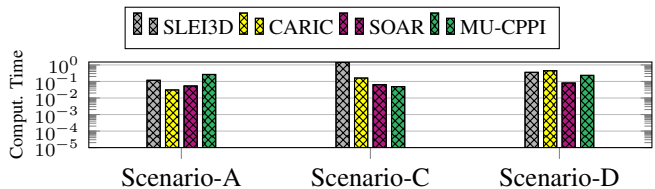

Furthermore, 
as Fig.~\ref{fig:comp-computation} illustrates,
the comparative analysis of computation time with baseline methods (CARIC, SOAR and MU-CPPI) reveals SLEI3D exhibits higher average computation times
(0.6325\,s vs. 0.2128\,s vs. 0.0662\,s vs. 0.1827\,s) across scenarios,
primarily due to the $\texttt{SOEI}(\cdot)$ optimization.
MU-CPPI achieves the shortest processing times, 
whereas SLEI3D exhibits stable computation times ($0.118–1.4225s$). 
Despite higher computational costs, 
SLEI3D’s method yields significant performance improvements as analyzed in Table~\ref{tab:result}.

\subsection{Generalizations}
\label{subsec:ext-exp}

\subsubsection{Robot Failure}
\label{subsubsec:exp-robot-failure}

{As described in Sec.~\ref{sec:discuss}},
the proposed method can detect and recover from potential robot failures
of explorers and inspectors during execution.
Particularly, the parameter~$\overline{\delta}$ is set to~$10s$.

{(I) \textbf{Explorer Failure}}.
Consider the same setting as in Sec.~\ref{subsubsec:gcs-exp-evol}
with~$1$ GCS,~$2$ explorers and~$4$ inspectors in Scenario-A,
where one explorer fails online.
The final meeting results between GCS and explorers are shown in Fig.~\ref{fig:exp-fail}.
Explorer~$1$ fails at~$t=338s$ \emph{before} the meeting with GCS at BBox-3.
This failure is then detected by GCS at the predefined time and location after waiting for~$10s$.
Then,
the GCS directly proceeds to the next meeting event at BBox-8 to communicate with explorer~$2$.
As shown in Fig.~\ref{fig:exp-fail},
the waiting period of GCS results in a slight delay of its arrival at BBox-8
compared to the planned time.
Following the proposed failure recovery mechanism,
the whole subgroup belonging to explorer~$2$ is reassigned
to BBox-3 to restart the task,
as the failed explorer~$1$ is unable to relay the inspection results of BBox-3 to the GCS.
The overall mission is completed at~$1215s$,
which is slightly longer than the~$788s$ in Fig.~\ref{fig:gcs-exp-evol}
without failures.

\begin{figure}[!t]
  \centering
  \includegraphics[width=0.98\linewidth]{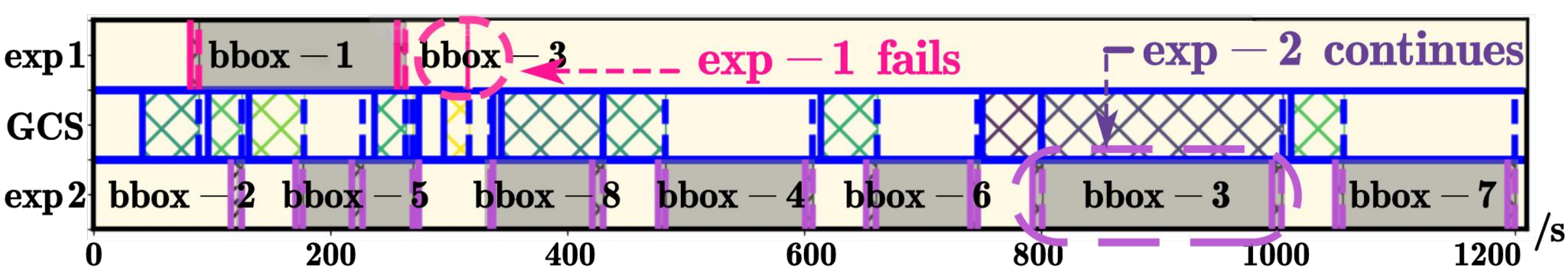}
  \centering
  \caption{Failure detection and recovery
  {when explorer~$1$
    fails at~$t=338s$ in Scenario-A}.}
  \label{fig:exp-fail}
  \vspace{-5mm}
\end{figure}

\begin{figure}[!t]
  \centering
  \includegraphics[width=0.98\linewidth]{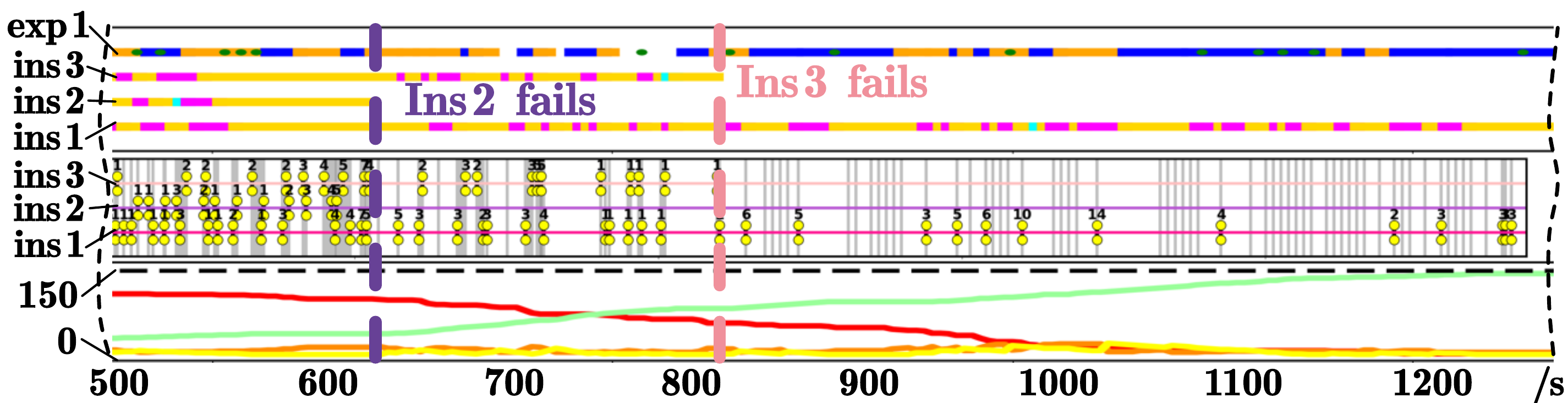}
  \centering
  \caption{Failure detection and recovery
  {when inspector~$2$ fails first at~$t=683s$,
    followed by inspector~$3$ fails at~$t=825s$ in Scenario-B}.}
  \label{fig:ins-fail}
  \vspace{-4mm}
\end{figure}

{(II) \textbf{Inspector Failure}}.
Consider the same setup as in Sec.~\ref{subsubsec:exp-ins-evol},
but two inspectors fail consecutively during execution.
The overall changes in robot state,
decision-making results, and completion status of features
are summarized in Fig.~\ref{fig:ins-fail}.
Inspector~$2$ fails at~$t=683s$ \emph{before} meeting with the explorer,
of which is detected by the explorer after waiting for~$10s$ at the end of its local plan.
Then, inspector~$2$ is excluded from the set of active inspectors that are known to the explorer.
In addition,
it is clear that once the explorer detects the failure of inspector~$2$,
no further features are assigned to inspector~$2$.
After inspector~$3$ fails at~$t=825s$,
the explorer behaves similarly to detect this failure.
Lastly,
the overall mission is completed at~$t = 1380s$,
which is slightly higher than the~$1006s$ in the nominal scenario
without failures.
It is worth noting that
after the failures of inspectors~$2$ and~$3$,
the explorer tends to have fewer meetings but assign more features to inspector~$1$.

\subsubsection{Multiple GCS}
\label{subsubsec:exp-multiple-gcs}

\begin{figure}[!t]
  \centering
  \includegraphics[width=0.99\linewidth]{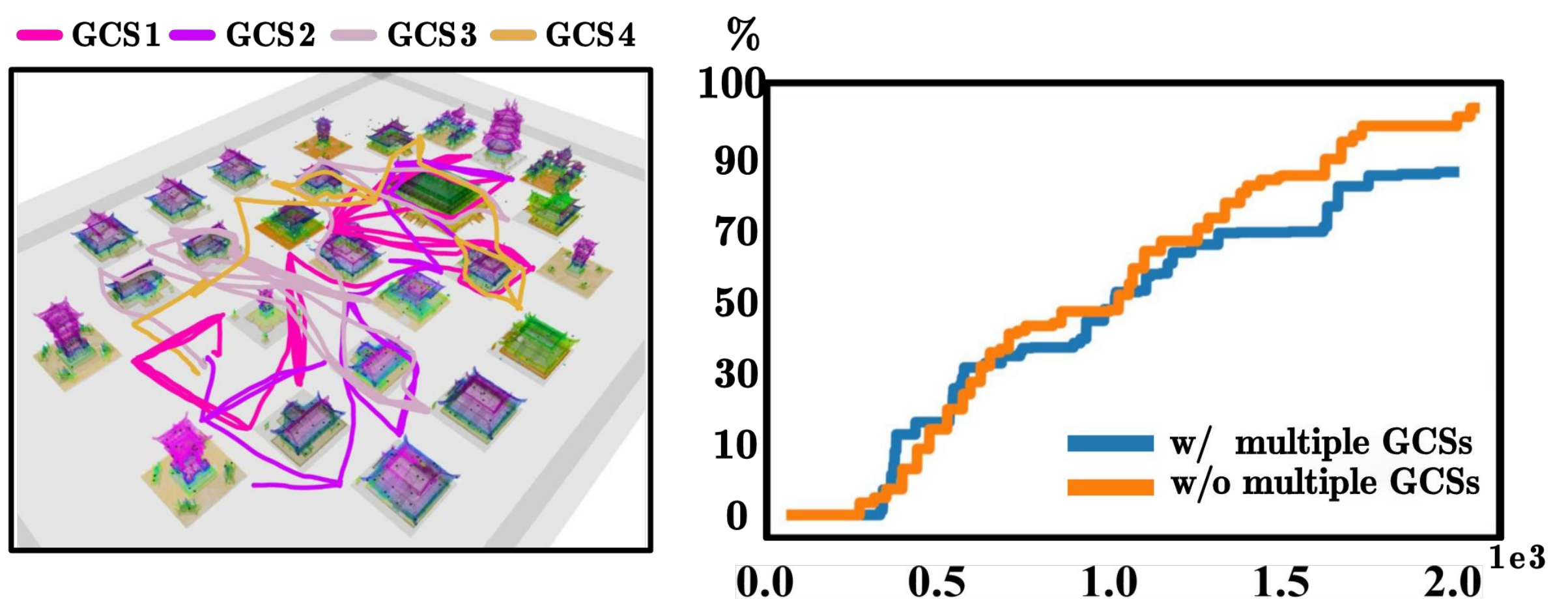}
  \centering
  \caption{Scenario-D with~$4$ GCS, $8$ explorers and $40$ inspectors:
  final meeting trajectories of all GCSs (\textbf{Left});
  comparison of collected features over time (\textbf{Right}).
  }
  \label{fig:multi-gcs}
  \vspace{-5mm}
\end{figure}

The proposed SLEI3D algorithm is evaluated in Scenario-D under two configurations:
one leveraging a multi-GCS setup with~$4$ GCSs, $8$ explorers, and~$40$ inspectors,
and the other employing a single-GCS configuration.
Specifically,
each GCS is pre-assigned to specific subsets of BBoxes
and a subset of explorers, each of which acts independently.
The final trajectories of all GCSs
and the collected features along with time
are shown in Fig.~\ref{fig:multi-gcs}.
It is evident that each GCS successfully meets with the assigned explorers,
thus collecting all features.
However,
it is interesting to notice that SLEI3D with multiple GCSs initially achieves a
faster feature collection compared to its single-GCS counterpart.
However, as the process progresses,
the SLEI3D catches up and eventually surpasses the performance of
the multi-GCS configuration.
This is due to the fact that all explorers start from the same location
and are assigned BBoxes of similar sizes.
In other words, the completion times of these BBoxes are almost
synchronized via the proposed routing scheme for multiple GCSs,
which allows each GCS to concurrently collect features from its assigned BBoxes.
In contrast, {SLEI3D} can only sequentially collect features from multiple BBoxes,
yielding a much slower collection of features.
As the mission proceeds,
the subsets of BBoxes to be explored varies significantly in volume,
yielding a diverse distribution of completion time.
Consequently, the benefits of having multiple GCSs diminish
as the rate of feature collection becomes less dependent on
the distributed coordination.
Thus, it shows that the multi-GCS counterpart is particularly suitable
for scenarios requiring simultaneous feature collection
across a wide distribution of locations.

\begin{figure}[!t]
  \centering
  \includegraphics[width=0.99\linewidth]{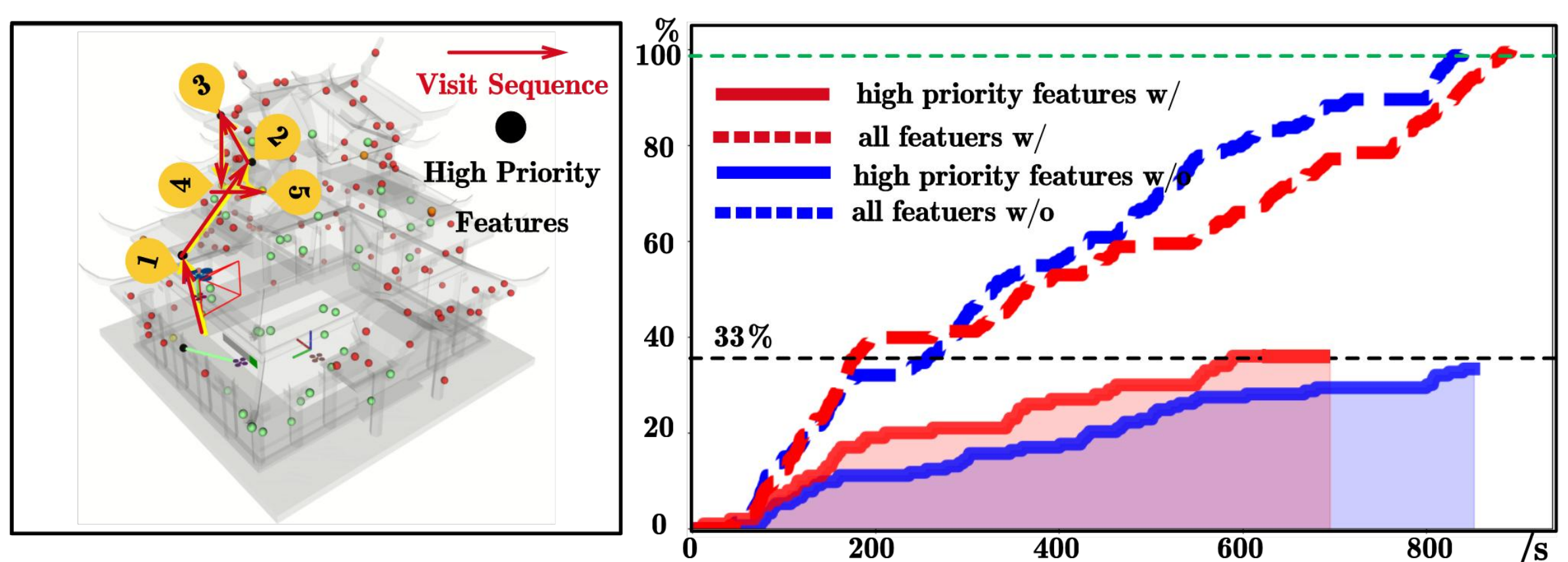}
  \centering
  \caption{
   High-priority features in Scenario-B
   with~$150$ features and~$50$ high-priority features:
   inspection sequence of features with high-priority (\textbf{Left});
   comparison of collected high-priority features over time (\textbf{Right}).
   }
  \label{fig:high-priority-features}
  \vspace{-6mm}
\end{figure}

\begin{figure}[!t]
  \centering
  \includegraphics[width=0.99\linewidth]{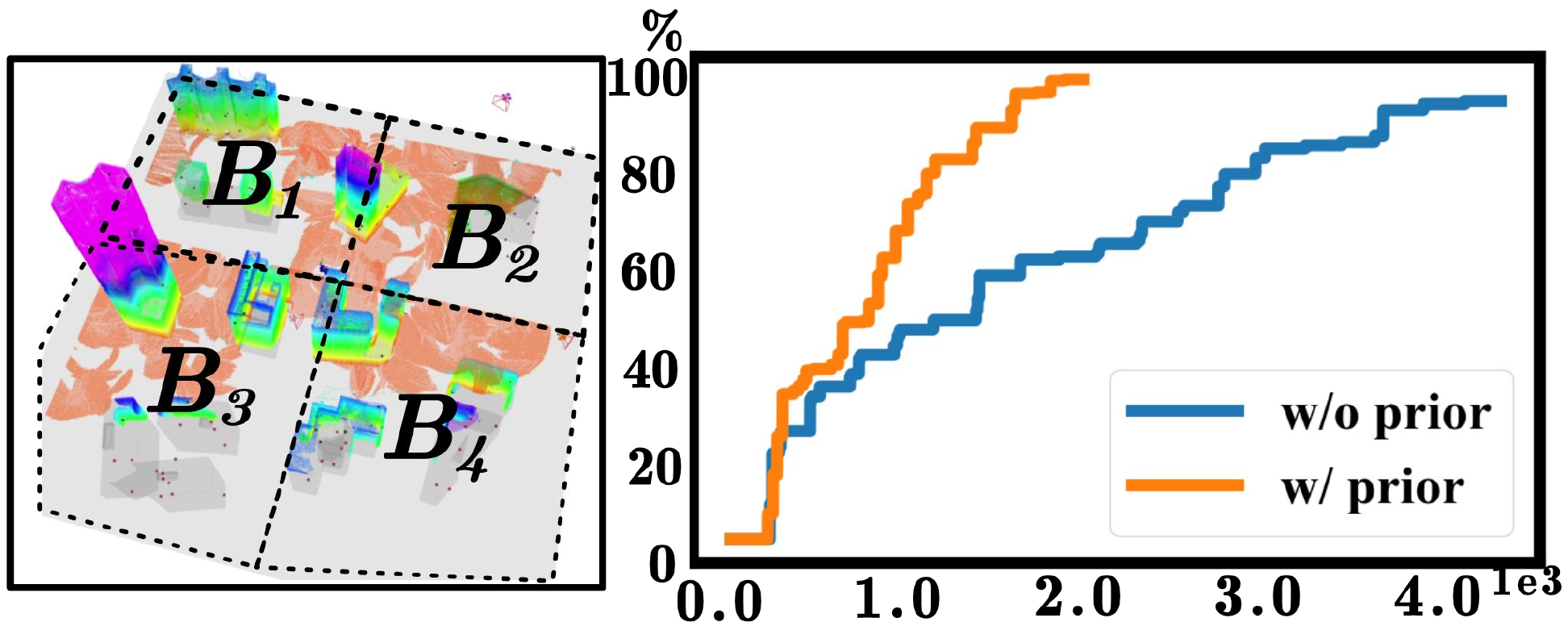}
  \centering 
  \caption{
   {Overall evolution for our system without any priors under Scenario-C (\textbf{left});}
   comparison of collected features over time (\textbf{Right}).
   }
  \label{fig:no-prior}
  \vspace{-4mm}
\end{figure}

\subsubsection{High Priority Features}

{To further validate SLEI3D with high-priority features,}
$50$ features of the~$150$ features in Scenario-B are selected as high-priority features.
Snapshots of the inspector  trajectory along with the priority of these features
are shown in Fig.~\ref{fig:high-priority-features},
which highlights the fact that high-priority features are inspected first.
Namely, the inspector prioritizes~$3$ high-priority features
out of~$5$ features.
Moreover, the collection of high-priority features
or all features over time are also shown.
It can be seen that SLEI3D with consideration for high-priority features achieves much faster collection of
high-priority features than SLEI3D,
while the overall completion time is slightly longer ($842s$ vs. $893s$).
This is because prioritizing high-priority features prevents
simultaneous optimization of the allocation for all features,
resulting in an increase in the overall completion time.

\subsubsection{Fully-unknown Environment}

To validate SLEI3D's capability in completely unknown environments,
comparative experiments are conducted for Scenario-C under~$1$ GCS,~$4$ explorers, and~$8$ inspectors
with and without prior information regarding BBoxes.
As shown in Fig.~\ref{fig:no-prior},
prior-free exploration exhibits significantly slower feature discovery rates for explorers.
This delay is due to the absence of BBoxes,
compelling explorers to exhaustively explore uniform subspaces
without prioritizing targeted areas.
Consequently, the GCS experiences significant delays in collecting the features
(1832s vs. 4000s), which validates the effectiveness of the proposed
mechanism for adaptive partitioning.
It also highlights how prior information regarding BBoxes can accelerate the overall mission of exploration and inspection.

\begin{figure}[!t]
  \centering
  \includegraphics[width=0.9\linewidth]{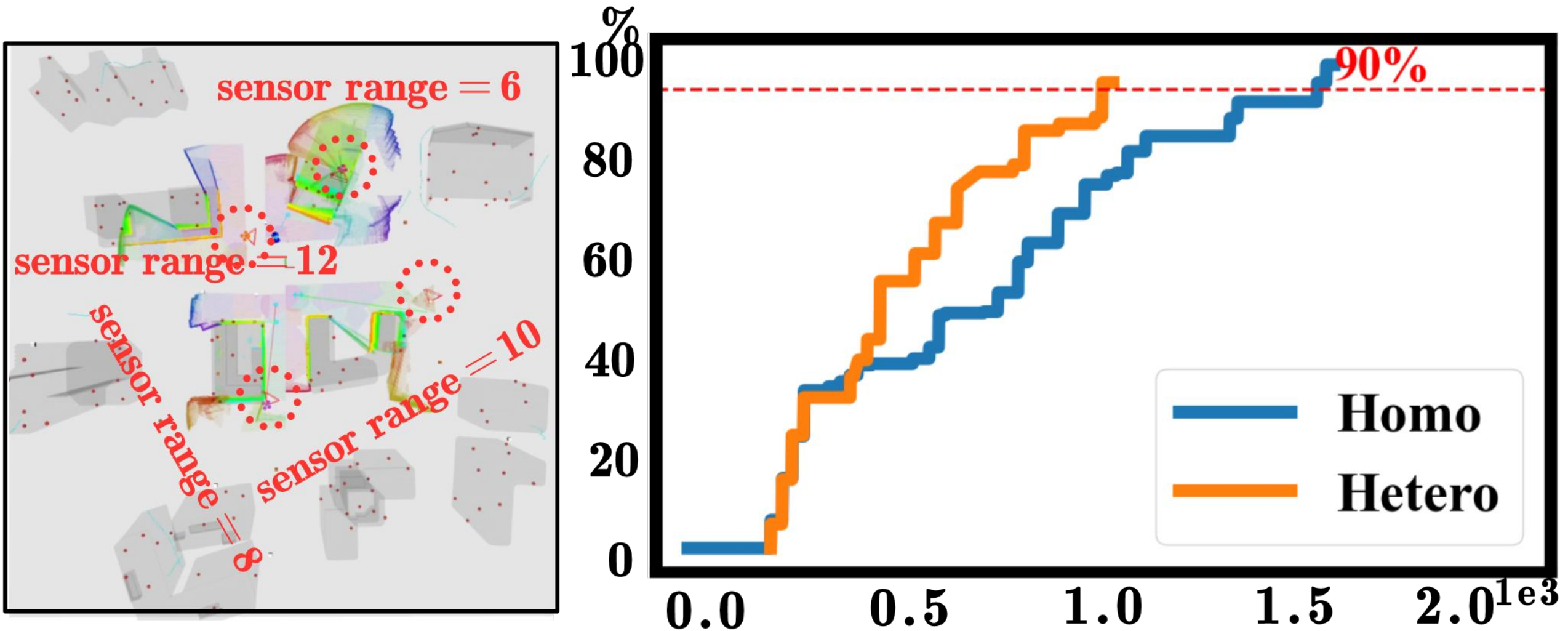}
  \centering
  \caption{
   Heterogeneous sensor range of explorers in Scenario-C with
   $1$ GCS,~$4$ explorers and~$8$ inspectors (\textbf{left});
   comparison of collected features over time (\textbf{right}).
   }
  \label{fig:hetero-exp}
  \vspace{-2mm}
\end{figure}

\begin{figure}[!t]
  \centering
  \includegraphics[width=0.7\linewidth, height=0.5\linewidth]{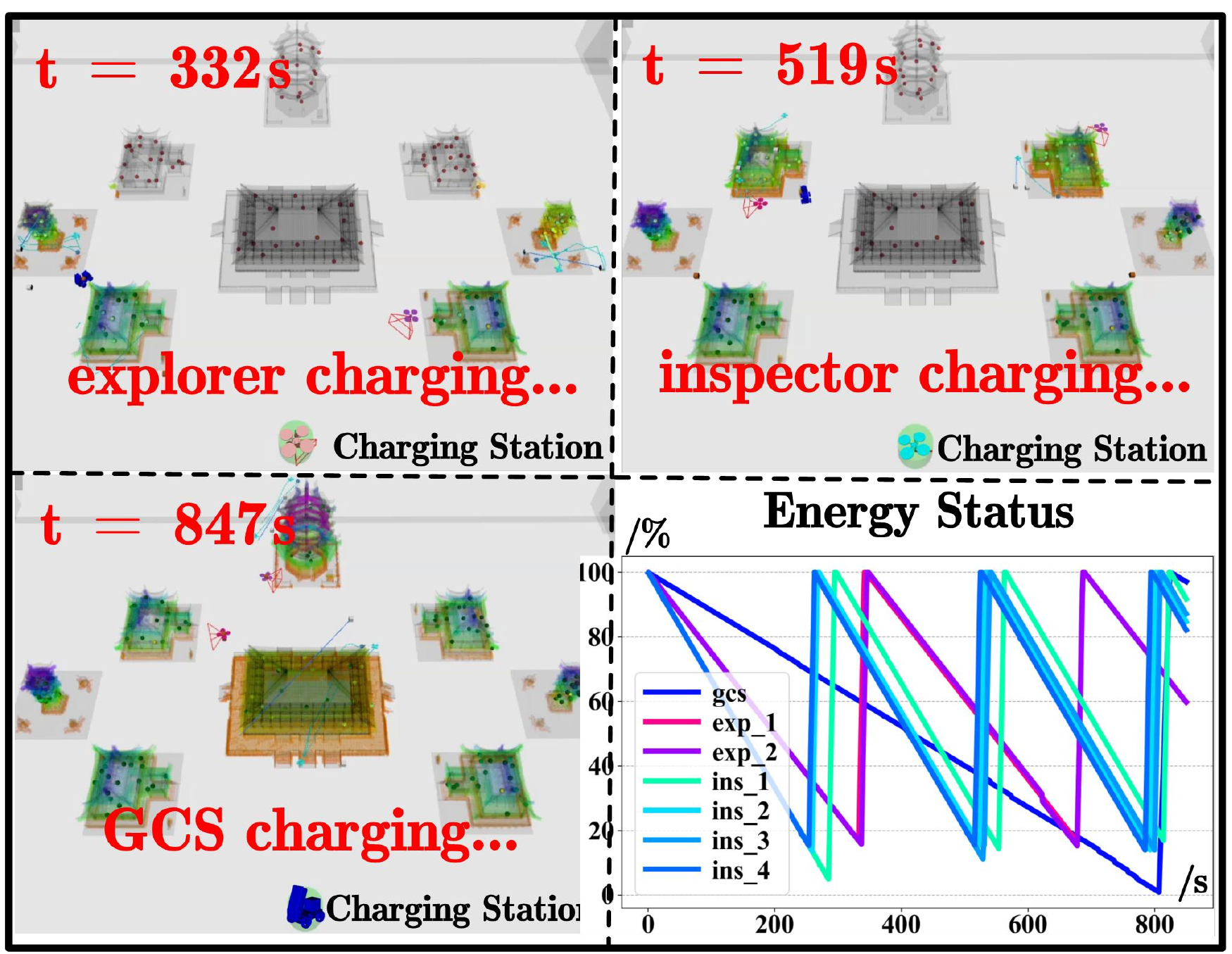}
  \centering
  \caption{
    The representative snapshots of recharging for diffrent role of robots via returning to the fixed charging station
    under Scenario-A and the overall evolution of energy status for entire fleet.
   }
  \label{fig:battery}
  \vspace{-6mm}
\end{figure}

\subsubsection{Heterogeneous Capability of Explorers}

To verify the support for heterogeneous fleets,
the explorers in Scenario-C are set to different sensing ranges
($6m$, $8m$, $10m$, $12m$)
versus homogeneous units (uniform $5m$ sensors).
As shown in Fig.~\ref{fig:hetero-exp},
the heterogeneous system achieves faster feature discovery and inspector dispatch through longer ranges.
This takes $40\%$ less time to collect around~$90\%$ of features ($1005s$ vs. $1656s$),
compared to the homogeneous counterparts.
Notably, explorers with larger sensing ranges are prioritized for larger BBoxes via receding-horizon allocation by GCS ,
which aligns with the time-minimization objective as presented in~\eqref{eq:prb_formulation}.
The results confirm the compatibility of the proposed framework
to heterogeneous fleets,
with the overall performance positively correlating with increased sensor ranges.

\begin{figure*}[!t]
  \centering
  \includegraphics[width=0.99\linewidth,height=0.36\linewidth]{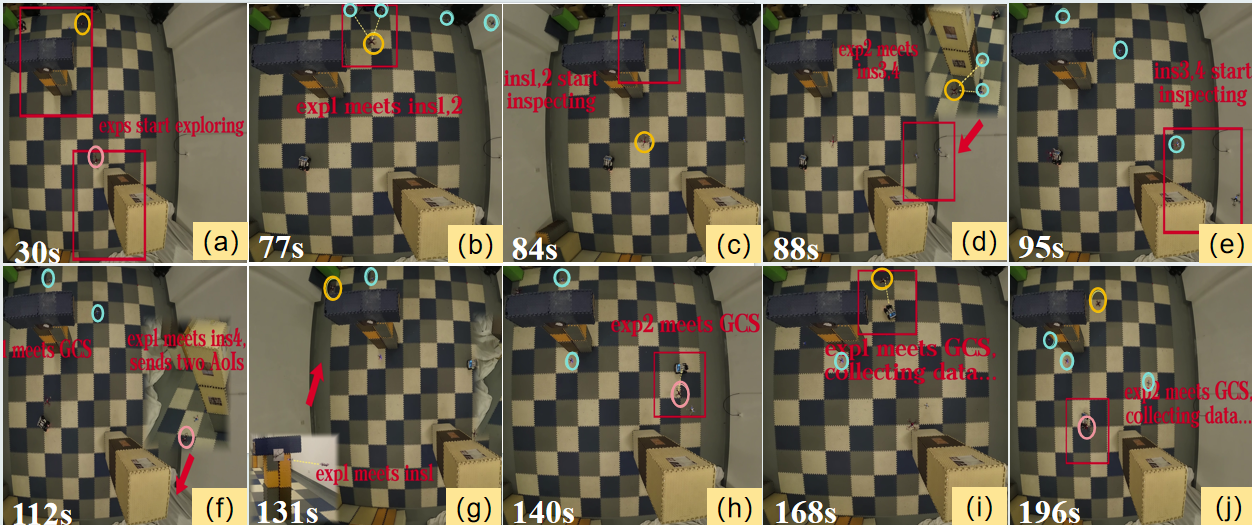}
  \centering
  \caption{
     Snapshots of the hardware experiments involving~$2$ explorers (marked by yellow and pink circles),
     $4$ inspectors (marked by indigo circles), and $1$ GCS for the collaborative inspection
     of $2$ structures and $9$ AoIs.
     \textbf{(a)}: The explorers initiate the exploration;
     \textbf{(b)}: Explorer~1 identifies~$2$ AoIs and coordinates with inspectors~1\&2;
     \textbf{(c)}: AoIs are allocated to inspectors~1\&2;
     \textbf{(d)}: Explorer~2 detects~$2$ additional AoIs and coordinates with inspectors~3\&4;
     \textbf{(e)}: AoIs are allocated to inspectors~3\&4 to start inspecting;
     \textbf{(f)}: Explorer~1 initiates its first data transmission to the GCS;
     \textbf{(g)}: Explorer~1 discovers another AoI and coordinates with inspector~1;
     \textbf{(h)}: Explorer~2 establishes its first GCS communication;
     \textbf{(i)}: Explorer~1 transmits the inspected features and subsequent meetings to the GCS;
     \textbf{(j)}: Explorer~2 transmits the inspected features and subsequent meetings to the GCS.
  }
  \label{fig:exp-snapshot}
  \vspace{-1mm}
\end{figure*}

\usetikzlibrary{patterns}

\pgfplotstableread[col sep=comma]{
    type,       communication, allocated-BBox, exploration, decision-making, prediction, ins-localPlan
    Scenario-A, 0.0035,        0.003,          0.036,       1.384,           0.000047,   0.050
    Scenario-C, 0.003,         0.005,          0.02,        1.273,           0.000020,   0.050
    Scenario-D, 0.008,         0.004,          0.065,       1.283,           0.000050,   0.070
    }\mytable
\pgfplotstableread[col sep=comma]{
    type,       communication, allocated-BBox, exploration, decision-making, prediction, ins-localPlan
    Ins-time 1, 0.003,         0.003,          0.04,        1.383,           0.000035,    0.003
    Ins-time 3, 0.0035,        0.003,          0.036,       1.384,           0.000047,    0.0035
    Ins-time 5, 0.003,         0.003,          0.055,       1.083,           0.000030,    0.003
    }\Ntable
\pgfplotstableread[col sep=comma]{
    type,   communication, allocated-BBox, exploration, decision-making, prediction, ins-localPlan
    N=4,    0.005,         0.0007,         0.025,        1.13,            0.000267,      0.01
    N=13,   0.008,         0.004,          0.065,        1.283,           0.00005,       0.01
    N=49,   0.005,         0.04,           0.07,         5.305,           0.00006,       0.07
    N=65,   0.005,         0.01,           0.04,         6.806,           0.00005,       0.12
    N=75,   0.004,         0.01,           0.05,         6.855,           0.0003,        0.14
    }\Vtable
  \pgfplotstableread[col sep=comma]{
    type,          ours,     caric,     soar,     mu-cppi
    Scenario-A,    0.118,    0.0307,    0.0535,   0.2655
    Scenario-C,    1.4225,   0.1611,    0.0641,   0.0491
    Scenario-D,    0.357,    0.4465,    0.081,    0.2335
  }\Mtable

\definecolor{v0}{rgb}{0.1, 0.3, 0.8}          
\definecolor{v1}{rgb}{0.6, 0.4, 0.1}          
\definecolor{v2}{rgb}{0.2, 0.7, 0.4}          
\definecolor{v3}{rgb}{0.7, 0.1, 0.5}          
\definecolor{v4}{rgb}{1.0,1.0. 0.0}           
\definecolor{v5}{rgb}{0.7, 0.7, 0.7}          

\pgfplotsset{/pgfplots/error bars/error bar style={thick, black}}
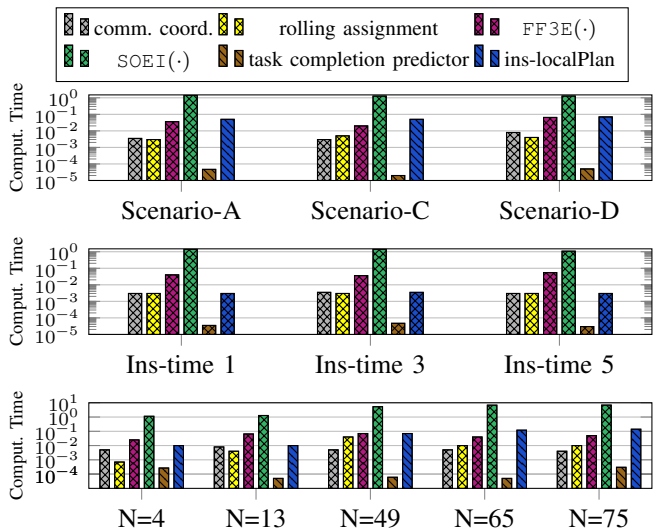
\begin{figure}[t]
	\centering

    \begin{subfigure}[t]{0.50\textwidth}
        \begin{tikzpicture}
          \begin{semilogyaxis}[
            width=\linewidth,
            ybar,
            bar width=5pt,
            log origin = infty,
            ymin=0.00001,
            ymax=1.5,
            xtick pos=bottom,
            ytick={0.00001,0.0001,0.001, 0.01, 0.1, 1},
            yticklabel style = {font=\scriptsize},
            minor y tick num=10,
            width=1.0\linewidth,
            height=0.30\linewidth,
            ymajorgrids=true,
            enlarge x limits={abs=35pt},
            legend style={at={(0.45,1.15)},
            anchor=south,legend columns=3, font=\footnotesize},
            ylabel={Comput. Time},
            y label style={at={(axis description cs:0.07,0.5)},anchor=south, font=\scriptsize},
            symbolic x coords={Scenario-A, Scenario-C, Scenario-D},
            x label style={font=\tiny},
            xtick=data,
            cycle list={v5, v4, v3, v2, v1, v0},
          ]

              \addplot+[draw=black,fill,postaction={
                pattern=crosshatch
            }] table[x=type,y=communication]{\mytable};
            \addlegendentry{comm. coord.}

              \addplot+[draw=black,fill,postaction={
                pattern=crosshatch
            }] table[x=type,y=allocated-BBox]{\mytable};
            \addlegendentry{rolling assignment}

              \addplot+[draw=black,fill,postaction={
                pattern=crosshatch
            }] table[x=type,y=exploration]{\mytable};
            \addlegendentry{$\texttt{FF3E}(\cdot)$}

              \addplot+[draw=black,fill,postaction={
                pattern=crosshatch
            }] table[x=type,y=decision-making]{\mytable};\
            \addlegendentry{$\texttt{SOEI}(\cdot)$}

              \addplot+[draw=black,fill,postaction={
                pattern=north west lines
            }] table[x=type,y=prediction]{\mytable};
            \addlegendentry{task completion predictor}

            \addplot+[draw=black,fill,postaction={
                pattern=north west lines
            }] table[x=type,y=ins-localPlan]{\mytable};
            \addlegendentry{ins-localPlan}

        \end{semilogyaxis}
        \end{tikzpicture}
      \end{subfigure}


    \begin{subfigure}[t]{0.50\textwidth}
        \begin{tikzpicture}
        \begin{semilogyaxis}[
            width=\linewidth,
            ybar,
            bar width=5pt,
            log origin = infty,
            ymin=0.00001,
            ymax=1.5,
            xtick pos=bottom,
            ytick={0.00001,0.0001,0.001, 0.01, 0.1, 1},
            yticklabel style = {font=\scriptsize},
            minor y tick num=10,
            width=1.0\linewidth,
            height=0.30\linewidth,
            ymajorgrids=true,
            enlarge x limits={abs=35pt},
            ylabel={Comput. Time},
            y label style={at={(axis description cs:0.07,0.5)},anchor=south, font=\scriptsize},
            symbolic x coords={Ins-time 1, Ins-time 3, Ins-time 5},
            x label style={font=\tiny},
            xtick=data,
            cycle list={v5, v4, v3, v2, v1, v0},
        ]
            \addplot+[draw=black,fill,postaction={
                pattern=crosshatch
            }] table[x=type,y=communication]{\Ntable};

            \addplot+[draw=black,fill,postaction={
                pattern=crosshatch
            }] table[x=type,y=allocated-BBox]{\Ntable};

            \addplot+[draw=black,fill,postaction={
                pattern=crosshatch
            }] table[x=type,y=exploration]{\Ntable};

            \addplot+[draw=black,fill,postaction={
                pattern=crosshatch
            }] table[x=type,y=decision-making]{\Ntable};

            \addplot+[draw=black,fill,postaction={
                pattern=crosshatch
            }] table[x=type,y=prediction]{\Ntable};

            \addplot+[draw=black,fill,postaction={
                pattern=crosshatch
            }] table[x=type,y=ins-localPlan]{\Ntable};

        \end{semilogyaxis}
        \end{tikzpicture}
    \end{subfigure}


    \begin{subfigure}[t]{0.50\textwidth}
        \begin{tikzpicture}
          \begin{semilogyaxis}[
            width=\linewidth,
            ybar,
            bar width=3.6pt,
            log origin = infty,
            ymin=0.00001,
            ymax=10,
            xtick pos=bottom,
            ytick={0.0001,0.0001,0.001, 0.01, 0.1, 1, 10},
            yticklabel style = {font=\scriptsize},
            minor y tick num=10,
            width=1.0\linewidth,
            height=0.30\linewidth,
            ymajorgrids=true,
            enlarge x limits={abs=20pt},
            ylabel={Comput. Time},
            y label style={at={(axis description cs:0.07,0.5)},anchor=south, font=\scriptsize},
            x label style={font=\tiny},
            symbolic x coords={N=4, N=13, N=49,N=65, N=75},
            xtick=data,
            cycle list={v5, v4, v3, v2, v1, v0},
          ]

              \addplot+[draw=black,fill,postaction={
                pattern=crosshatch
            }] table[x=type,y=communication]{\Vtable};

              \addplot+[draw=black,fill,postaction={
                pattern=crosshatch
            }] table[x=type,y=allocated-BBox]{\Vtable};

              \addplot+[draw=black,fill,postaction={
                pattern=crosshatch
            }] table[x=type,y=exploration]{\Vtable};

              \addplot+[draw=black,fill,postaction={
                pattern=crosshatch
            }] table[x=type,y=decision-making]{\Vtable};

              \addplot+[draw=black,fill,postaction={
                pattern=north west lines
            }] table[x=type,y=prediction]{\Vtable};

            \addplot+[draw=black,fill,postaction={
                pattern=north west lines
            }] table[x=type,y=ins-localPlan]{\Vtable};

        \end{semilogyaxis}
        \end{tikzpicture}
      \end{subfigure}

      \caption{
    Computation time of main components of the proposed method:
    for three scenarios (\textbf{Top});
    change of computation time under different inspection time
    in Scenario-C (\textbf{Middle});
    with different fleet sizes in Scenario-A (\textbf{Bottom}).}
  \label{fig:efficiency}
\vspace{-5mm}
\end{figure}


\subsubsection{Energy-constrained Fleet Management}
\label{subsubsec:energy-manage}

To validate that the proposed framework can handle energy constraints as described in Sec.~\ref{subsec:energy},
a team comprising of 1 GCS
with $\overline{E}_0 = 800~mAh$ energy, 2 explorers ($\overline{E}_i = 400~mAh, \forall i \in \mathcal{N}_{\texttt{e}}$),
and 4 inspectors ($\overline{E}_j = 300~mAh, \forall j \in \mathcal{N}_{\texttt{O}}$)
is deployed within the Scenario-A.
A unified protocol for recharging is activated when the energy level of any robot drops below $\underline{E} = 20\%\overline{E}$,
with all charging stations providing $\beta = 10\%\overline{E}$ charging rate per second.
The key events captured in Fig.~\ref{fig:battery} demonstrate the overall efficiency under energy constraints:
At $t=332$~s, Explorer~1 initiates the protocol of recharging once
it reaches the critical threshold~$\underline{E}$ (as visualized in the plot of energy status),
followed by a similar event for Inspector~3 at $t=540$~s.
The GCS executes its scheduled recharge at $t=780$~s while maintaining the planned communication events
with the explorers.
Throughout the mission duration of $878$~s when all AoIs are inspected,
the team performs in total 1, 2 and 3 recharging events for the GCS, explorers and inspectors, respectively.
The energy levels of all robots are above $5\%$ at all time.

\subsection{Computation Complexity Analysis}
\label{subsec:comp-time-analysis}

To facilitate the broader application of the proposed framework by practitioners,
the computation complexity of the proposed framework
is analyzed across scenarios of varying scales,
different inspection time and fleet sizes.
Specifically,
the computation time mainly consists of three parts:
{(I) the local planning for GCS,
including the coordination of communication events
and the rolling assignment of BBoxes;
(II) the local planning for explorers,
involving the~$\texttt{FF3E}(\cdot)$ for exploration,
the planning via~$\texttt{SOEI}(\cdot)$
and the predictor of completion time;
(III) the local planning for inspectors,
encompassing the update of local plan~$\{\xi_j^+\}$.
}

To begin with,
the computation time of each step above is analyzed in different scenarios,
where~$1$ GCS,~$4$ explorers and~$8$ inspectors are deployed in three scenarios from Fig.~\ref{fig:sys-scene}
under inspection time~$t_q=3s$.
As summarized in Fig.~\ref{fig:efficiency},
all modules except the~$\texttt{SOEI}(\cdot)$ for the explorers,
take less than~$0.1s$ across all scenarios.
{The~$\texttt{SOEI}(\cdot)$ algorithm},
which relies on an iterative optimization process,
dominates the computation time of the local planning of GCS.
Nonetheless,
the total computation time of each step is less than~$1.3s$ on average.
Second,
the same analysis is examined for different inspection time~$t_q$,
i.e., $1s$, $3s$, and $5s$ in Scenario-C with~$1$ GCS,~$4$ explorers
and~$8$ inspectors.
It shows that
the inspection time slightly affects the computation time of~$\texttt{SOEI}(\cdot)$module,
but the overall computation time of each step is less than~$0.25s$ on average.
The same analysis is performed across different fleet sizes,
i.e., $4$ ($1$ GCS,~$1$ explorer,~$2$ inspectors),
$13$ ($1$ GCS,~$4$ explorers,~$8$ inspectors),
$49$ ($1$ GCS,~$8$ explorers,~$40$ inspectors),
$65$ ($5$ GCS,~$10$ explorers,~$50$ inspectors),
and $75$ ($5$ GCS,~$10$ explorer,~$60$ inspectors) in Scenario-A.
It shows that the computation time grows steadily as the fleet size increases,
ranging from~$0.13s$ to~$6.91s$ on average,
as the planning requires more interactions to converge
when the number of explorers and inspectors increases.
Nonetheless, the overall computation for $N=75$ remains below~$8.82s$,
showcasing its scalability to large-scale robotic fleets.

\subsection{Hardware Experiments}
\label{subsec:real-exp}

\begin{figure}[!t]
  \centering
  \includegraphics[width=0.8\linewidth]{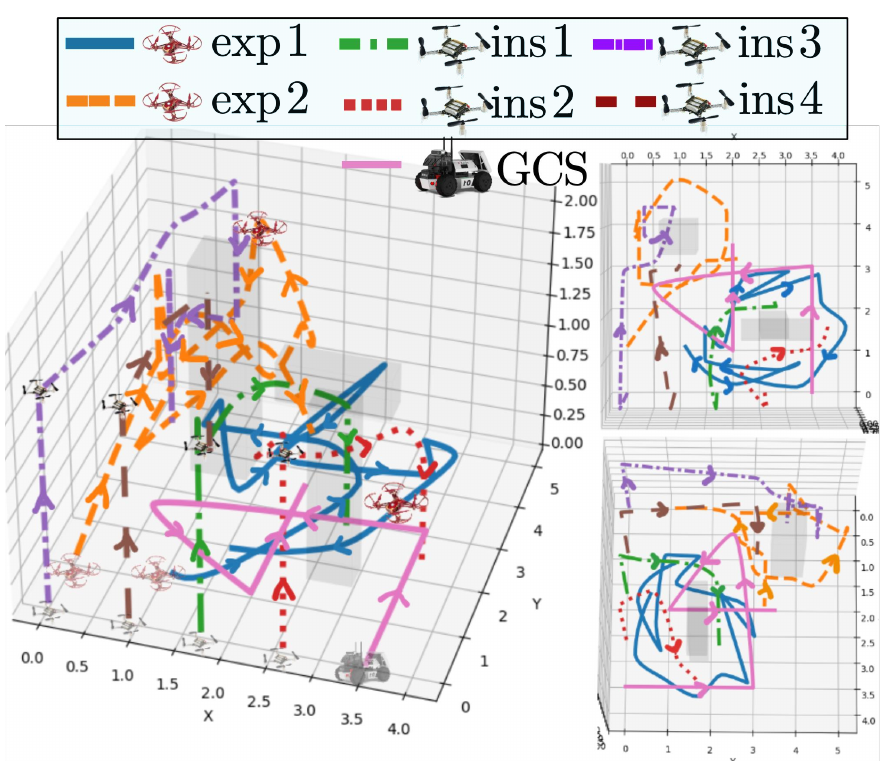}
  \centering
  \caption{
     Complete trajectories of all $8$ robots over the mission duration of $200s$,
     with arrows indicating the timeline in different of views:
     Front view (\textbf{left}), top view (\textbf{right-top}), and side view (\textbf{right-bottom}).
   }
  \label{fig:exp-trajectory}
  \vspace{-4mm}
\end{figure}

\begin{figure}[!t]
  \centering
  \includegraphics[width=0.7\linewidth, height=0.3\linewidth]{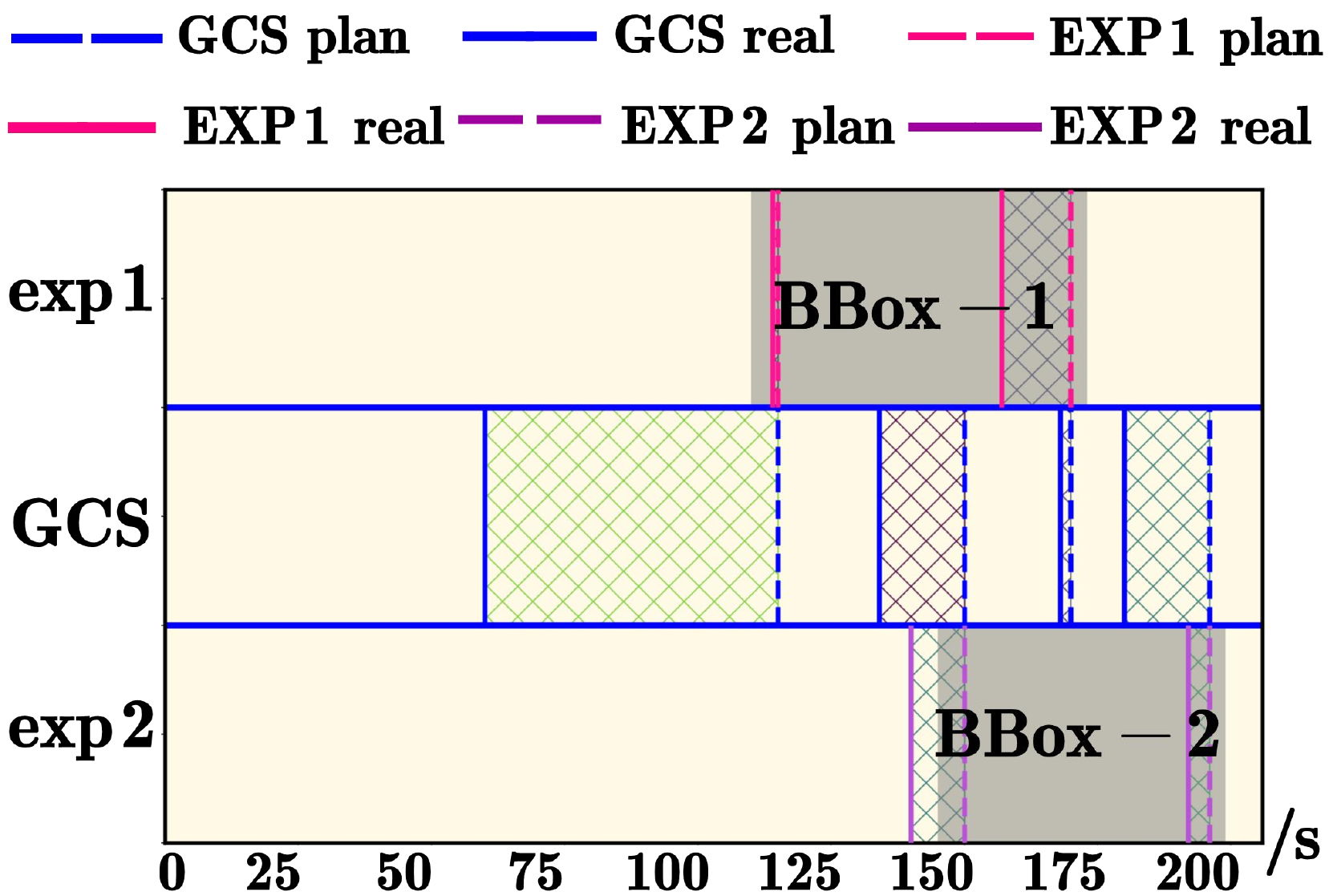}
  \centering
  \caption{
    The planned and actual meeting events between the GCS and two explorers
    in the hardware experiments.
    The planned meeting time for the GCS (blue dashed lines)
    aligns with the explorers (pink and purple dashed lines),
    indicating the confirmed meeting event and the target BBox.
   }
  \label{fig:exp-gcs-exp}
  \vspace{-4mm}
\end{figure}

\subsubsection{Experimental Setup}
\label{subsubsec:real-exp-setup}

Hardware experiments are also conducted to validate the practical feasibility of the proposed framework.
As shown in Fig.~\ref{fig:exp-snapshot},
there is a $~5m\times5m\times3m$ arena containing two mini-nature 3D structures,
a T-shape structure of size~$1.8m\times0.3m\times1.9m$
and a rectangular structure of size~$0.5m \times 0.5m \times 2.0m$.
In total,~$9$ features of interest as cracks are printed at unknown locations on their surfaces.
Moreover, two ``Tello" UAVs~\cite{tello} equipped with monocular cameras serve as explorers,
to perform exploration via 3D reconstruction.
These explorers navigate at the speed of~$0.7m/s$ with a kinodynamic planner for obstacle avoidance
and the communication range of $0.4m$.
In addition,
four ``Crazyflie" drones~\cite{crazyflie} operate as inspectors at $0.3m/s$,
employing an improved A$^\star$ algorithm with priority-based collision avoidance and a communication range of $0.2m$.
The AoIs on the surfaces are identified via the YOLOv7 package,
the semantic segmentation module guides the FOV of the explorers towards to the target buildings,
and reconstruction of the 3D environment by fusing RGB images with estimated depth information,
which is predicted from monocular images of ``Tello" using the MiDas network~\cite{birkl2023midas}. 
Lastly, one ``LIMO"~\cite{limo} four-differential ground robot acts as the moveable GCS,
utilizing a EAI XL2 LiDAR for the real-time mapping and the \texttt{move\_base} package for autonomous navigation.
All robots are tracked by an OptiTrack motion capture system with~$20$ infrared cameras for positioning data,
while all inter-robot communications are strictly limited to Line-of-Sight (LOS) constraints.
It should be noted that the communication constraints among the robots are enforced centrally via a central node,
including the relative distance, obstacle occlusion and bandwidth.

\subsubsection{Results}
\label{subsubsec:real-exp-result}

\begin{figure}[!t]
    \centering
    \includegraphics[width=0.8\linewidth]{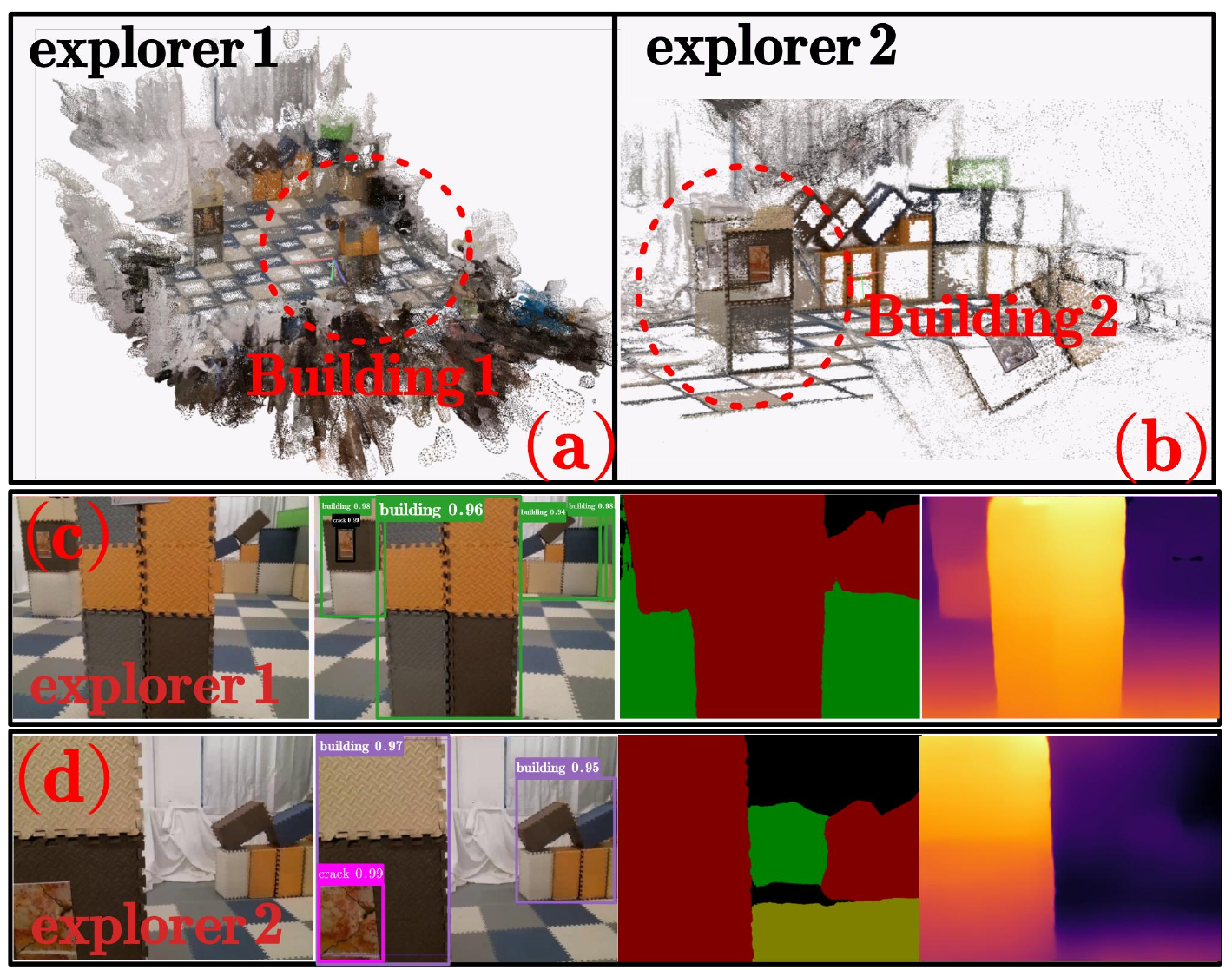}
    \centering
    \caption{
     Perception results in the 3D environment via~$2$ explorers with the monocular cameras,
     including:
     re-constructed 3D environment for explorer~1(\textbf{(a)}) and explorer~2(\textbf{(b)}),
     recognition results of the AoIs and building structures with the 
     raw camera images, AoI detection, semantic segmentation and depth estimation for explorer~1 (\textbf{(c)}),
     and explorer~2 (\textbf{(d)}).
    }
    \label{fig:exp-reconstruction}
    \vspace{-4mm}
  \end{figure}

\begin{figure}[!t]
    \centering
    \includegraphics[width=0.7\linewidth,height=0.3\linewidth]{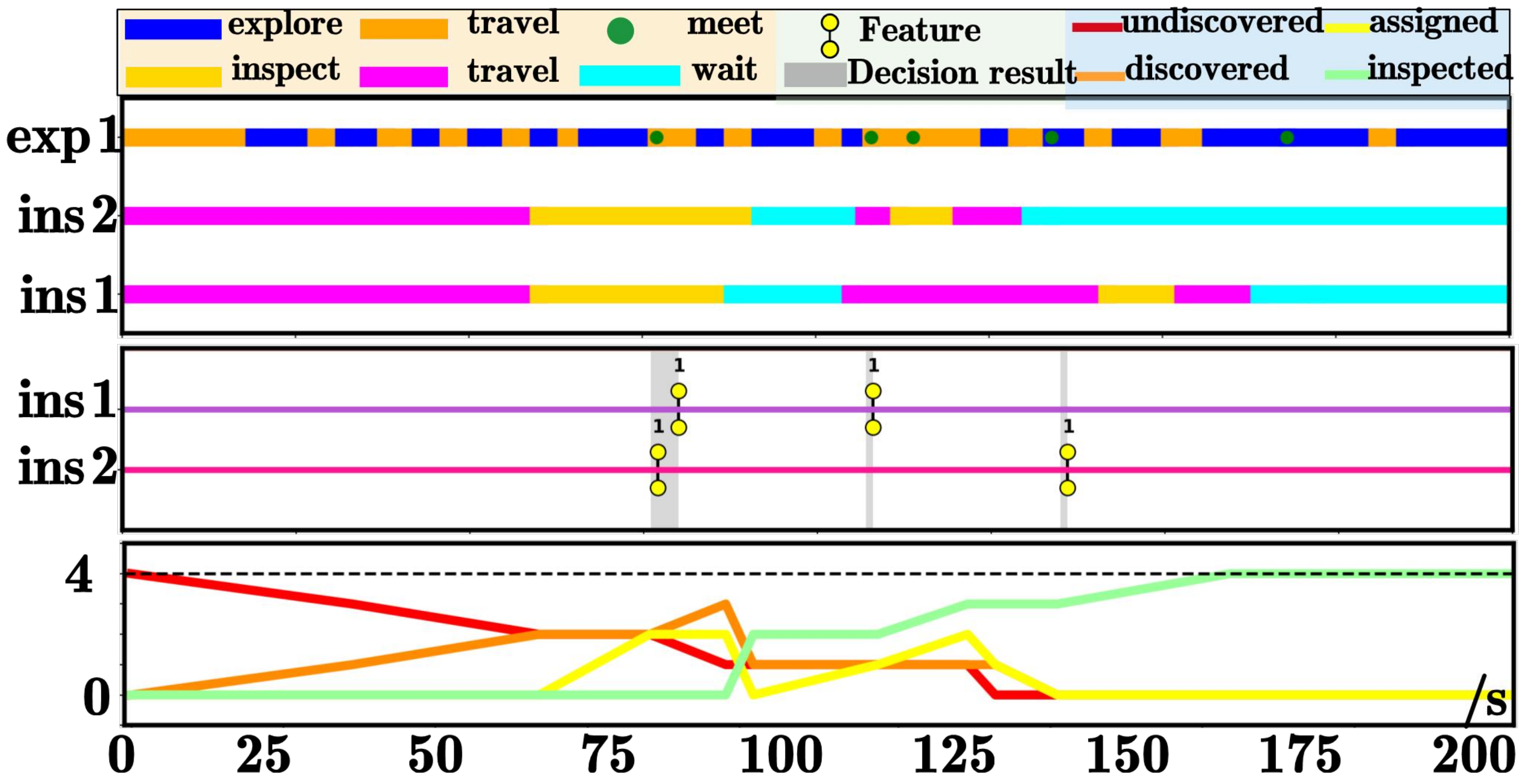}
    \centering
    \caption{
      \textbf{Top}: Evolution of the robot status within one subgroup;
      \textbf{Middle}: Planning results during the meeting events between the explorer 1
      and the inspectors 1 and 2;
      \textbf{Bottom}: Evolution of number of features being discovered, assigned and inspected
      over time.
     }
    \label{fig:exp-exp-ins}
    \vspace{-4mm}
  \end{figure}

The hardware experiment validates the performance of the proposed framework during the~$200s$ mission.
As shown in Fig.~\ref{fig:exp-snapshot},
$2$ explorers initiate the exploration at $t=30s$, as marked by red boxes in Fig.~\ref{fig:exp-snapshot}(a).
Explorer~1 identifies~$2$ AoIs that might contain cracks by $t=77s$ and coordinates with
inspectors~1\&2 through subsequent communication meetings as shown in Fig.~\ref{fig:exp-snapshot}(b).
The AoIs are allocated following the $\texttt{SOEI}$ algorithm,
which directs the inspectors to their respective regions in Fig.~\ref{fig:exp-snapshot}(c).
Simultaneously, explorer~2 detects~$2$ additional AoIs at $t=88s$  in Fig.~\ref{fig:exp-snapshot}(d) and coordinates with inspectors~3\&4 in Figs.~\ref{fig:exp-snapshot}(e-f).
Concurrently, explorer~1 initiates its first data transmission to the movable GCS at $t=112s$ in Fig.~\ref{fig:exp-snapshot}(f),
including the inspected features and the subsequent meetings.
At $t=131s$, explorer~1 discovers another AoI and coordinates with inspector~1,
while explorer~2 establishes its first GCS communication at $t=140s$ in Fig.~\ref{fig:exp-snapshot}(h).
Both explorers strictly follow the confirmed meeting events,
with the last communication to GCS at $t=168s$ and $t=196s$ shown in Figs.~\ref{fig:exp-snapshot}(i-j).
By then, all 3D structures and inspected AoIs are available at the GCS.
Furthermore, the trajectories of all robots are recorded in Fig.~\ref{fig:exp-trajectory},
which indicate that (i) the explorers achieve a complete coverage and all AoIs are successfully inspected, as shown in Fig.~\ref{fig:exp-reconstruction};
(ii) the GCS coordinates with both explorers for the optimal meeting events at four locations, as shown in Fig.~\ref{fig:exp-gcs-exp};
and (iii) all robots are collision-free with other robots and the 3D structures.
Lastly, the detailed evolution of robot status within the first subgroup is summarized in Fig.~\ref{fig:exp-exp-ins},
which validates the robustness of the proposed scheme under fluctuations in navigation time due to motion uncertainties and collision avoidance.
Detailed videos of the hardware experiments can be found in the supplementary material.

\section{Conclusion}
\label{sec:conclusion}

This work tackles the challenge of simultaneous exploration, inspection,
and communication in large-scale, unknown 3D environments,
via heterogeneous robotic fleets under limited communication.
The proposed framework SLEI3D
employs a multi-layer and multi-rate strategy,
dividing the fleet into subgroups based on sensing capabilities,
explorers equipped with long-range sensors identify areas of interest (AoIs),
while inspectors with close-range sensors inspect the features.
A ground control station (GCS) allocates tasks and facilitates communication
using an intermittent protocol between the GCS and explorers,
and a proactive protocol between explorers and inspectors to
streamline task allocation and data collection.
Extensive simulations validate its scalability and applicability
to diverse scenarios and large-scale robotic fleets.

Future work includes:
(I)  enhancing the exploration efficiency of subgroups
by enabling multi-robot exploration,
     which necessitates more sophisticated communication protocols;
(II) developing exploration strategies that are independent of prior information,
     such as bounding boxes;
(III) incorporating diverse feature types and inspector roles to
      handle more complex task requirements;
{(IV) integrating more realistic communication models,
      including Non-Line-of-Sight (NLOS), signal attenuation, and multi-path effects,
      to enhance the system's robustness for real-world complex environments.}


\bibliographystyle{IEEEtran}
\bibliography{contents/root}

\end{document}